\documentclass[11pt]{article}
\usepackage[dvipsnames]{xcolor}
\usepackage{ulem}
\usepackage{pgfplots}
\usepackage{booktabs}
\usepackage{authblk}
\usepackage{graphicx,wrapfig,lipsum}
\usepackage{float}    %
\usepackage{verbatim} %
\usepackage{amsmath, amssymb, amsthm, mathtools}  %
\usepackage{subfigure}   %
\usepackage[colorlinks=true,citecolor=blue,linkcolor=blue,urlcolor=blue]{hyperref}
\usepackage{cleveref}
\usepackage{bookmark}
\usepackage{fullpage}
\usepackage{enumerate}
\usepackage{paralist}
\usepackage{xspace}
\usepackage{multirow}
\usepackage{caption}
\usepackage{bbm}
\usepackage{bm}
\usepackage{url}
\usepackage{lipsum}
\usepackage{algorithm}%
\usepackage{algorithmicx}%
\usepackage{algpseudocode}%
\usepackage{color}
\usepackage{microtype}
\usepackage[section]{placeins}
\usepackage{lscape}
\usepackage{multirow}
\usepackage{todonotes}

\usepackage{tcolorbox}

\newcommand{\stoptocwriting}{%
	\addtocontents{toc}{\protect\setcounter{tocdepth}{-5}}}

\newcommand{\vertiii}[1]{{\left\vert\kern-0.25ex\left\vert\kern-0.25ex\left\vert #1
		\right\vert\kern-0.25ex\right\vert\kern-0.25ex\right\vert}}

\makeatletter

\newcommand{\vect}[1]{\ensuremath{\mathbf{#1}}}
\newcommand{\x}{\vect{x}}
\renewcommand{\u}{\vect{u}}
\renewcommand{\v}{\vect{v}}

\newcommand{\mTheta}{\vect{\Theta}}
\newcommand{\bu}{\vect{u}}
\newcommand{\bv}{\vect{v}}
\newcommand{\X}{\vect{X}}
\newcommand{\Y}{\vect{Y}}

\newcommand{\bxi}{\bm{\xi}}

\newcommand{\mSigma}{\bm{\Sigma}}
\newcommand{\mDelta}{\bm{\Delta}}
\newcommand{\y}{\vect{y}}
\newcommand{\z}{\vect{z}}
\newcommand{\M}{{\vect{M}^{\star}}}

\newcommand{\I}{\vect{I}}
\newcommand{\R}{\vect{R}}
\newcommand{\mL}{\vect{L}}
\newcommand{\A}{\vect{A}}
\newcommand{\B}{\vect{B}}

\newcommand{\D}{\vect{D}}
\newcommand{\U}{\vect{U}}
\newcommand{\W}{\vect{W}}
\newcommand{\mH}{\vect{H}}
\newcommand{\V}{\vect{V}}
\newcommand{\Z}{\vect{Z}}
\newcommand{\mO}{\vect{O}}

\newcommand{\mtheta}{\bm{\theta}}
\newcommand{\vdelta}{\bm{\delta}}
\newcommand{\fro}{\mathrm{F}}
\newcommand*{\rom}[1]{%
\textup{\uppercase\expandafter{\romannumeral#1}}%
}

\newcommand{\barX}{\bar{\X}}

\newcommand{\zero}{\bm{0}}

\DeclareMathOperator{\col}{col}

\DeclareMathOperator{\dist}{dist}

\DeclareMathOperator{\poly}{poly}
\DeclareMathOperator{\polylog}{polylog}

\DeclareMathOperator{\vd}{d\!}

\newcommand{\bareg}{\bar\epsilon_{\mathrm{g}}} 
\newcommand{\bareH}{\bar\epsilon_{\mathrm{H}}} 
\newcommand{\bareM}{\bar\epsilon_{\mathcal{M}}}

\newcommand{\rank}{\mathrm{rank}}

\newcommand{\proj}{\cP}

\newcommand{\inner}[2]{\left\langle #1, #2 \right\rangle}

\newcommand{\bignorm}[1]{\left\lVert#1\right\rVert}
\newcommand{\norm}[1]{\left\lVert#1\right\rVert}

\DeclareMathOperator*{\argmin}{arg\,min}

\newcommand{\cA}{\mathcal{A}}

\newcommand{\cD}{\mathcal{D}}

\newcommand{\cM}{\mathcal{M}}
\newcommand{\cN}{\mathcal{N}}
\newcommand{\cO}{\mathcal{O}}
\newcommand{\cP}{\mathcal{P}}

\newcommand{\cR}{\mathcal{R}}

\newcommand{\cX}{\mathcal{X}}
\newcommand{\cY}{\mathcal{Y}}

\newcommand{\bB}{\mathbb{B}}

\newcommand{\bP}{\mathbb{P}}

\newcommand{\bR}{\mathbb{R}}

\newcommand{\XM}{\mathcal{X}_{\mathcal{M}}}

\global\long\def\xsharp{\boldsymbol{\x}^{\sharp}}
\global\long\def\xperp{\boldsymbol{\x}^{\perp}}

\newtheorem{theorem}{Theorem}
\newtheorem{definition}{Definition}
\newtheorem{lemma}{Lemma}

\newtheorem{proposition}{Proposition}
\newtheorem{example}{Example}
\newtheorem{assumption}{Assumption}

\begin{document}

	\title{Understanding the Implicit Regularization of Gradient Descent in Over-parameterized Models}
	
    \author[1]{Jianhao Ma}
\author[2]{Geyu Liang}
\author[3]{Salar Fattahi}

\affil[1]{University of Pennsylvania\\ \texttt{jianhaom@wharton.upenn.edu}}
\affil[2]{Amazon\\ \texttt{lianggy@umich.edu}}
\affil[3]{University of Michigan\\ \texttt{fattahi@umich.edu}}

	\maketitle
	\begin{abstract}
		 Implicit regularization refers to the tendency of local search algorithms to converge to low-dimensional solutions, even when such structures are not explicitly enforced. Despite its ubiquity, the mechanism underlying this behavior remains poorly understood, particularly in over-parameterized settings. We analyze gradient descent dynamics and identify three conditions under which it converges to second-order stationary points within an implicit low-dimensional region: (i) suitable initialization, (ii) efficient escape from saddle points, and (iii) sustained proximity to the region. We show that these can be achieved through \textit{infinitesimal perturbations} and a small \textit{deviation rate}. Building on this, we introduce \textit{Infinitesimally Perturbed Gradient Descent} (IPGD), which satisfies these conditions under mild assumptions. We provide theoretical guarantees for IPGD in over-parameterized matrix sensing and empirical evidence of its broader applicability.
	\end{abstract}

	\stoptocwriting

    \section{Introduction}
\label{sec::intro}
\textit{Over-parameterization}, where the sheer number of variables surpasses what is essential for identifying the optimal solution, is a ubiquitous phenomenon in contemporary optimization problems in machine learning (ML). In the face of over-parameterization, gradient descent (GD) or its variants may encounter significant statistical gaps as they become prone to \textit{overfitting}. However, recent research has postulated that these algorithms exhibit an \textit{implicit regularization} in a wide range of problems---spanning from sparse recovery problem~\cite{vaskevicius2019implicit, zhao2019implicit, ma2022blessing, woodworth2020kernel} to low-rank matrix and tensor optimization~\cite{li2018algorithmic, stoger2021small, soltanolkotabi2023implicit, wang2020beyond, tong2022scaling} to neural network training~\cite{belkin2019reconciling, yang2020rethinking, arora2018optimization, huh2021low, neyshabur2014search, neyshabur2017implicit}---favoring solutions with \textit{low dimensionality}, even when such low-dimensional structures are neither explicitly specified nor encoded in the problem.

More specifically, consider the general task of estimating unknown parameters $\mtheta$ of a model by minimizing a loss function $L:\mathbb{R}^n\to \mathbb{R}$:
\begin{align}\label{eq::loss1}
    \min_{\mtheta\in\mathbb{R}^n} \ L(\mtheta)\qquad \text{subject to}\qquad \mtheta\in \cD.
\end{align}
In the above optimization, $\cD\subseteq \mathbb{R}^n$ denotes the feasible region containing the underlying ground truth parameter $\mtheta^\star$ of the model, frequently presumed to possess a \textit{low dimension}, specifically $\dim(\cD)=k\ll n$. Many learning tasks are naturally endowed with such low-dimensional structures: Consider the following examples:
\begin{itemize}
    \item {\bf Sparse recovery:} In the sparse recovery problem, the goal is to obtain a sparse vector $\mtheta^\star\in \bR^n$ with at most $k\ll n$ nonzero entries that minimizes a loss function $L(\mtheta)$~\cite{mazumder2023subset, hazimeh2020fast}. In this case, the feasible region can be defined as $\cD = \{\mtheta\in \mathbb{R}^n: \mtheta_i=0,\ \text{if $\mtheta^\star_i=0$}\}$, which satisfies $\dim(\cD)=k$.
    \item {\bf Low-rank matrix/tensor recovery:} Another category of problems that reveal naturally low-dimensional structures is found in low-rank matrix recovery \cite{chi2019nonconvex}. A prominent example is the matrix sensing problem, where the goal is to recover an \( n \times n \) positive semi-definite (PSD) matrix of rank at most \( r \) from a limited number of random measurements. In this setting, the feasible region is $\cD=\{\mTheta\in \mathbb{R}^{n\times n}: \mTheta\succeq 0, \rank(\mTheta)\leq r \}$ with $\dim(\cD)=\cO(nr)$, where $\mTheta\succeq 0$ imposes PSD constraint. This setting can be naturally extended to low-rank tensor recovery, where the goal is to recover an order-$L$ tensor of rank $r$~\cite{cai2019nonconvex}. In this scenario, the feasible region $\cD$ can be defined as the set of order-$L$ tensors of rank $r$.
\end{itemize}

A critical challenge in solving \eqref{eq::loss1}, however, is that the feasible region $\cD$ \textit{is rarely known in practice}. For example, in sparse recovery, it is unreasonable to assume prior knowledge of the support of the true parameter $\mtheta^\star$, which is crucial in characterizing $\cD$. Similarly, in low-rank models, the true rank $r$ is rarely known \textit{a priori}. A common workaround is to introduce a regularizer that promotes the low-dimensional structure underlying $\cD$. However, this strategy requires tuning 
 regularization parameters (e.g., chosen via cross-validation~\cite{zhang1993model}), for which the statistically optimal values may still depend on the intrinsic dimension of $\cD$ or on the structure of $\mtheta^\star$~\cite{cawley2010over, raskutti2011minimax, negahban2011estimation}.

Alternatively, a recent widely practiced yet theoretically underexplored approach is to \textit{over-parameterize} the model, that is, to craft a sufficiently large region $\cD'$ that safely encompasses $\cD$. This is typically achieved by overestimating the dimension $\dim(\cD)=k$ with $k'\geq k$. However, directly solving $\min_{\mtheta \in \cD'} L(\mtheta)$ often leads to overfitting, as the optimal solution might be in $\cD' \setminus \cD$, resulting in poor statistical performance. To mitigate this issue, a common strategy is to introduce a {\it reparameterization map} $\varphi: \mathbb{R}^{d} \to \mathbb{R}^n$ whose range covers $\cD'$, i.e., $\cD' \subseteq \{ \varphi(\x): \x \in \mathbb{R}^d \}$. Therefore, the over-parameterization procedure transforms problem~\eqref{eq::loss1} into:
\begin{align}\label{eq:over-loss}
    \underbrace{\min_{\mtheta\in\cD}\ L(\mtheta)}_{\text{\bf Exactly-parameterized}} \stackrel{\text{\bf Over-parameterization}}{\implies} \underbrace{\min_{\mtheta\in\cD'}\ L(\mtheta) \quad \stackrel{\text{\bf Reparameterization}}{\implies}\quad \min_{\x\in\bR^{d}}\  f(\x) := L(\varphi(\x))}_{\text{\bf Over-parameterized}}.
\end{align}
Over-parameterized models have become prominent in sparse recovery, where the reparameterization map is often defined as $\varphi:\bR^{2n}\to \bR^n$ with $\varphi(\u,\v) = \u\odot\v$~\cite{zhao2019implicit} or $\varphi(\u,\v) = \u\odot\u-\v\odot\v$~\cite{vaskevicius2019implicit}. Here, the notation $\odot$ denotes the element-wise product. Additionally, over-parameterized models are widely used in low-rank matrix and tensor models when the true rank $r$ is unknown. In such scenarios, the true rank $r$ is often over-estimated with a {search rank} $r'\geq r$, and the reparameterization map is defined as $\varphi:\bR^{n\times r'}\to \bR^{n\times n}$ with $\varphi(\X)=\X\X^\top$ for PSD matrices~\cite{chi2019nonconvex}.

This over-parameterization approach offers two notable benefits. First, it converts the original constrained problem into an unconstrained one, making it amenable to a broad class of efficient gradient-based optimization methods. Second, its formulation does not require prior knowledge of $\cD$. At the same time, the approach introduces two main challenges. The first is that the composite function $f = L \circ \varphi$ is generally \textit{nonconvex}, even when the original problem~\eqref{eq::loss1} is convex. Consequently, gradient-based algorithms may become trapped in suboptimal solutions of $f$. The second challenge is statistical: the reparameterization trick does not, by construction, prevent overfitting. Even if a global minimum $\hat{\x}$ of $f$ is efficiently found, the corresponding parameter $\hat{\mtheta} = \varphi(\hat{\x})$ may lie outside the true feasible set when $\min_{\x\in \bR^{n}}f(\x)\leq \min_{\mtheta\in\cD'} L(\mtheta)<\min_{\mtheta\in\cD} L(\mtheta)$.

Despite these challenges, recent research has shown that, in a wide range of over-parameterized learning problems with proper reparameterization maps, gradient-based algorithms exhibit an \textit{implicit regularization} that steers them away from solutions outside $\cD$, favoring those in the low-dimensional feasible region $\cD$. For example, in sparse recovery, GD initialized at a small point and applied to the composite function $f(\u,\v) = L(\u\odot\u - \v\odot\v)$ converges to the true $k$-sparse parameter $\mtheta^\star$~\cite{vaskevicius2019implicit, woodworth2020kernel}, even though denser solutions may achieve comparable or even lower loss values. Remarkably, this guarantee holds even when the sparsity level $k$ is unknown. A related phenomenon arises under the reparameterization $\varphi(\u,\v) = \u \odot \v$~\cite{zhao2019implicit, ma2022blessing}. Analogously, in low-rank matrix recovery, GD with a small initialization applied to $f(\X) = L(\X\X^\top)$ converges to the true rank-$r$ matrix $\mTheta^\star$, even when $r$ is unknown and severely overestimated (up to $r' = n$)~\cite{li2018algorithmic, stoger2021small, ma2023global}. Extensions of these results have also been established for low-rank tensor recovery~\cite{anandkumar2017analyzing}, and various types of nonlinear neural networks~\cite{ma2023behind}.

Despite these intriguing and promising findings, the theoretical foundations of the implicit regularization induced by GD in over-parameterized settings remain largely underexplored. In this paper, we aim to address this gap by presenting unified conditions that explain the implicit regularization in GD, potentially extending its applicability to broader classes of problems.

\subsection{Our Goal}
Our objective is to establish general conditions under which GD converges to a second-order stationary point (SOSP) of $f$ within the set $\varphi^{-1}(\cD) = \{\x \in \mathbb{R}^d : \varphi(\x) \in \cD\}$. For brevity, we refer to $\varphi^{-1}(\cD)$ as the \textit{implicit region} and denote it by $\cM := \varphi^{-1}(\cD)$. Accordingly, an SOSP of $f$ within $\cM$ will be called an $\cM$-SOSP and is formally defined as:

\begin{equation}
    \text{$\x$ is $\cM$-SOSP} \qquad \iff \qquad \x\in \cM,\ \ \nabla f(\x)=0,\ \ \nabla^2f(\x)\succeq 0.
\end{equation}
Our interest in an $\cM$-SOSP stems from the key observation that, for generic functions, any $\cM$-SOSP corresponds to a local minimum of $f$ within $\cM$~\cite[Theorem~2.9]{davis2019active}. By construction of $\cM$, any solution $\x \in \cM$ maps to a parameter $\mtheta = \varphi(\x) \in \cD$, which, by definition of $\cD$, is low-dimensional and satisfies the feasibility constraints. In general, however, a local minimum $\x \in \cM$ of $f$ need not map to a local minimum $\varphi(\x) \in \cD$ of $L$; indeed, it is possible for a local minimum of $f$ to map under $\varphi$ to a non-stationary point of $L$~\cite[Example~3.7]{levin2025effect}. Nevertheless, \cite[Theorem~2.3]{levin2025effect} shows that local minima of $f$ do map to local minima of $L$ under the mild assumption that $\varphi$ is an open map. For the canonical reparameterizations used in sparse recovery and low-rank matrix recovery, the openness of $\varphi$ can be verified directly~\cite[Section~2]{levin2025effect}. Therefore, for these instances, any local minimum $\x\in \cM$ of $f$ maps to a local minimum $\varphi(\x)\in \cD$ of $L$. Motivated by this connection, our goal is to address the following question:

\begin{tcolorbox}
{\it Under what conditions on the implicit region $\cM$, the objective function $f$, 
and the initialization $\x_0$ does GD converge to an $\cM$-SOSP of $f$?}
\end{tcolorbox}
In this work, we identify three key conditions that enable the convergence of GD to an \( \cM \)-SOSP: (i) initialization sufficiently close to the implicit region \( \cM \); (ii) efficient escape from strict saddle points (SSPs); and (iii) preservation of proximity to \( \cM \) throughout the iterations.
Next, we explain these conditions and outline the main challenges associated with ensuring that they hold.

\subsection{Major Challenges}\label{subsec:failure}

Somewhat surprisingly, even if $\cD$ (and consequently $\cM$) are unknown \textit{a priori}, it is still possible to select an initialization $\x_0$ within the proximity of $\cM$ through an appropriate choice of $\varphi$. For instance, recall that in low-rank matrix recovery, the feasible region is the set of all PSD matrices with rank at most $r$, and the reparameterization is given by $\varphi(\X) = \X\X^\top$. In this case, the zero matrix belongs to $\cM = \varphi^{-1}(\cD)$ for every $r \leq n$. Similarly, in sparse recovery, the feasible region is the set of all vectors with sparsity level at most $k$, with reparameterizations $\varphi(\u,\v) = \u \odot \u - \v \odot \v$ or $\varphi(\u,\v) = \u \odot \v$. Here too, a pair of all-zero vectors lie in $\varphi^{-1}(\cD)$ for any $k\leq n$. As we will discuss later, analogous initializations can also be constructed for more general models with sparse basis function decompositions. Thus, initializing GD at these points satisfies the first condition.

Now, we turn to the second condition, that is, the efficient escape from SSPs. It has been shown that various gradient-based algorithms, such as \textit{perturbed gradient descent} (PGD)~\cite{ge2015escaping, jin2017escape} and \textit{stochastic gradient descent} (SGD)~\cite{fang2019sharp}, can escape SSPs and converge to an approximate SOSP as effectively as their second-order counterparts. To illustrate the underlying success mechanisms of these algorithms, let us consider PGD.
Starting from an initial point $\x_0$, PGD iteratively updates the solution using the rule $\x_{t+1} = \x_t - \eta\nabla f(\x_t) + \bm{\xi}_t$, where $\bm{\xi}_t$ represents an isotropic perturbation uniformly drawn from the Euclidean ball with radius $\gamma$ centered at zero when $\norm{\nabla f(\x_t)}$ is sufficiently small; otherwise, $\bm{\xi}_t$ is set to zero~\cite[Algorithm 1]{jin2017escape}. The reasoning behind this episodic perturbation is natural: when $\norm{\nabla f(\x_t)}$ is small, the iterate $\x_t$ is either near an SSP or an SOSP. Unlike SOSPs, SSPs are \textit{unstable} fixed points of GD. Consequently, an isotropic random perturbation of sufficient magnitude enables GD to escape from SSPs efficiently \cite[Theorem 3]{jin2017escape}. 

\begin{figure}
\begin{minipage}{.47\textwidth}
\begin{center}
\includegraphics[width=0.85\textwidth]{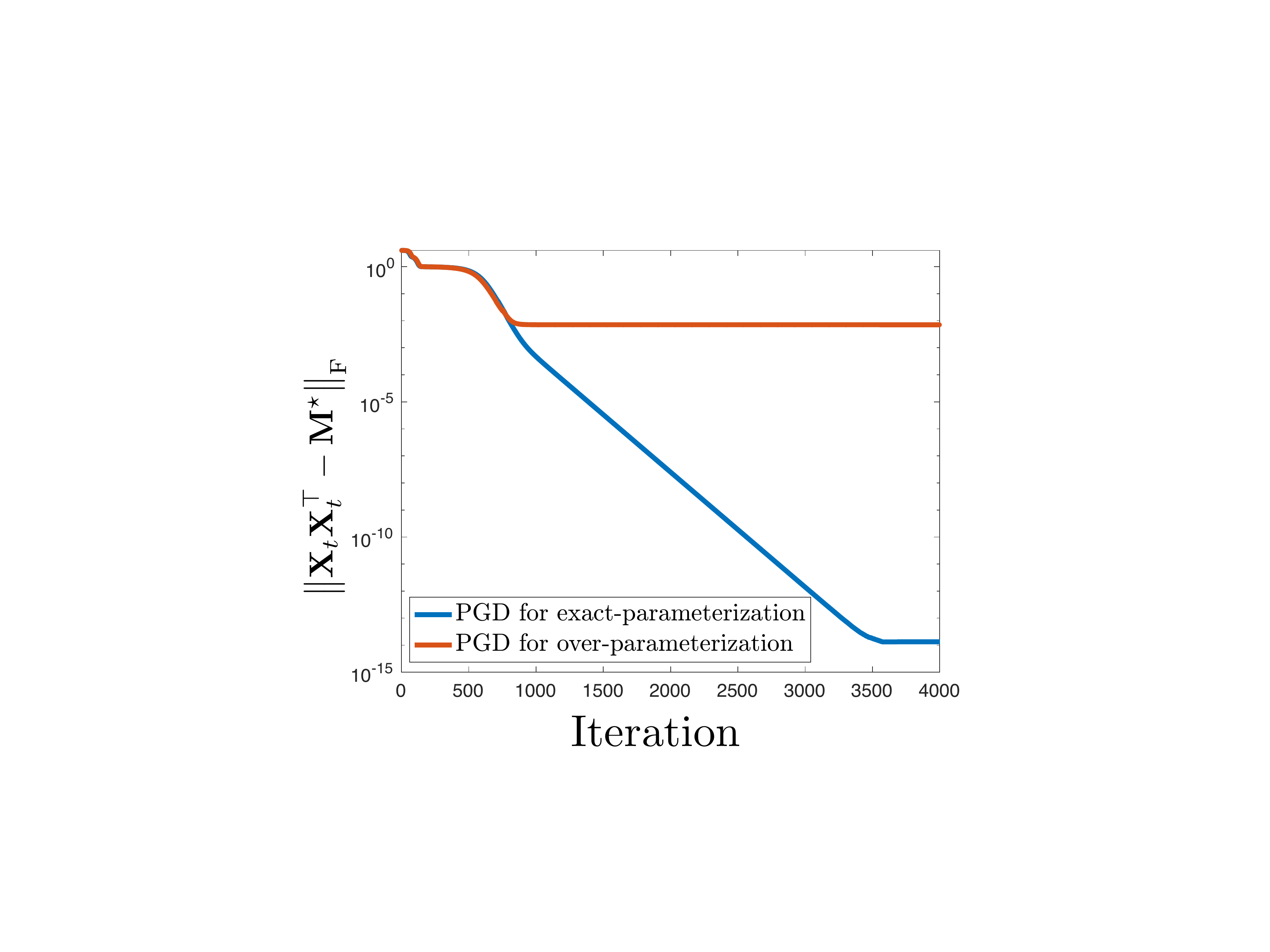}
\caption{PGD applied to matrix sensing converges under exact parameterization but fails to converge with over-parameterization. In this example, the true matrix $\mTheta^\star$ is a $20 \times 20$ PSD matrix with rank $r = 3$. The search rank is set to $r' = 3$ and $r'=4$ for exact- and over-parameterized settings.}\label{fig:comparison}
  \end{center}
\end{minipage}
\hfill 
\begin{minipage}{.48\textwidth}
\vspace{-0.3cm}
\includegraphics[width=\textwidth]{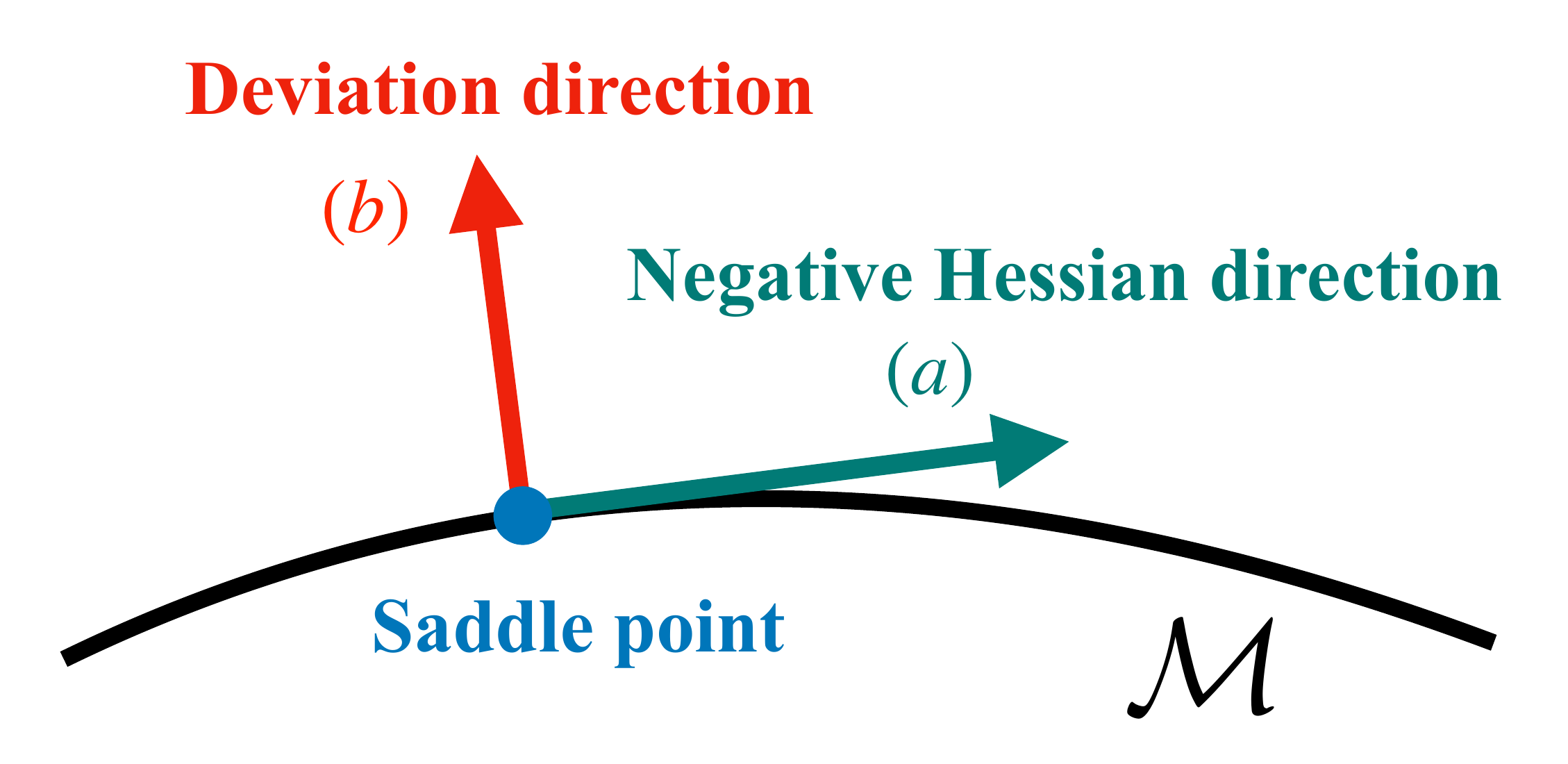}\vspace{0.7cm}
\caption{When perturbation is added to the gradient, it contributes to two directions: deviation direction, which steers the iterations away from the implicit region $\cM$, and the negative Hessian direction, which enables an escape from SSPs.}\label{fig:pgd}
\end{minipage}
\vspace{-0.15in}
\end{figure}

Unfortunately, in over-parameterized settings, guaranteed escape from SSPs may come at the cost of significant deviation from the implicit region $\cM$, thereby violating the third condition and preventing PGD from converging to an $\cM$-SOSP. In fact, even mild over-parameterization can disrupt the convergence of PGD. Figure~\ref{fig:comparison} illustrates this phenomenon on an instance of matrix sensing (introduced in \Cref{sec::intro}). When the rank $r$ of the target matrix is known, PGD initialized at random and applied to $f(\X) = L(\X\X^\top)$ with $\X \in \mathbb{R}^{n \times r}$ converges to the ground-truth rank-$r$ matrix. By contrast, even a slight over-parameterization with a search rank of $r' = r+1$ causes PGD to deviate significantly from $\cM$ and converge to a solution of rank strictly larger than $r$, thereby incurring substantial error.

 The failure of PGD to converge to an $\cM$-SOSP in the over-parameterized regime is precisely due to the impact of perturbations, as shown in Figure~\ref{fig:pgd}. When an isotropic random perturbation is applied, it inevitably affects two directions: a direction spanned by the eigenvectors corresponding to the negative eigenvalues of the Hessian, which pushes the iterate away from the SSP (shown as $(a)$), and another direction that causes the iterate to deviate away from the implicit region $\mathcal{M}$ (shown as $(b)$). When the perturbation is both isotropic and large, these directions become comparable, causing the PGD iterates to escape the saddle point but simultaneously drift away from $\mathcal{M}$, ultimately leading to the failure of the algorithm.

The central premise of this paper is that the second and third conditions can be met simultaneously: SSPs can be escaped efficiently via \textit{infinitesimal perturbations}, while closeness to $\cM$ is maintained by quantifying the \textit{deviation rate} of $f$ from $\cM$. We show that GD exhibits implicit regularization toward an $\cM$-SOSP when initialized near $\cM$ and equipped with both infinitesimal perturbations and a small deviation rate. Moreover, while the guaranteed escape from SSPs via infinitesimal perturbations is largely independent of $\cD$ and the specific choice of $\varphi$, ensuring a small deviation rate depends crucially on $\cD$ and the choice of $\varphi$. We establish that these conditions hold provably for the class of over-parameterized matrix sensing problems, and empirically for broader classes of problems.

\subsection{Overview of Our Contributions}
At the core of our framework is a \textit{signal–residual decomposition}: each iterate $\x_t$ is written as $\x_t = \xsharp_t + \xperp_t$, where $\xsharp_t$ is the \textit{signal}, obtained by projecting $\x_t$ onto the implicit region $\cM$, and $\xperp_t := \x_t - \xsharp_t$ is the \textit{residual}. Since $\xsharp_t \in \cM$, convergence to an approximate $\cM$-SOSP requires showing that (i) $\x_t$ approaches an SOSP and (ii) the residual norm $\|\xperp_t\|$ remains small.
 
\paragraph{Escaping saddle points efficiently.}
We show that large perturbations are \textit{unnecessary} for GD to escape SSPs. Specifically, we prove that, when GD encounters an SSP, a single perturbation of arbitrarily small radius $\gamma>0$, followed by $T=O(\polylog(1/\gamma))$ iterations of GD updates, enables the algorithm to escape from this SSP. This result is an improvement over the analysis of~\cite{jin2017escape}, revealing that the perturbation radius affects the convergence rate {\it only poly-logarithmically}, and hence, can be made arbitrarily small at the cost of only poly-logarithmic increase in the number of iterations (\Cref{thm::IPGD}). Crucially, such infinitesimal isotropic perturbations preserve a small residual norm, overcoming the main limitation of projected GD. We refer to this variant as \textit{Infinitesimally Perturbed Gradient Descent} (IPGD) (\Cref{alg::IPGD}).

\paragraph{Implicit regularization.}
Even with infinitesimal perturbations, the iterates may drift away from $\cM$ during the standard GD updates, causing the residual norm $\norm{\xperp_t}$ to grow. We analyze this behavior by tracking the dynamics of $\|\xperp_t\|$ and show that its evolution is governed by a quantity termed the \textit{deviation rate} (Definition~\ref{def::deviation_rate}). A negative deviation rate contracts the residual, while a positive one causes drift (\Cref{prop::residual_norm}). Combining this with the saddle-escaping property, we establish that IPGD converges to an $\epsilon$-neighborhood of an $\cM$-SOSP within $\tilde{\mathcal{O}}(1/\epsilon^2)$ iterations, under three conditions: (1) closure of $\cM$ under GD updates, (2) initialization near $\cM$, and (3) a sufficiently small deviation rate (\Cref{thm::IPGD-MSOSP}).

\paragraph{Improved convergence under additional structure.}
We show that, when the function satisfies the strict-saddle property within $\cM$, convergence to an $\cM$-SOSP implies convergence to a local minimum (\Cref{cor::IPGD-SSP}). Moreover, under a mild regularity condition, the convergence rate of IPGD improves to nearly linear (\Cref{thm::linear}). Unlike prior analyses that assume global structure, our results only require these properties to hold within the low-dimensional manifold $\cM$.

\paragraph{Over-parameterized matrix sensing.}
We illustrate our framework on a canonical problem: \textit{over-parameterized matrix sensing}. This problem serves as an ideal test-bed because, as discussed in \Cref{subsec:failure}, general-purpose saddle-escaping methods like PGD fail to reach the true solution, while existing GD analyses are often intricate and lengthy~\cite{stoger2021small,li2018algorithmic}. In contrast, our general conditions yield a concise and unified proof of IPGD’s near-linear convergence, greatly simplifying the analysis (\Cref{thm::global-convergence-matrix-sensing}).

\paragraph{Broader applications.}
Finally, we verify empirically that the assumptions ensuring IPGD convergence hold across various over-parameterized problems, including matrix completion (symmetric, asymmetric, and one-bit) and sparse recovery. Notably, we show that in practice, a perturbation radius of \(\gamma = 1 \times 10^{-15}\)---which is close to the round-off error in double precision---suffices to ensure the efficient convergence of IPGD. This underscores the insight that, while theoretically distinct, IPGD can be viewed as a practical implementation of GD under inexact arithmetic.

\subsection{Related Work}
\label{sec::related_work}

\paragraph{Saddle-avoiding algorithms.} 
Second-order methods such as cubic regularization~\cite{nesterov2006cubic}, trust-region algorithms~\cite{curtis2017trust}, and their efficient variants~\cite{carmon2018accelerated,agarwal2017finding,carmon2019gradient} are well known for efficiently escaping SSPs. However, their use in large-scale problems is limited by the high cost of computing or approximating Hessian information. Recent work shows that even first-order methods can escape SSPs under suitable conditions: GD with random initialization avoids them almost surely~\cite{lee2019first,panageas2019first}; under stronger geometric assumptions, such as global Hessian invertibility, GD can escape linearly~\cite{dixit2023exit,dixit2022boundary}, though such assumptions rarely hold in over-parameterized regimes.  

To overcome these limitations, several GD variants have been proposed. PGD escapes SSPs using episodic random perturbations, a property shared by its non-periodic version~\cite{jin2021nonconvex}, stochastic GD with \textit{dispersive noise}~\cite{fang2019sharp}, and earlier noisy or normalized variants~\cite{ge2015escaping,levy2016power}. Such additive noise appears essential, as plain GD may take exponential time to escape~\cite{du2017gradient}. More advanced techniques, including accelerated GD~\cite{jin2018accelerated} and adaptive methods~\cite{staib2019escaping}, achieve faster escape rates but can still deviate significantly from the implicit region in over-parameterized settings.

\begin{sloppypar}
    \paragraph{Implicit regularization of local-search algorithms} 
As discussed earlier, local-search algorithms in over-parameterized models often exhibit \textit{implicit regularization}, converging to low-dimensional solutions (e.g., sparse or low-rank) even without explicit penalties. For instance, training factorized matrix models via gradient or alternating updates typically yields low-rank solutions: \cite{gunasekar2017implicit} conjectured, and provided evidence, that GD on full-rank matrix factorizations converges to the minimum nuclear-norm solution, while \cite{arora2019implicit} showed that increasing the factorization depth amplifies this low-rank bias. Relatedly, \cite{li2022implicit} proved that continuous gradient flow under any commuting parametrization is equivalent to mirror descent with a Legendre function acting as a regularizer, formalizing implicit regularization in transformed spaces. Coordinate descent methods exhibit similar sparsity biases: in classification, coordinate-wise descent converges to the maximum $\ell_1$-margin (sparsest) separator, whereas standard GD yields the maximum $\ell_2$-margin solution~\cite{soudry2018implicit}.  

Despite these advances, existing analyses of implicit regularization remain instance-specific. To our knowledge, this work is the first to present a unified framework for studying the implicit regularization behavior of GD in general over-parameterized models.

\end{sloppypar}

\paragraph{Desirable landscape properties.}
Several landscape conditions have been identified that enable first-order methods to achieve strong convergence guarantees. A central one is the \textit{strict saddle property}, which ensures every saddle point has a direction of negative curvature, allowing algorithms such as PGD~\cite{jin2017escape} to escape SSPs and converge to local minima. Combined with \textit{benign nonconvexity}—the absence of suboptimal local minima—this guarantees global convergence for any method that avoids saddle points.
These properties hold in many problems, including matrix sensing and completion~\cite{bhojanapalli2016global,ge2017no}, robust PCA~\cite{ge2017no}, dictionary learning~\cite{sun2016complete1,sun2016complete2}, tensor decomposition~\cite{ge2015escaping}, phase retrieval~\cite{sun2018geometric}, and certain neural networks~\cite{brutzkus2017globally,du2017convolutional,soltanolkotabi2018theoretical}. More generally, the strict saddle property holds \textit{generically} in smooth optimization problems~\cite{davis2019active}: for a smooth $f$ on $\mathbb{R}^d$ and almost every perturbation $\bv$, the function $f_{\bv}(\x)=f(\x)-\langle \bv,\x\rangle$ satisfies this property. Many such problems also exhibit structural regularities, e.g., restricted strong convexity or $(\alpha,\beta)$-regularity, that further accelerate GD~\cite{chi2019nonconvex}.  

However, these assumptions can be restrictive in over-parameterized regimes, where redundant dimensions may hinder convergence. Instead, we focus on neighborhoods of the \textit{implicit solution region}, where over-parameterization effects vanish and favorable convergence behavior emerges.

\subsection{Notations}
For a twice-differentiable function $f: \mathbb{R}^d \to \mathbb{R}$, the gradient and Hessian at a point $\mathbf{x}$ are denoted by $\nabla f(\mathbf{x})$ and $\nabla^2 f(\mathbf{x})$, respectively. For a vector \(\mathbf{x}\), \(\|\mathbf{x}\|\) denotes its usual Euclidean norm. The Euclidean distance of a vector $\x \in \bR^d$ from a set $\cY \subseteq \bR^d$ is denoted by $\dist(\x, \cY) = \min_{\y \in \cY} \norm{\x - \y}$. The projection of a vector $\x$ onto a set $\cY$ and its orthogonal complement are shown as $\proj_{\cY}(\x) = \arg\min_{\y \in \cY} \norm{\x - \y}$ and $\proj^\perp_{\cY}(\x) = \x - \proj_{\cY}(\x)$, respectively\footnote{If the minimizer in $\arg\min_{\y \in \cY} \norm{\x - \y}$ is not unique, we select an arbitrary one.}. For a set $\cX \subset \bR^d$, its $\epsilon$-neighborhood is defined as $\cN_{\cX}(\epsilon) = \{\x \in \bR^d : \dist(\x, \cX) \leq \epsilon\}$. We also define $\bB_{\z}(\xi)$ as the Euclidean ball with radius $\xi$ centered at the point $\z$. For short, we use $\bB(\xi):=\bB_{\zero}(\xi)$. Finally, $\text{Unif}(\bB(\xi))$ denotes the uniform distribution over the Euclidean ball $\bB(\xi)$.

For a matrix $\X$, we use $\norm{\X}$, $\norm{\X}_*$ and $\norm{\X}_{\fro}$ to denote its operator norm, nuclear norm, and Frobenius norm, respectively. The identity matrix of size $d \times d$ is denoted as $\I_d$. The set of all $d_1 \times d_2$ orthonormal matrices with $d_1 \geq d_2$ is defined as $\cO_{d_1 \times d_2} := \{\mO \in \bR^{d_1 \times d_2} : \mO^\top \mO = \I_{d_2}\}$.
For two matrices $\X, \Y \in \bR^{d_1 \times d_2}$, their Procrustes distance is defined as $\dist(\X, \Y) = \min_{\mO \in \cO_{d_2 \times d_2}} \bignorm{\X - \Y\mO}_{\fro}$. 
For an arbitrary matrix $\X$, we denote $\col(\X)$ as its column space. Additionally, the projection matrix onto $\col(\X)$ is denoted as $\proj_{\X} = \X\X^{\dagger}$, where $\X^{\dagger}$ refers to the pseudo-inverse of $\X$.

The gradient of a scalar-valued function $f(\X)$ with a matrix variable $\X \in \mathbb{R}^{d \times d}$ is a $d \times d$ matrix, whose $(i, j)$-th entry is $[\nabla f(\X)]_{i, j}=\frac{\partial f(\X)}{\partial X_{i j}}$ for $i, j \in[d]$. Alternatively, we can view the gradient as a linear operator satisfying $\nabla f(\X)[\Z]=\langle\nabla f(\X), \Z\rangle=\sum_{i, j} \frac{\partial f(\X)}{\partial X_{i j}} Z_{i j}$ for any $\Z \in \mathbb{R}^{d \times d}$. Similarly, the Hessian and the third-order derivative of $f(\X)$ can be viewed as multi-linear operators defined by
\begin{equation}
    \begin{aligned}
        \nabla^2 f(\X)[\Z, \W]&=\sum_{i, j, k, l} \frac{\partial^2 f(\X)}{\partial X_{i j} \partial X_{k l}} Z_{i j} W_{k l}, \ \text{for any }\Z, \W \in \mathbb{R}^{d \times d};\\
        \nabla^3 f(\X)[\Z, \W, \U]&=\sum_{i, j, k, l, m, n} \frac{\partial^3 f(\X)}{\partial X_{i j} \partial X_{k l}\partial X_{m n}} Z_{i j} W_{k l}U_{m n}, \ \text{for any }\Z, \W, \U \in \mathbb{R}^{d \times d}.
    \end{aligned}
\end{equation}

 Given $f(n)$ and $g(n)$, we write $f(n)=O(g(n))$ when there exists a universal constant $C>0$ satisfying $f(n) \leq Cg(n)$ for all large enough $n$. Similarly, the notation $f(n)=\Omega(g(n))$ implies the existence of a universal constant $C>0$ satisfying $f(n) \geq Cg(n)$ for all large enough $n$. Lastly, we use $f(n) = \Theta(g(n))$ if $f(n) = O(g(n))$ and $f(n)= \Omega(g(n))$. Finally, we use the symbols $\tilde{O}, \tilde{\Omega}, \tilde{\Theta}$ to hide poly-logarithmic factors.

\section{Main Results}
\label{sec::main-result}
First, we present our assumptions on the gradient- and Hessian-Lipschitz continuity for the function \( f \), which we require to hold only within a neighborhood of the implicit region \( \mathcal{M} \).

\begin{definition}[$(L, \cM, \tau)$-gradient-Lipschitz]
	\label{def::l-smooth}
	The function $f$ is {\bf $\bm{(L, \cM, \tau)}$-gradient-Lipschitz} if for all $\x,\x'\in \cN_{\cM}(\tau)$, we have $\norm{\nabla f(\x)-\nabla f(\x')}\leq L\norm{\x-\x'}$.
\end{definition}
\begin{definition}
	[$(\rho, \cM, \tau)$-Hessian-Lipschitz]
	\label{def::rho-hessian}
	The function $f$ is {\bf $\bm{(\rho, \cM, \tau)}$-Hessian-Lipschitz} if for all $\x,\x'\in \cN_{\cM}(\tau)$, we have $\norm{\nabla^2 f(\x)-\nabla^2 f(\x')}\leq \rho\norm{\x-\x'}$.
\end{definition}

Unlike the classical definitions of gradient- and Hessian-Lipschitz functions, which require smoothness conditions to hold globally, Definitions~\ref{def::l-smooth} and~\ref{def::rho-hessian} impose these conditions only within a $\tau$-neighborhood of the implicit region $\cM$. These localized gradient- and Hessian-Lipschitz conditions are advantageous when the smoothness constants scale with the dimensionality of the problem. In such cases, the global smoothness constants scale with the ambient dimension $d$, while their localized counterparts depend only on $\dim(\cM) = \dim(\varphi^{-1}(\cD))$ which, depending on $\dim(D)$ and the choice of $\varphi$, can be significantly smaller than $d$. 

Recall that our goal is to identify conditions guaranteeing the efficient convergence to $\cM$-SOSPs. Unfortunately, practical algorithms can only compute an \textit{approximate} SOSP, where the stationarity conditions are satisfied within some small tolerance, as formally defined below.

\begin{definition}[approximate ($\cM$-)SOSP]\label{def:M-SOSP}
Suppose that $f$ is $(L, \cM,\tau)$-gradient-Lipschitz and $(\rho, \cM,\tau)$-Hessian-Lipschitz. 
	Given a tolerance parameter $\epsilon>0$, a point $\x\in \cN_{\cM}(\tau)$ is called $\bm{\epsilon}${\bf -SOSP} if $\norm{\nabla f(\x)}\leq \epsilon$ and $\nabla^2 f(\x)\succeq -\sqrt{\rho\epsilon}\ \I$. Additionally, $\x$ is called $\bm{\epsilon}${\bf -}$\bm{\cM}${\bf -SOSP} if $\x$ is $\epsilon$-SOSP and $\x\in \cN_{\cM}(\epsilon/L)$ where $\epsilon/L\leq \tau$.
\end{definition}

Closely related is the notion of {\it approximate} {strict saddle points} (SSPs), defined as follows:

\begin{definition}[Approximate SSP]
	Given a parameter $\epsilon>0$, a point $\x$ is called $\bm{\epsilon}${\bf-SSP} if $\norm{\nabla f(\x)}\leq \epsilon$ and $\lambda_{\min}\left(\nabla^2 f(\x)\right)< -\sqrt{\rho\epsilon}$, where $\lambda_{\min}(\cdot)$ denotes the minimum eigenvalue. 
\end{definition}
We say ``$\x$ is an {approximate} SOSP'' or ``$\x$ is an {approximate} SSP'' if $\x$ is an $\epsilon$-SOSP or $\epsilon$-SSP with a sufficiently small parameter $\epsilon>0$. An approximate $\cM$-SOSP is defined analogously.
Let $\cX(\epsilon)$ and $\XM(\epsilon)$ respectively denote the set of all $\epsilon$-SOSPs and $\epsilon$-$\cM$-SOSPs of $f$. For simplicity and whenever it is clear from the context, we use $\cX$ and $\XM$ to denote $\cX(0)$ and $\XM(0)$, respectively.
Our goal is to recover a point from $\XM(\epsilon)$ for sufficiently small tolerance $\epsilon>0$. By definition, these points are approximate SOSPs and lie within the proximity of the implicit region $\cM$. 

As previously highlighted, the concept of a {\it signal-residual decomposition} is central to our analysis. At each iteration $t$, the current point $\x_t$ is decomposed as $\x_t = \xsharp_t + \xperp_t$, where $\xsharp_t$ represents the signal, defined as $\xsharp_t = \cP_{\cM}(\x_t)$, and $\xperp_t$ denotes the residual, defined as $\xperp_t = \cP_{\cM}^\perp(\x_t) := \x_t - \xsharp_t$. This decomposition immediately implies that $\x_t$ is an $\epsilon$-$\cM$-SOSP if and only if $\x_t$ is an $\epsilon$-SOSP and that $\norm{\xperp_t}\leq \epsilon/L$. Therefore, to ensure efficient convergence to an approximate $\cM$-SOSP, it suffices to: (1) ensure that the iterates efficiently converge to an approximate SOSP; and (2) maintain a small residual norm for the iterates. The first property will be analyzed in Section~\ref{subsec::SOSP}, while the second will be addressed in Section~\ref{subsec::residual}.

\subsection{Efficient Convergence to Approximate SOSPs}\label{subsec::SOSP}
We first formalize the intuition behind the PGD algorithm discussed in Section~\ref{subsec:failure} by stating its convergence from~\cite{jin2017escape}. 

\begin{theorem}[Theorem 3 of~\cite{jin2017escape}; informal]\label{prop_PGD}
	Suppose that $f$ is ${(L, \cM, +\infty)}$-gradient-Lipschitz and ${(\rho, \cM, +\infty)}$-Hessian-Lipschitz with some parameters $L,\rho>0$. Then, with an overwhelming probability, the PGD algorithm with perturbation radius $\gamma = \tilde{\Theta}({\epsilon}/L)$ outputs an $\epsilon$-SOSP within $T = \tilde{O}\left(\frac{\Delta_f}{\eta\epsilon^{2}}\right)$ iterations. Here, $\Delta_f=f(\x_0)-\min_{\x\in \bR^d} f(\x)$. 
\end{theorem}

If PGD encounters $N$ approximate SSPs along its trajectory, it applies $N$ isotropic perturbations, each of radius $\gamma$. A single perturbation around an SSP $\x$ can increase the residual norm by $O(\gamma)$, so the cumulative increase in $\|\xperp_T\|$ can reach $O(\gamma N)$. As shown in~\cite{jin2017escape} and in our analysis, $N$ can be as large as $O\left(\epsilon^{-3/2}\right)$. Hence, with $\gamma = \tilde{\Theta}(\epsilon/L)$ as required by \Cref{prop_PGD}, the final residual norm scales as $\tilde{O}(\epsilon^{-3/2} \cdot \epsilon/L) = \tilde{O}(\epsilon^{-1/2}/L) \gg 1$. In other words, {\it ensuring the convergence of PGD to an approximate SOSP may inevitably lead to the divergence of the residual norm in the over-parameterized regime.}

To address this challenge, we prove that PGD can efficiently escape approximate SSPs using a significantly smaller perturbation radius. The detailed implementation of this variant of PGD, which we call \textit{infinitesimally-perturbed gradient descent} (IPGD), can be found in \Cref{alg::IPGD}. 

\begin{algorithm}[H]
    \caption{Infinitesimally-perturbed gradient descent (IPGD)}\label{alg::IPGD}
	\begin{algorithmic}[1]
		\Statex \textbf{Input:} Perturbation radius $\gamma$, tolerance parameter $\epsilon$,  stepsize $\eta$, initial point $\x_0$, Hessian-Lipschitz parameter $\rho$, optimality gap $\Delta_f$, and failure probability $\chi$.
            \State $F\leftarrow \frac{1}{C}\!\cdot\!\frac{\epsilon^{3/2}}{\sqrt{\rho}\left(\log^{3}\left(\frac{1}{\gamma}\right)+\log^3\left(\frac{\rho d\Delta_f}{\chi\epsilon}\right)\right)}$, $G\leftarrow \frac{1}{C}\!\cdot\! \frac{\epsilon}{\log^{2}(1/\gamma)}$, $T_{\mathrm{escape}}\leftarrow \frac{C}{\eta\sqrt{\rho\epsilon}}\left(\log\left(\frac{1}{\gamma}\right)+\log\left(\frac{\rho d\Delta_f}{\chi\epsilon}\right)\right)$, for a sufficiently large constant $C>0$
		\State $T_{\mathrm{noise}}\leftarrow -T_{\mathrm{escape}}-1$
		\For {$t=0, 1, \ldots$}
		\If {$\bignorm{\nabla f(\x_t)}\leq G$ and $t-T_{\mathrm{noise}}>T_{\mathrm{escape}}$}
            \State $\z_t\leftarrow \x_t$
		\State $\x_{t}\leftarrow \x_{t} + \bxi_t$ where $\bxi_t\sim \text{Unif}(\bB(\gamma)), \quad T_{\mathrm{noise}}\leftarrow t$
		\EndIf
        \If {$t-T_{\mathrm{noise}} = T_{\mathrm{escape}}$ and $f(\x_t)-f(\z_{T_{\mathrm{noise}}})\geq -F$}
            \State \textbf{Return} $\z_{T_{\mathrm{noise}}}$
        \EndIf
		\State $\x_{t+1}\leftarrow \x_t-\eta \nabla f\left(\x_t\right)$
		\EndFor
	\end{algorithmic}
\end{algorithm}

Our first theorem demonstrates that IPGD achieves nearly the same convergence guarantee as PGD (differing only by a logarithmic factor), while requiring a smaller perturbation radius.
\begin{sloppypar}
	\begin{theorem}[Convergence of IPGD to an approximate SOSP]\label{thm::IPGD}
		Suppose that $f$ is ${(L, \cM, +\infty)}$-gradient-Lipschitz and ${(\rho, \cM, +\infty)}$-Hessian-Lipschitz. Then, with probability of at least $1-\chi$, IPGD (\Cref{alg::IPGD}) with any perturbation radius satisfying $\gamma = O\left(\min\left\{\sqrt{\frac{\epsilon}{\rho}}\log^{-4}\left(\frac{\rho d\Delta_f}{\chi\epsilon}\right), \frac{\epsilon}{L}\right\}\right)$ and stepsize $\eta\leq 1/L$ outputs an $\epsilon$-SOSP within $T = O\left(\frac{\Delta_f}{\eta\epsilon^2}\left(\log^4\left(\frac{1}{\gamma}\right)+\log^4\left(\frac{\rho d\Delta_f}{\chi\epsilon}\right)\right)\right)$ iterations. Here, $\Delta_f=f(\x_0)-\min_{\x\in \bR^d} f(\x)$. 
	\end{theorem}
\end{sloppypar}
Unlike \Cref{prop_PGD}, our result demonstrates that IPGD can efficiently converge to an approximate SOSP using a perturbation radius that is significantly smaller than $\tilde{\Theta}(\epsilon/L)$, which is required in \Cref{prop_PGD}. This improvement comes at the cost of increasing the number of iterations only by a factor of $\log^4(1/\gamma)$. This implies that the total increase in the residual norm caused by perturbations can be kept small, without significantly impacting the convergence rate.

\begin{figure}
    \centering
    \includegraphics[width=0.65\textwidth]{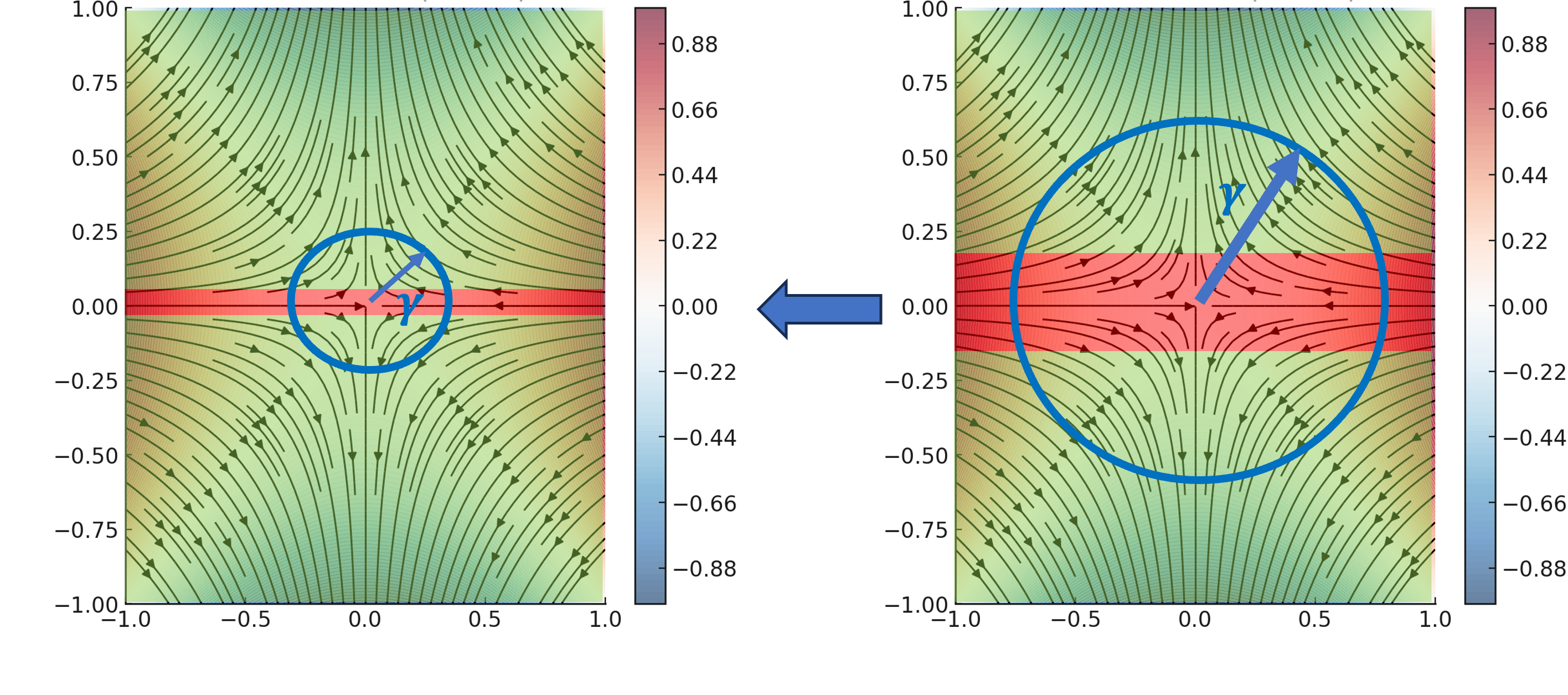}
    \caption{{\bf (Right)} The stuck and escape regions of PGD are depicted in green and red, respectively. A wide stuck region necessitates proportionally larger perturbations for escape. {\bf (Left)} Our results demonstrate that the width of the stuck region can be reduced to an arbitrarily small value, at the cost of a poly-logarithmic increase in the number of iterations.}
    \label{fig::PGD_vs_IPGD}
\end{figure}

Analogous to~\cite{jin2017escape}, our approach to proving \Cref{thm::IPGD} is based on characterizing the volume of the ``stuck region'' surrounding an approximate SSP. Specifically, let us assume that $\z$ is a SSP of $f$. \cite{jin2017escape} defines the {\it stuck region} $\cR_{\mathrm{stuck}}$ and the {\it escape region} $\cR_{\mathrm{escape}}$ as follows:
\begin{align}
    \cR_{\mathrm{escape}} &:= \{\x_0\in \bB_{\z}(\gamma): f(\x_t)-f(\x_0)< -F_{\mathrm{PGD}}, \text{within $t\leq T_{\mathrm{PGD}}$ GD updates}\},\\
    \cR_{\mathrm{stuck}} &:= \{\x_0\in \bB_{\z}(\gamma): f(\x_t)-f(\x_0)\geq -F_{\mathrm{PGD}}, \text{within $t\leq T_{\mathrm{PGD}}$ GD updates}\},
\end{align}
for appropriately chosen values of $F_{\mathrm{PGD}}$ and $T_{\mathrm{PGD}}$. Evidently, any point in $\cR_{\mathrm{escape}}$ successfully escapes the strict saddle point $\z$ within the specified number of GD updates, whereas any point in $\cR_{\mathrm{stuck}}$ remains trapped and fails to do so. The key insight from \cite{jin2017escape} is that $\cR_{\mathrm{stuck}}$ forms a narrow band with a ``width'' proportional to $\epsilon/\sqrt{d}$. Consequently, even if $\z$ lies within this region, a single perturbation with a radius scaling as $\epsilon$ is likely to move the perturbed point into $\cR_{\mathrm{escape}}$ (see the right panel of Figure~\ref{fig::PGD_vs_IPGD}). Our main contribution is to demonstrate that the width of $\cR_{\mathrm{stuck}}$ can be reduced to $\gamma/\sqrt{d}$, for an arbitrarily small value $\gamma > 0$ (potentially much smaller than $\epsilon$), while only requiring $F_{\mathrm{PGD}}$ and $T_{\mathrm{PGD}}$ to scale poly-logarithmically in $1/\gamma$. Specifically, we show that this can be achieved by adjusting $F_{\mathrm{IPGD}}$ and $T_{\mathrm{IPGD}}$ as $F_{\mathrm{IPGD}}\leftarrow F_{\mathrm{PGD}}/\mathrm{poly}\log(1/\gamma)$ and $T_{\mathrm{IPGD}}\leftarrow T_{\mathrm{PGD}}\cdot\mathrm{poly}\log(1/\gamma)$. An immediate implication of this is that a perturbation radius of $\gamma$ is sufficient to guarantee an efficient escape (see the left panel of Figure~\ref{fig::PGD_vs_IPGD}).

\subsection{Implicit Regularization}\label{subsec::residual}
Even if the infinitesimal perturbations of IPGD do not cause a significant deviation from the implicit region $\cM$, there remains the possibility that the iterates may drift away from this region during the standard GD updates.
To characterize this deviation, we provide a control over the dynamics of the residual norm $\norm{\xperp_t}$. At a high level, 
a single GD update can be reformulated as:
\begin{equation}
    \begin{aligned}
	\x_{t+1} &= \x_t-\eta\nabla f\big(\x_t\big)=\xsharp_t-\eta\nabla f\big(\xsharp_t\big)+\xperp_t-\eta\left(\nabla f(\x_t)-\nabla f\big(\xsharp_t\big)\right)\\
	& = {\xsharp_t-\eta\nabla f\big(\xsharp_t\big)}
	\!+\!{\left(\I-\eta\nabla^2 f\big(\xsharp_t\big)\right)\xperp_t}+\eta H\big(\xperp_t\big), \text{ where } \norm{H\big(\xperp_t\big)} \leq \rho\norm{\xperp_t}^2.
\end{aligned}
\end{equation}
The last equality follows from the Hessian-Lipschitzness of $f$.
Therefore, the one-step dynamic of the residual norm $\norm{\xperp_t}$ can be controlled by applying the projection $\cP^\perp_{\cM}(\cdot)$ to both sides:
\begin{align}\label{eq:residual}
	\norm{\xperp_{t+1}} = \norm{\cP_{\cM}^\perp(\x_{t+1})} \leq \underbrace{\norm{\cP^\perp_{\cM}\left(\xsharp_t\!-\!\eta\nabla f\big(\xsharp_t\big)\right)\!}}_{\text{\bf signal deviation}}+\underbrace{\norm{\left(\I\!-\!\eta\nabla^2 f\big(\xsharp_t\big)\right)\xperp_t}}_{\text{\bf residual deviation}}+\eta\rho \norm{\xperp_t}^2.
\end{align}
In the above inequality, we used the simple property that, for any $\x,\y\in \bR^d$, we have $\norm{\cP^\perp_{\cM}(\x+\y)}\leq \norm{\cP^\perp_{\cM}(\x)}+\norm{\y}$ (see \Cref{lem::project}).
According to the above decomposition, for sufficiently small residual norm $\norm{\xperp_t}$, the dynamic of the residual norm is mainly governed by the {\it signal deviation} and {\it residual deviation}. 
The favorable landscape properties of $f$ within the implicit region $\cM$ can be leveraged to effectively control the signal deviation. In particular, the signal deviation is completely eliminated if the implicit region $\cM$ is \textit{closed under GD updates}.

\begin{assumption}[Closure of $\cM$ under GD updates]
	\label{assumption::GD-on-the-manifold}
	For any $\x\in \cM$ and stepsize $\eta>0$, we have $\x-\eta \nabla f(\x)\in \cM$. 
\end{assumption}

This assumption ensures that the GD iterates remain within the implicit region, provided the initial point lies within that region. While this assumption may appear restrictive at first and is certainly not universally valid for all choices of $\cM=\varphi^{-1}(\cD)$, we demonstrate in the following three examples that it holds, possibly with minor modifications, in low-rank matrix recovery, sparse recovery, and, more generally, models involving sparse basis function decomposition.

\begin{example}[Low-rank matrix recovery]\label{example::low-rank}
	\label{example::matrix-opt}
	\begin{sloppypar}
		Consider the symmetric low-rank matrix recovery problem discussed in Section~\ref{sec::intro} with the feasible region $\cD=\{\mTheta: \mTheta\succeq 0, \rank(\mTheta)\leq r\}$. With reparameterization map $\varphi: \mathbb{R}^{n\times r'}\to \mathbb{R}^{n\times n}$ with $r'\geq r$ and $\varphi(\X) = \X\X^\top$, the associated implicit region is defined as $\cM = \varphi^{-1}(\cD) = \{\X\in \bR^{n\times r'}: \rank(\X)\leq r\}$.
		For a general function $f(\X) = L(\X\X^\top)$, a single GD update can be written as:
		\begin{align}
			\X_+ = \X-\eta\nabla f(\X) = \left(\I-2\eta\nabla L\left(\X\X^\top\right)\right)\X.
		\end{align}
		It is easy to verify that if $\X\in \cM$, then $\X_+\in \cM$. A similar property can be established for  the asymmetric low-rank matrix recovery and its extension to low-rank tensor decomposition.
		
	\end{sloppypar}
\end{example}

In many applications, however, it may be preferable to work with a smaller implicit region $\cM \subset \varphi^{-1}(\cD)$, as requiring \Cref{assumption::GD-on-the-manifold} to hold over the entire set $\varphi^{-1}(\cD)$ could be too restrictive. Importantly, the reparameterization map $\varphi$ is typically not bijective. By carefully choosing a smaller implicit region $\cM$, we can guarantee that \Cref{assumption::GD-on-the-manifold} holds, while still ensuring that $\varphi(\cM) = \cD$. The following example illustrates this idea in the context of sparse recovery.
\begin{example}[Sparse recovery]\label{example::sparse-recovery}
	Consider the sparse recovery discussed in Section~\ref{sec::intro} with the feasible region $\cD = \{\mtheta\in \mathbb{R}^n: \mtheta_i=0,\ \text{if $\mtheta^\star_i=0$}\}$. Under the reparameterization map $\varphi(\u,\v) = \u \odot \v$, the implicit region is $\cM = \varphi^{-1}(\cD)= \{(\u,\v): \u_i= 0\ \text{or}\ \v_i=0\ \text{if}\ \mtheta^\star_i=0\}$. For a general function $f(\u,\v) = L(\u \odot \v)$, a single GD update takes the form
	\begin{align}
		\begin{bmatrix}
			\u_{+}\\
			\v_{+}
		\end{bmatrix} = \begin{bmatrix}
			\u-\eta\nabla_{\u} f(\u,\v)\\
			\v-\eta\nabla_{\v} f(\u,\v)
		\end{bmatrix} = \begin{bmatrix}
			\u-\eta\nabla L(\u\odot\v)\odot \v\\
			\v-\eta\nabla L(\u\odot\v)\odot \u
		\end{bmatrix}.
	\end{align}
	For this problem, $(\u,\v) \in \cM$ alone does not guarantee that $(\u_+,\v_+) \in \cM$. This issue can be resolved by restricting $\cM$ to the set of \textit{balanced} vectors, defined as $\cM_{\text{balanced}} = \{(\u,\v) : \u_i = \v_i = 0 \ \text{if} \ \mtheta^\star_i = 0\}$. This restricted implicit region continues to satisfy $\varphi(\cM_{\text{balanced}}) = \cD$, while also ensuring that $(\u,\v) \in \cM_{\text{balanced}}$ implies $(\u_+,\v_+) \in \cM_{\text{balanced}}$. In \Cref{sec::simulation}, we empirically show that IPGD exhibits implicit regularization toward solutions in $\cM_{\text{balanced}}$.
\end{example}

\begin{example}[Models with sparse basis function decomposition]\label{example::sparse-basis}
	The recent paper \cite{ma2023behind} demonstrates that, although the dynamics of GD can vary significantly across different deep neural network models, they tend to exhibit a sparse behavior when projected onto a suitable orthonormal function basis (see~\cite[Definition 1]{ma2023behind}). While the explicit characterization of this orthonormal function basis depends on the specific task, its mere existence is sufficient to define the corresponding implicit region as the set of parameters with a sparse representation---characterized by sparse basis coefficients---after projection onto this basis. It can be verified that this implicit region remains closed under GD updates, provided the GD trajectory satisfies a condition known as {\it gradient independence} (see~\cite[Theorem 1]{ma2023behind}). ~\cite{ma2023behind} illustrates that the gradient independence condition is approximately satisfied across a wide range of deep neural networks.
\end{example}

An arguably more challenging task is to control the residual deviation, which we discuss next. 
Consider the eigen-decomposition of the Hessian $\nabla^2 f(\x) = \V \D \V^\top$, where $\V \in \mathbb{R}^{d \times d}$ is an orthonormal matrix whose columns correspond to the eigenvectors of $\nabla^2 f(\x)$, and $\D \in \mathbb{R}^{d \times d}$ is a diagonal matrix containing the eigenvalues of $\nabla^2 f(\x)$. The decomposition of the Hessian into \textit{positive semi-definite} (PSD) and \textit{negative semi-definite} (NSD) components is defined as $\nabla^2 f(\x) = \nabla^2_+ f(\x)+\nabla^2_- f(\x)$, where $\nabla^2_+f(\x) = \V\D_+\V^\top$ and $\nabla^2_-f(\x) = \V\D_-\V^\top$. Here $\D_+$ is a diagonal matrix whose $(i,i)$-th entry is given by $\max\{0, \D_{ii}\}$. Similarly, $\D_-$ is defined as a matrix whose $(i,i)$-th entry is given by $\min\{0, \D_{ii}\}$.

\begin{definition}[Deviation rate]\label{def::deviation_rate}
	The \textbf{deviation rate} of $f$ at a point $\x$ from $\cM$ is defined as $r(\x) = r_-(\x)+3r_+(\x)$, where $r_-(\x)$ and $r_+(\x)$ are the {\bf negative deviation rate} and {\bf positive deviation rate} of $f$ at point $\x$ respectively, and are given as:
	\begin{align}
		r_-(\x)\! =\! \begin{cases}
			\frac{-1}{\norm{\xperp}^2}\inner{\xperp}{\nabla^2_+f(\xsharp)\xperp}& \text{if $\xperp\not=0;$}\\
			0& \text{if $\xperp=0$}
		\end{cases},
        r_+(\x) \!=\! \begin{cases}
			\frac{-1}{\norm{\xperp}^2}\inner{\xperp}{\nabla^2_-f(\xsharp)\xperp}& \text{if $\xperp\not=0$},\\
			0& \text{if $\xperp=0$}.
		\end{cases}
	\end{align}
	Accordingly, the \textbf{$\pmb{\tau}$-deviation rate} $R(\tau)$ of $f$ is defined as:
	\begin{align}
		R(\tau) = \sup_{\x\in\cN_{\cM}(\tau)} r(\x).
	\end{align}
\end{definition}

Clearly, we have $r_-(\x)\leq 0$ and $r_+(\x)\geq 0$ for every $\x\in \bR^d$.  Intuitively, the magnitudes of the negative and positive deviation rates at any point represent the effective negative and positive curvatures of the Hessian, respectively, evaluated at the signal $\xsharp$, along the direction of the residual $\xperp$. As we will show in the next proposition, the deviation rate $r(\x)$, which is a weighted sum of these two quantities, governs the growth rate of the residual norm.

\begin{proposition}[Residual norm dynamic]\label{prop::residual_norm}
	Suppose that $f$ is ${(L, \cM, \tau)}$-gradient-Lipschitz and ${(\rho, \cM, \tau)}$-Hessian-Lipschitz. Moreover, assume that the implicit region $\cM$ is closed under GD updates (Assumption~\ref{assumption::GD-on-the-manifold}) and $\norm{\xperp_t}\leq \tau$. Then, a single GD update $\x_{t+1} = \x_t-\eta\nabla f(\x_t)$ with any fixed stepsize $\eta\leq 1/L$ satisfies:
	\begin{align}\label{eq::one-step-residual}
		\norm{\xperp_{t+1}}\leq \left(1+\frac{\eta}{2}r(\x_t)+\frac{\eta\rho}{2}\norm{\xperp_t}\right)\norm{\xperp_t}.
	\end{align}
\end{proposition}
Based on the above proposition, the dynamic of the residual norm under GD updates is controlled by the deviation rate, expressed as $r(\x_t) = r_-(\x_t)+3r_+(\x_t)$. \Cref{prop::residual_norm} indicates that a large deviation rate can lead to rapid growth of the residual norm, while a small or negative deviation rate results in slower growth or even shrinkage. In Section~\ref{sec::application}, we demonstrate that this relationship is also observed in practice across a wide range of problems. 

To provide a better intuition on the role of \Cref{prop::residual_norm} in our analysis, let us, for simplicity, assume that $\norm{\xperp_t} \leq \tau \leq 2/(T\eta\rho)$ for all $0\leq t \leq T-1$. {We note that this assumption is not required in our formal analysis but is introduced here to provide a crisp intuition.} 
Unfolding the dynamics of $\norm{\xperp_t}$ for $t = 0,1,\dots, T-1$ using \eqref{eq::one-step-residual} yields:
\begin{align}\label{eq: x perp bound}
		\norm{\xperp_{T}}\leq \exp\left\{\frac{\eta}{2} \left(\sum_{t=0}^{T-1}r(\x_t)\right)+\frac{\eta\rho}{2}\left(\sum_{t=0}^{T-1}\norm{\xperp_t}\right)\right\}\norm{\xperp_0}\leq 3\exp\{\eta TR(\tau)\}\norm{\xperp_0}.
	\end{align}
The first inequality follows from $1+x\leq \exp\{x\}$ and the second inequality follows from the imposed bound on $\left\{\norm{\xperp_t}\right\}_{t=0}^{T-1}$ and the definition of the $\tau$-deviation rate $R(\tau)$. 
Assuming that $\norm{\xperp_{0}}$ is small, in order to ensure that $\norm{\xperp_{T}} = O\left(\norm{\xperp_{0}}\right)$, we require $R(\tau) \leq 1/(\eta T)$. As we will show later, the required number of iterations $T$ is primarily dictated by the behavior of the projected iterates onto the implicit region $\cM$, represented by the signal $\xsharp_t$. Consequently, a faster convergence of $\xsharp_t$ to an $\cM$-SOSP leads to a smaller value of $T$, which in turn results in a more relaxed bound on $R(\tau)$. 
Our next theorem formalizes this trade-off.

\begin{sloppypar}
    \begin{theorem}[Convergence of IPGD to $\cM$-SOSP]\label{thm::IPGD-MSOSP}
	Suppose that $f$ is ${(L, \cM, \tau)}$-gradient-Lipschitz and ${(\rho, \cM, \tau)}$-Hessian-Lipschitz with $\tau\geq \epsilon/L$. Moreover, assume that the following conditions are satisfied: 
    \begin{enumerate}
		\item[C1] The implicit region $\cM$ is closed under GD updates (Assumption~\ref{assumption::GD-on-the-manifold}).
        \item[C2] The initial point satisfies $\norm{\xperp_0}\leq O\left(\min\left\{\frac{\epsilon}{L}, \frac{1}{\eta\rho T}\right\}\right)$.
		\item[C3] The $\tau$-deviation rate $R(\tau)$ satisfies $R(\tau)= O\left(\frac{1}{\eta T}\right)$.
	\end{enumerate}
    Then, with probability at least $1-\chi$, IPGD (\Cref{alg::IPGD}) with any perturbation radius satisfying $\gamma = O\Big(\min\Big\{\frac{\epsilon^{5/2}}{\sqrt{\rho}L\Delta_f}, \frac{\epsilon^{7/2}}{\rho^{3/2}\Delta_f^2}\Big\}\cdot \log^{-7}\left(\frac{\rho d\Delta_f}{\chi\epsilon}\right)\Big)$ and stepsize $\eta\leq 1/L$ outputs an $\epsilon$-$\cM$-SOSP within $T = O\left(\frac{\Delta_f}{\eta\epsilon^2}\left(\log^4\left(\frac{1}{\gamma}\right)+\log^4\left(\frac{\rho d\Delta_f}{\chi\epsilon}\right)\right)\right)$ iterations.
\end{theorem}
\end{sloppypar}
Below, we discuss a few key insights.

\paragraph{Lipschitz conditions} The convergence of IPGD to an approximate $\cM$-SOSP only requires gradient- and Hessian-Lipschitz properties within an $(\epsilon/L)$-neighborhood of the implicit region $\cM$. This is a significant improvement over the global Lipschitz conditions required for PGD and for the convergence of IPGD to a general SOSP (\Cref{thm::IPGD}), especially when $\cM$ is low-dimensional. This relaxation is possible because, by controlling the residual norm via \Cref{prop::residual_norm}, the trajectory of IPGD can be confined to a small neighborhood around the implicit region.

\paragraph{Initialization} The convergence of IPGD to an approximate $\cM$-SOSP depends on the initial point being close to the implicit region, which corresponds to a small initial residual norm. As we mentioned earlier, while it may not always be feasible to determine such an initialization for an arbitrary implicit region $\cM$, it is often possible to ensure that $\bm{0} \in \cM$. In this case, a small or zero initialization naturally satisfies this condition. Notably, it is straightforward to verify that $\bm{0} \in \cM$ holds for the implicit regions discussed in \Cref{example::sparse-recovery,example::low-rank,example::sparse-basis}.

\paragraph{Deviation rate} The most restrictive condition in \Cref{thm::IPGD-MSOSP} is the upper bound imposed on the $\tau$-deviation rate. Noting the dependency of $T$ on $\epsilon$, this condition requires that, to achieve an $\epsilon$-$\cM$-SOSP, the $\tau$-deviation rate within an $\epsilon/L$-neighborhood of the implicit region must scale quadratically with $\epsilon$. While this condition is satisfied empirically in certain learning problems, rigorously proving it can be a daunting task, and in some cases, may even be impossible. 

Nevertheless, we demonstrate that this requirement can be significantly relaxed if the objective function $f$ exhibits additional structures within the implicit region $\cM$.

\subsection{Improved Convergence with Additional Structures}\label{subsec:IPGD-local}

A desirable geometric structure that can facilitate a faster convergence of IPGD is the so-called strict saddle property on the implicit region $\cM$, which we introduce next.

\begin{definition}[$\cM$-strict saddle property]
	\label{assumption::strict-saddle}
	We say the function $f$ satisfies the $(\bareg, \bareH,\bareM)$-$\cM$-strict saddle property for some parameters $\bareg, \bareH,\bareM>0$ if, for any $\x\in \cM$, at least one of the following conditions holds:
	\begin{itemize}
		\item $\bignorm{\nabla f(\x)}\geq \bareg$;
		\item $\lambda_{\min}(\nabla^2f(\x))\leq -\bareH$;
		\item $\dist(\x, \XM)\leq \bareM$, that is, there exists an exact $\cM$-SOSP $\z\in \cM$ satisfying $\nabla f(\z) = 0$ and $\nabla^2f(\z)\succeq 0$ such that $\norm{\x-\z}\leq \bareM$.
	\end{itemize}
\end{definition} 

Intuitively, the $\cM$-strict saddle property partitions the implicit region $\cM$ into three areas: (i) regions with large gradient norms; (ii) regions containing approximate SSPs with sufficiently negative curvature; and (iii) neighborhoods around exact $\cM$-SOSPs.

This property is akin to the global strict saddle property~\cite[Assumption A2]{ge2017no}, with the key distinction that it is only imposed on the implicit region $\cM$, whereas the global variant must hold for every $\x\in \bR^d$. While the global strict saddle property holds for {generic} functions~\cite[Theorem 2.9]{davis2019active}, the associated parameters $(\bareg, \bareH,\bareM)$ tend to degrade as the level of over-parameterization increases, leading to overly conservative convergence rate for the algorithm. Therefore, imposing this condition solely on the implicit region $\cM$---rather than the entire $\bR^{d}$---can lead to less conservative parameter choices.
For example, several problems, including matrix sensing, matrix completion, and tensor decomposition, have been shown to exhibit strict saddle properties with desirable parameters within their implicit regions \cite{ge2017no,ge2015escaping}.\footnote{The original results consider a slightly different setting, assuming the problem is exactly-parameterized, whereas we address the over-parameterized setting, where $\x \in \cM \subset \bR^d$ with $\dim(\cM) = k < d$. However, a simple lifting trick can establish the equivalence between these cases.} In contrast, in over-parameterized regimes, the global strict saddle property is conjectured to hold only with unfavorable parameters, resulting in poor convergence rates \cite[Page 5]{zhang2021sharp}.

We note that, in general, the iterates of IPGD do not necessarily belong to the implicit region $\cM$. Consequently, it is not immediately clear why the $\cM$-strict saddle property---a property defined solely on $\cM$---should facilitate the convergence of iterates that lie outside this region. Nonetheless, our next proposition shows that the strict saddle property, when restricted to $\cM$, is sufficient to improve the convergence of IPGD.

\begin{proposition}[Convergence of IPGD under strict saddle property]\label{cor::IPGD-SSP}
	Suppose that $f$ is ${(L, \cM, \tau)}$-gradient-Lipschitz and ${(\rho, \cM, \tau)}$-Hessian-Lipschitz with $\tau\geq \epsilon/L$, and satisfies the $(\bareg, \bareH,\bareM)$-$\cM$-strict saddle property. Furthermore, suppose that Conditions C1-C3 from \Cref{thm::IPGD-MSOSP} hold. Let $\bar\epsilon=\frac{1}{4}\min\left\{\bareg, \bareH^2/\rho\right\}$ and suppose $\epsilon\leq \min\left\{\frac{3}{4}\bareg, \frac{L}{2\rho}\bareH\right\}$. Then, with probability of at least $1-\chi$, IPGD (\Cref{alg::IPGD}) with any perturbation radius satisfying $\gamma = O\Big(\min\Big\{\frac{\bar\epsilon^{3/2}\cdot\epsilon}{\sqrt{\rho}L\Delta_f}, \frac{\bar\epsilon^{7/2}}{\rho^{3/2}\Delta_f^2}\Big\}\cdot \log^{-7}\left(\frac{\rho d\Delta_f}{\chi\bar\epsilon}\right)\Big)$ and stepsize $\eta\leq 1/L$ outputs a point $\x_T$ satisfying: 
    \begin{align}
        \norm{\xperp_T}\leq \frac{\epsilon}{L},\qquad \dist\left(\xsharp_T, \XM\right)\leq \bareM,
    \end{align}
    within $T = O\left(\frac{\Delta_f}{\eta\bar\epsilon^2}\left(\log^4\left(\frac{1}{\gamma}\right)+\log^4\left(\frac{\rho d\Delta_f}{\chi\bar\epsilon}\right)\right)\right)$ iterations.
\end{proposition}

\Cref{cor::IPGD-SSP} establishes that, under the $\cM$-strict saddle property, setting the gradient threshold $\epsilon$ in IPGD to $\bar{\epsilon} = \frac{1}{4}\min\left\{\bareg, \bareH^2/\rho\right\}$ ensures that the output lies within an $\bareM$-neighborhood of an exact $\cM$-SOSP. 
Notably, such a guarantee is unattainable without the strict saddle property and cannot be derived solely from \Cref{thm::IPGD-MSOSP}. In particular, while \Cref{thm::IPGD-MSOSP} demonstrates that the output of IPGD is approximately an $\cM$-SOSP, it does not establish a bound on the distance between the obtained point and an exact $\cM$-SOSP.

Additionally, \Cref{cor::IPGD-SSP} offers another distinct advantage: under the assumption that $\bar{\epsilon} = \Theta(1)$, the required bound on the deviation rate simplifies to $R(\tau) = \tilde{O}\left(1/\Delta_f\right)$, which is a significant relaxation compared to the stricter requirement $R(\tau) = \tilde{O}\left(\epsilon^2/\Delta_f\right)$ in \Cref{thm::IPGD-MSOSP}. As we will show later, both $\bar{\epsilon} = \Theta(1)$ and $R(\tau) = \tilde{O}\left(1/\Delta_f\right)$ can be provably established for the over-parameterized matrix sensing.

Finally, we note that while the $\cM$-strict saddle property is a desirable property, it is not sufficient to guarantee convergence to a reasonably accurate $\cM$-SOSP, particularly in settings where $\bareM = \Omega(1)$. In such cases, IPGD may still require the same conditions as those in \Cref{thm::IPGD-MSOSP} to ensure the output is an $\epsilon$-$\cM$-SOSP with $\epsilon\ll \bareM$, even under the strict saddle property. To truly accelerate the convergence of IPGD, an additional condition, akin to strong convexity within a local region of an $\cM$-SOSP, is necessary. Unfortunately, however, strong convexity---even locally---is rarely satisfied in over-parameterized problems. A more relaxed notion is a regularity property, which we impose only on the implicit region $\cM$.

\begin{definition}[$\cM$-regularity property]
\label{assumption::regularity}
We say the function $f$ satisfies the $(\alpha, \beta, \zeta)$-$\cM$-regularity property for some parameters $\alpha, \beta, \zeta>0$ if, for any $\x\in \cN_{\XM}(\zeta)\cap\cM$, we have:
\begin{equation}
    \inner{\nabla f(\x)}{\x-\cP_{\XM}(\x)} \geq \frac{\alpha}{2}\bignorm{\x-\cP_{\XM}(\x)}^2+\frac{1}{2 \beta}\bignorm{\nabla f(\x)}^2.
\end{equation}
\end{definition}

The $(\alpha, \beta, \zeta)$-$\cM$-regularity property is a weaker condition than local strong convexity. In fact, it can be shown that gradient Lipschitzness and one-point strong convexity together imply this property (see \cite[Section II.B]{chi2019nonconvex} for a proof). 
However, unlike local strong convexity---which typically fails to hold even in the exactly-parameterized regime\footnote{Strong convexity provably fails within any neighborhood if the solution set forms a connected set. This occurs in various models, including sparse recovery and low-rank matrix recovery.}---the $(\alpha, \beta, \zeta)$-$\cM$-regularity property is generally easier to establish. For example, it has been verified for several nonconvex optimization problems, such as matrix factorization \cite{jin2017escape}, matrix sensing \cite{bhojanapalli2016global}, and matrix completion \cite{zheng2016convergence} in the exactly-parameterized regime.

Intuitively, $(\alpha, \beta, \zeta)$-$\cM$-regularity property entails that, within a small neighborhood of any $\cM$-SOSP, the gradient $\nabla f(\x)$ is positively correlated with the error vector $\x-\cP_{\XM}(\x)$. This implies that vanilla GD, initialized at a point exactly within $\cM$, converges efficiently to an $\cM$-SOSP. Our next result shows that even if the initial point is not exactly in $\cM$, vanilla GD can still exploit the $(\alpha, \beta, \zeta)$-$\cM$-regularity property and converge linearly to an approximate $\cM$-SOSP.
\begin{sloppypar}
  \begin{proposition}[Local linear convergence of GD under regularity property]\label{thm::regularity}
	Suppose that $f$ is ${(L, \cM, \tau)}$-gradient-Lipschitz and ${(\rho, \cM, \tau)}$-Hessian-Lipschitz with $\tau\geq \epsilon/L$, and satisfies the $(\alpha, \beta, \zeta)$-$\cM$-regularity property. Suppose that the initial point $\x_0$ satisfies $\norm{\xperp_0}= O\left(\frac{\eta\alpha\epsilon}{L}\right)$ and  $\dist(\xsharp_0,\XM)\leq \zeta$. Furthermore, suppose that Conditions C1 and C3 from \Cref{thm::IPGD-MSOSP} hold. Then, GD with stepsize $\eta\leq \frac{1}{10\max\{\alpha, \beta, L\}}$ outputs a solution $\x_T$ satisfying 
    \begin{equation}
        \dist\left(\x_T, \cX_{\cM}\right)\leq \frac{\epsilon}{L}
    \end{equation}
    within $T = O\left(\frac{1}{\eta\alpha}{\log\left(\frac{L\zeta}{\epsilon}\right)}\right)$ iterations.
\end{proposition}  
\end{sloppypar}

Note that the improvement in the iteration count---from $\poly(1/\epsilon)$ to $\log(1/\epsilon)$---enabled by the regularity condition relaxes the dependence of the deviation rate $R(\tau)$ on $\epsilon$ to $\log^{-1}(1/\epsilon)$. This is significantly less restrictive than the requirement in \Cref{thm::IPGD-MSOSP}, where $R(\tau)$ must scale as $\epsilon^2$. 

Finally, by leveraging both the strict saddle property and the $\cM$-regularity property, we can establish linear convergence to an $\cM$-SOSP. Specifically, using the strict saddle property, we first run IPGD to obtain a point within a $\zeta$-neighborhood of an $\cM$-SOSP. Then, by exploiting the regularity condition, we apply vanilla GD to achieve linear convergence to an $\cM$-SOSP. This combined approach, which we call {\it IPGD with local improvement} (IPGD+), is outlined in \Cref{alg::tsp-gd-local-refinement}.

\begin{algorithm}[H]
\caption{{IPGD with local improvement}: IPGD+$(\gamma, \epsilon, \eta, \x_0, \rho, \Delta_f, \chi, T')$}\label{alg::tsp-gd-local-refinement}
	\begin{algorithmic}[1]
        \State $\x_{0}\leftarrow \textsc{IPGD}(\gamma, \epsilon, \eta, \x_0, \rho, \Delta_f, \chi)$
        \For {$t=0, 1, \cdots, T'$}
		\State $\x_{t+1}\leftarrow \x_t-\eta \nabla f\left(\x_t\right)$
        \EndFor
    \end{algorithmic}
\end{algorithm}

\begin{sloppypar}
    \begin{theorem}[Global near-linear convergence of IPGD+]\label{thm::linear}
	Suppose that $f$ is ${(L, \cM, \tau)}$-gradient-Lipschitz and ${(\rho, \cM, \tau)}$-Hessian-Lipschitz with $\tau\geq \epsilon/L$. Suppose that $f$ satisfies the $(\alpha, \beta, \zeta)$-$\cM$-regularity property and $(\bareg, \bareH,\bareM)$-$\cM$-strict saddle property with $\bareM\leq \zeta$. Suppose that the initial point satisfies $\norm{\xperp_0}=O\left(\min\left\{\eta\alpha\cdot\frac{\epsilon}{L}, \frac{1}{\eta\rho T}\right\}\right)$. Suppose that Conditions C1 and C3 from \Cref{thm::IPGD-MSOSP} hold. Let $\bar\epsilon= O\left(\min\left\{\bareg, \frac{\bareH^2}{\rho}\right\}\right)$ and $T' = O\left(\frac{1}{\eta\alpha}{\log\left(\frac{L\zeta}{\epsilon}\right)}\right)$. Then, with probability of at least $1-\chi$, IPGD+ (\Cref{alg::tsp-gd-local-refinement}) with any perturbation radius $\gamma = O\Big(\min\Big\{\frac{\bar\epsilon^{3/2}\cdot\epsilon}{\sqrt{\rho}L\Delta_f}, \frac{\bar\epsilon^{7/2}}{\rho^{3/2}\Delta_f^2}\Big\}\cdot \log^{-7}\left(\frac{\rho d\Delta_f}{\chi\bar\epsilon}\right)\Big)$ and stepsize $\eta\leq \frac{1}{10\max\{\alpha, \beta, L\}}$ outputs a solution satisfying
    \begin{equation}
        \dist\left(\x_T, \cX_{\cM}\right)\leq \frac{\epsilon}{L}
    \end{equation}
    within a total of $T = O\left(\frac{\Delta_f}{\eta\bar\epsilon^2}\left(\log^4\left(\frac{1}{\gamma}\right)+\log^4\left(\frac{\rho d\Delta_f}{\chi\bar\epsilon}\right)\right)+\frac{1}{\eta\alpha}\log\left(\frac{L\zeta}{\epsilon}\right)\right)$ iterations.
\end{theorem}
\end{sloppypar}
To better interpret the above result, let us focus solely on its dependency on the final desired error $\epsilon$. \Cref{thm::linear} implies that, under the $\cM$-strict saddle and $\cM$-regularity properties, IPGD+ reaches an $\epsilon$-neighborhood of an $\cM$-SOSP in at most $T = O(\log^4(1/\epsilon))$ iterations, provided that the initial point satisfies $\norm{\xperp_0} = O(\epsilon)$ and the deviation rate satisfies $R(\tau) = O(\log^{-4}(1/\epsilon))$. 
As we will show in \Cref{sec::application}, the conditions of \Cref{thm::linear} are provably satisfied for the over-parameterized matrix sensing problem.

We note that IPGD+ is similar to the PGD with local improvement (referred to as PGDli) introduced in \cite{jin2017escape}. However, IPGD+ has a slower convergence rate by a factor of $\log^3(1/\epsilon)$. This slowdown stems from the use of a smaller perturbation radius, which is necessary to enable implicit regularization and ensure that the iterates remain close to the implicit region $\cM$. 

Next, we provide the proofs for our main results. To this end, in~\Cref{subsec::proof-SOSP}, we first prove the convergence of IPGD to an approximate SOSP, thereby establishing~\Cref{thm::IPGD}. Next, in~\Cref{subsec::proof-MSOSP}, we establish the implicit regularization of IPGD and use this property to show the convergence of IPGD to an approximate $\cM$-SOSP, ultimately proving~\Cref{thm::IPGD-MSOSP}. Finally, in~\Cref{subsec::proof-local}, we refine our results by proving the improved convergence guarantees of~\Cref{thm::linear} under additional local structural assumptions.

\section{Establishing Convergence of IPGD to an Approximate SOSP}
\label{subsec::proof-SOSP}
To prove the convergence of IPGD, we carefully refine the analysis of PGD from~\cite{jin2017escape,jin2021nonconvex}, extending it to handle infinitesimal perturbations. Our analysis proceeds in two parts:
\begin{itemize}
    \item {\bf Bounding the iteration count:} We derive an upper bound on the number of iterations that IPGD can take before meeting its termination criterion, characterizing its convergence rate.
    \item {\bf Establishing the error guarantee:} We demonstrate that, upon termination, the returned solution satisfies the conditions for an approximate SOSP.
\end{itemize}

To provide a bound on the iteration count, let us define $N_{\mathrm{large}}$ as the number of iterations where the gradient norm satisfies $\norm{\nabla f(\x_t)}> G$. Moreover, let $N_{\mathrm{perturb}}$ be the number of iterations in which a perturbation is applied to $\x_t$ (Line 6 of \Cref{alg::IPGD}). Formally, these are given by:
\begin{align*}
        N_{\mathrm{large}} = \left|\left\{t: \norm{\nabla f(\x_t)}> G \right\}\right|, \qquad N_{\mathrm{perturb}} = \left|\left\{t: \text{perturbation is added to $\x_t$} \right\}\right|.
\end{align*}
The following lemma provides an upper bound on the total number of iterations of the algorithm in terms of $N_{\mathrm{large}}$ and $N_{\mathrm{perturb}}$.
\begin{lemma}\label{lem::number}
    Let $t=T$ denote the last iteration of IPGD (\Cref{alg::IPGD}). We have
    \begin{align}
        T\leq N_{\mathrm{large}}+N_{\mathrm{perturb}}\cdot T_{\mathrm{escape}}.
    \end{align}
\end{lemma}
\begin{proof}
    At every iteration $t$, IPGD takes one of the following steps:
    \begin{itemize}
        \item If $\norm{\nabla f(\x_t)}> G$ and no perturbation has been applied in the past $T_{\mathrm{escape}}$ iterations, IPGD performs a single GD update.
        \item If $\norm{\nabla f(\x_t)}\leq G$ and no perturbation has been applied in the past $T_{\mathrm{escape}}$ iterations,  IPGD applies a perturbation to $\x_t$, followed by $T_{\mathrm{escape}}$ GD updates.
    \end{itemize}
    By combining these observations with the definitions of $N_{\mathrm{large}}$ and $N_{\mathrm{perturb}}$, we obtain $T\leq N_{\mathrm{large}}+N_{\mathrm{perturb}}\cdot T_{\mathrm{escape}}$, thereby completing the f.
 \end{proof}
Based on the above lemma, to bound the iteration count, it suffices to provide separate bounds on $N_{\mathrm{large}}$ and $N_{\mathrm{perturb}}$. Towards this goal, we utilize a classic descent lemma for GD updates.
    \begin{lemma}[Descent lemma for GD updates; \protect{\cite[Section 1.2.3]{nesterov2018lectures}}]
    \label{lem::gd-lemma}
    Suppose that the conditions of \Cref{thm::IPGD} are satisfied.
        For the GD update $\x_{t+1}=\x_t-\eta\nabla f(\x_t)$ with $\eta\leq \frac{1}{L}$, we have:
        \begin{equation}
            f(\x_{t+1})-f(\x_t)\leq-\frac{1}{2}\eta\bignorm{\nabla f(\x_t)}^2.
        \end{equation}
    \end{lemma}
    Equipped with~\Cref{lem::number} and~\Cref{lem::gd-lemma}, we are now ready to establish an upper bound on the maximum number of iterations of IPGD:
    \begin{proposition}[Bounding the iteration count of IPGD]\label{prop::number}
        Under the conditions of \Cref{thm::IPGD}, IPGD terminates within $T \leq \frac{4C^2\Delta_f}{\eta\epsilon^2}\left(\log^{4}\left(\frac{1}{\gamma}\right)+\log^4\left(\frac{\rho d\Delta_f}{\chi\epsilon}\right)\right)$ iterations.
    \end{proposition}
    \begin{proof}
        Indeed, if $\norm{\nabla f(\x_t)}> G$, \Cref{lem::gd-lemma} implies that a single GD update $\x_{t+1}=\x_t-\eta\nabla f(\x_t)$ reduces the objective by at least
        \begin{equation}
            f(\x_{t+1})-f(\x_t)\leq-\frac{1}{2}\eta\bignorm{\nabla f(\x_t)}^2\leq -\frac{1}{2}\eta G^2 = -\frac{1}{2}\eta\left(\frac{1}{C^2}\cdot \frac{\epsilon^2}{\log^4(1/\gamma)}\right) = -\frac{\eta\epsilon^2}{2C^2\log^4(1/\gamma)} 
        \end{equation}
        where the first equality is due to the definition of $G$ in \Cref{alg::IPGD}. This implies that:
        \begin{align}
            N_{\mathrm{large}}\leq \frac{\Delta_f}{\frac{\eta\epsilon^2}{2C^2\log^4(1/\gamma)}} = \frac{2C^2\Delta_f\log^4(1/\gamma)}{\eta\epsilon^2}.
        \end{align}
        Next, we provide an upper bound on $N_{\mathrm{perturb}}$. Let $t_k$ denote the iteration in which the $k$-th perturbation is applied to $\x_{t_k}$. For every $1\leq k\leq N_{\mathrm{perturb}}-1$, we must have
        \begin{align}
            f(\x_{t_k+T_{\mathrm{escape}}})-f(\z_{t_k})<-F = -\frac{1}{C}\cdot\frac{\epsilon^{3/2}}{\sqrt{\rho}\left(\log^{3}\left(\frac{1}{\gamma}\right)+\log^3\left(\frac{\rho d\Delta_f}{\chi\epsilon}\right)\right)},
        \end{align}
    as otherwise, the algorithm would have terminated within the first $N_{\mathrm{perturb}}-1$ perturbations, contradicting the definition of $N_{\mathrm{perturb}}$. The equality follows from the definition of $F$ in \Cref{alg::IPGD}.
    This implies that
    \begin{align}\label{eq::N_perturb}
            N_{\mathrm{perturb}}\leq \frac{\Delta_f}{\frac{1}{C}\cdot\frac{\epsilon^{3/2}}{\sqrt{\rho}\log^3(1/\gamma)}} = \frac{C\sqrt{\rho}\Delta_f\left(\log^{3}\left(\frac{1}{\gamma}\right)+\log^3\left(\frac{\rho d\Delta_f}{\chi\epsilon}\right)\right)}{\epsilon^{3/2}}.
        \end{align}
        Invoking \Cref{lem::number}, we obtain
        \begin{equation}
            \begin{aligned}
            T&\leq N_{\mathrm{large}}+N_{\mathrm{perturb}}\cdot T_{\mathrm{escape}}\\
            &\leq \frac{2C^2\Delta_f\log^4(1/\gamma)}{\eta\epsilon^2}+\frac{C\sqrt{\rho}\Delta_f\left(\log^{3}\left(\frac{1}{\gamma}\right)+\log^3\left(\frac{\rho d\Delta_f}{\chi\epsilon}\right)\right)}{\epsilon^{3/2}}\cdot \frac{C\left(\log\left(\frac{1}{\gamma}\right)+\log\left(\frac{\rho d\Delta_f}{\chi\epsilon}\right)\right)}{\eta\sqrt{\rho\epsilon}}\\
            &\leq \frac{4C^2\Delta_f\left(\log^{4}\left(\frac{1}{\gamma}\right)+\log^4\left(\frac{\rho d\Delta_f}{\chi\epsilon}\right)\right)}{\eta\epsilon^2}.
        \end{aligned}
        \end{equation}
        This completes the proof.
     \end{proof}

Next, we establish that the returned solution is indeed an approximate SOSP. To this goal, we rely on the following key proposition, which encapsulates the main novelty of our proof.

\begin{proposition}[Escaping strict saddle points with infinitesimal perturbation]
\label{prop::escaping-saddle}
    \begin{sloppypar}
        Suppose that the conditions of \Cref{thm::IPGD} are satisfied. Moreover, suppose the input $\z_t$ satisfies $\bignorm{\nabla f(\z_t)} \leq G$ and $\lambda_{\min}\left(\nabla^2 f(\z_t)\right)<-\sqrt{\rho\epsilon}$. Then, with probability at least $1-\frac{\chi}{N_{\mathrm{perturb}}}$, a single perturbation $\x_t = \z_t+\bxi_t$, where $\bxi_t\sim \mathrm{Unif}(\bB(\gamma))$, with perturbation radius satisfying $\gamma = O\left(\min\left\{\sqrt{\frac{\epsilon}{\rho}}\log^{-3}\left(\frac{\rho d\Delta_f}{\chi\epsilon}\right), \frac{\epsilon}{L}\right\}\right)$ followed by $T_{\mathrm{escape}}$ iterations of GD updates $\x_{t+1} = \x_t-\eta\nabla f(\x_t)$ with $\eta\leq \frac{1}{L}$ satisfies
    \end{sloppypar}
    \begin{equation}
        f(\x_{t+T_{\mathrm{escape}}})-f(\z_t)< -F,
    \end{equation}
    where $F:=\frac{\epsilon^{3/2}}{4608\sqrt{\rho}\left(\log^{3}\left(\frac{1}{\gamma}\right)+\log^3\left(\frac{\rho d\Delta_f}{\chi\epsilon}\right)\right)}$.
\end{proposition}
Before presenting the proof of the above proposition, we first show how it can be used to complete the proof of~\Cref{thm::IPGD}.
\begin{proof}[Proof of \Cref{thm::IPGD}]
    \Cref{prop::number} readily establishes the presented upper bound on the iteration count. Moreover, the termination criterion ensures that $\bignorm{\nabla f(\z_{T_{\mathrm{noise}}})} \leq G$ and $f(\x_{T_{\mathrm{noise}}+T_{\mathrm{escape}}})-f(\z_{T_{\mathrm{noise}}})\geq -F$. By \Cref{prop::escaping-saddle}, it follows that $\lambda_{\min}\left(\nabla^2 f(\z_{T_{\mathrm{noise}}})\right)\geq-\sqrt{\rho\epsilon}$ with probability at least $1-\frac{\chi}{N_{\mathrm{perturb}}}$. Since the algorithm can visit at most $N_{\mathrm{perturb}}$, a simple union bound implies that $\z_{T_{\mathrm{noise}}}$ is an $\epsilon$-SOSP with probability at least $1-N_{\mathrm{perturb}}\cdot\frac{\chi}{N_{\mathrm{perturb}}}=1-\chi$, thereby completing the proof.
 \end{proof}

Now, we proceed with the proof of \Cref{prop::escaping-saddle}. 
We first introduce the following lemma.
\begin{lemma}[Improve or localize, adapted from Lemma~5.4 in \cite{jin2021nonconvex}]
\label{lem::improve-or-localize}
    Let $\x_0$ be any initial point. Suppose the GD updates $\x_{t+1} = \x_t - \eta \nabla f(\x_t)$ are generated with step size $\eta \leq \frac{1}{L}$. Then, for all $t \geq 0$, we have
\begin{equation}
    \bignorm{\x_t - \x_0} \leq \sqrt{2 \eta t \left( f(\x_0) - f(\x_t) \right)}.
\end{equation}
\end{lemma}

Our next result is at the core of the guaranteed escape of IPGD from an approximate SSP. To explain the intuition behind this result, let $\z$ be an approximate SSP, satisfying $\norm{\nabla f(\z)}\leq \epsilon$ and $\lambda_{\min}(\nabla^2f(\z))<-\sqrt{\rho\epsilon}$. Let $\v$ be the eigenvector corresponding to the most negative eigenvalue of $\nabla^2f(\z)$. Consider any two initial points $\x_0, \x_0'$ such that $\x_0, \x_0'\in \bB_{\z}(\gamma)$ and $\x_0-\x_0' = \omega\cdot \v$ for some $\omega>0$. Our next result demonstrates that GD, when initialized at one of these points, must decrease the objective function rapidly, enabling an escape from $\z$. 

\begin{lemma}
\label{lem::two-escape-one}
Let $\z$ be an approximate SSP satisfying $\norm{\nabla f(\z)}\leq \epsilon$ and $\lambda_{\min}(\nabla^2f(\z))<-\sqrt{\rho\epsilon}$. Let $\v$ be the eigenvector corresponding to the most negative eigenvalue of $\nabla^2f(\z)$.
    Consider two points $\x_0, \x_0'$ satisfying $\x_0, \x_0'\in \bB_{\z}(\gamma)$ and $\x_0-\x_0' = \omega\cdot \v$ for some $0<\omega\leq 2\gamma$. Then, the two sequences $\{\x_t\}_{t=0}^T$ and $\{\x_t'\}_{t=0}^T$ generated by GD updates with constant stepsize $\eta\leq \frac{1}{L}$ satisfy the following property within $T\leq \frac{2}{\eta\sqrt{\rho\epsilon}}\log\left(\frac{1}{\omega}\right)$ 
    iterations:
    \begin{equation}
        \min\left\{f(\x_T), f(\x_T')\right\}-f(\z)\leq -\frac{\epsilon^{3/2}}{576\sqrt{\rho}}\log^{-3}\left(\frac{1}{\omega}\right)+\epsilon\gamma + \frac{L}{2}\gamma^2.
    \end{equation}
\end{lemma}

\begin{proof}
To streamline the presentation, let us define the following quantity:
\begin{align}
    \phi:=\frac{\sqrt{\epsilon}}{6\sqrt{\rho}\log\left(\frac{1}{\omega}\right)}.
\end{align}
     Let $T$ be the first iteration satisfying $\max\{\bignorm{\x_T-\z}, \bignorm{\x_T'-\z}\}\geq \phi$. First, we show that $T\leq \frac{2}{\eta\sqrt{\rho\epsilon}}\log\left(\frac{1}{\omega}\right)$. Evidently, we have $\max\{\bignorm{\x_t-\z}, \bignorm{\x_t'-\z}\}< \phi$ for all $0\leq t\leq T-1$. We denote $\vdelta_t=\x_t-\x_t'$, $ \mH=\nabla^2 f(\z)$, and $\mDelta_t=\int_0^1\left(\nabla^2 f\left(\x_t^{\prime}+s\vdelta_t\right)-\mH\right) \mathrm{d} s$. Then, for every $0\leq t\leq T-1$: 
    \begin{equation}
        \begin{aligned}
            \bignorm{\mDelta_t}\leq \int_0^1 \rho\bignorm{\x_t^{\prime}+s\vdelta_t - \z} \mathrm{d} s\leq \rho \max\left\{\bignorm{\x_t-\z}, \bignorm{\x_t'-\z}\right\}< \rho \phi,
        \end{aligned}
    \end{equation}
    where we use the Hessian Lipschitzness in the first inequality. By the Fundamental Theorem of Calculus, we have $\nabla f\left(\x_{t-1}\right)-\nabla f\left(\x_{t-1}'\right) = \mH\vdelta_{t-1}+\mDelta_{t-1}\vdelta_{t-1}$,
    which in turn implies
\begin{equation}
    \begin{aligned}
\vdelta_{t} & =\vdelta_{t-1}-\eta\left(\nabla f\left(\x_{t-1}\right)-\nabla f\left(\x_{t-1}'\right)\right)=(\I-\eta \mH) \vdelta_{t-1}-\eta \mDelta_{t-1} \vdelta_{t-1} \\
& =(\I-\eta \mH)^{t} \vdelta_0-\eta \sum_{s=0}^{t-1}(\I-\eta \mH)^{t-s-1} \mDelta_s \vdelta_s.
\end{aligned}
\end{equation}
Upon defining $p_t:=\bignorm{(\I-\eta \mH)^{t} \vdelta_0}$ and $q_t:=\bignorm{\eta \sum_{s=0}^{t-1}(\I-\eta \mH)^{t-s-1} \mDelta_s \vdelta_s}$, we have
\begin{equation}
    p_t-q_t\leq\bignorm{\vdelta_{t}}\leq p_t+q_t.
\end{equation}
Therefore, to control $\bignorm{\vdelta_{t}}$, it suffices to control 
$p_t$ and $q_t$ separately. Since $\x_0-\x_0' = \omega\cdot \v$ where $\v$ is the eigenvector corresponding to the most negative eigenvalue of $\nabla^2f(\z)$, we have 
\begin{equation}
    p_t= (1-\eta\lambda_{\min})^t \omega\geq (1+\eta\sqrt{\rho\epsilon})^t \omega,
\end{equation}
where we denote $\lambda_{\min}:=\lambda_{\min}(\nabla^2f(\z))\leq -\sqrt{\rho\epsilon}$.
Next, we use induction to show that $q_t\leq \frac{1}{2}p_t$ for all $0\leq t\leq T$. Indeed, the base case for $t=0$ holds since $q_0=0$. Now, suppose that $q_s\leq \frac{1}{2}p_s$ for all $0\leq s\leq t-1$.
Then, an upper bound on $q_t$ can be derived as follows
\begin{equation}
    \begin{aligned}
        q_t&\leq \eta \sum_{s=0}^{t-1}\bignorm{ (\I-\eta \mH)^{t-s-1}}\bignorm{\mDelta_s}\bignorm{\vdelta_s}\leq\eta \sum_{s=0}^{t-1}(1-\eta \lambda_{\min})^{t-s-1}\bignorm{\mDelta_s}\bignorm{\vdelta_s}\\
        &\stackrel{(a)}{\leq} \rho\phi\cdot\eta \sum_{s=0}^{t-1}(1-\eta \lambda_{\min})^{t-s-1}\cdot \frac{3}{2}p_s\leq  \frac{3}{2}\eta \rho\phi\cdot\sum_{s=0}^{t-1}(1-\eta \lambda_{\min})^{t-s-1}\cdot(1-\eta\lambda_{\min})^s p_0\\
        &=\frac{3}{2}\eta t \rho\phi\cdot p_{t-1}\stackrel{(b)}{\leq} \frac{1}{2}p_t.
    \end{aligned}
\end{equation}
Here in $(a)$, we use $\bignorm{\mDelta_t}\leq \rho \phi$ and the assumption $\bignorm{\vdelta_{s}}\leq p_s+q_s\leq \frac{3}{2}p_s$ for all $0\leq s\leq t-1$.
In $(b)$, we use the fact that $t\leq\frac{2}{\eta\sqrt{\rho\epsilon}}\log\left(\frac{1}{\omega}\right)$ and $\phi=\frac{\sqrt{\epsilon}}{6\sqrt{\rho}}\log^{-1}\left(\frac{1}{\omega}\right)\leq \frac{1}{4}$, which implies 
\begin{equation}
    \eta t \rho\phi\leq \eta \cdot \frac{2}{\eta\sqrt{\rho\epsilon}}\log\left(\frac{4\phi}{\omega}\right)\cdot \rho\cdot \frac{\sqrt{\epsilon}}{6\sqrt{\rho}}\log^{-1}\left(\frac{1}{\omega}\right)\leq \frac{1}{3}.
\end{equation}
Provided with this result, we further have 
\begin{equation}
    \bignorm{\vdelta_t}\geq p_t-q_t\geq \frac{1}{2}p_t\geq\frac{1}{2}(1+\eta\sqrt{\rho\epsilon})^t \omega.
\end{equation}
Therefore, within $T=\log\left(\frac{4\phi}{\omega}\right)/\log\left(1+\eta\sqrt{\rho\epsilon}\right)$ iterations, we must have $\bignorm{\vdelta_T}\geq 2\phi$. Note that $\phi\leq \frac{1}{4}$ and $\log(1+\eta\sqrt{\rho\epsilon})\geq \frac{\eta\sqrt{\rho\epsilon}}{1+\eta\sqrt{\rho\epsilon}}\geq \frac{1}{2}\eta\sqrt{\rho\epsilon}$. Hence, we have $T\leq \frac{2}{\eta\sqrt{\rho\epsilon}}\log\left(\frac{1}{\omega}\right)$. This implies
\begin{equation}
    \max\left\{\bignorm{\x_T-\z}, \bignorm{\x_T'-\z}\right\}\geq \frac{\bignorm{\x_T-\x_T'}}{2}=\frac{1}{2}\norm{\vdelta_T}\geq \phi.
\end{equation}
Now, without loss of generality, assume $\bignorm{\x_T-\z}\geq \phi$. Then, we have $\bignorm{\x_T-\x_0}\geq \phi -\gamma\geq \frac{1}{2}\phi$, and \Cref{lem::improve-or-localize} implies that 
    \begin{equation}
        f(\x_T)-f(\x_0)\leq -\frac{1}{8\eta T} \cdot \phi^2\leq -\frac{\epsilon^{3/2}}{576\sqrt{\rho}}\log^{-3}\left(\frac{1}{\omega}\right),
    \end{equation}
    where the last inequality follows from the provided upper bound on $T$ and the definition of $\phi$.
    On the other hand, we have
    \begin{equation}
        f(\x_0)-f(\z)\leq \inner{\nabla f(\z)}{\x_0-\z}+\frac{L}{2}\norm{\x_0-\z}^2\leq \epsilon\gamma + \frac{L}{2}\gamma^2.
    \end{equation}
    Combining the above two inequalities leads to
    \begin{equation}
        f(\x_T)-f(\z)\leq -\frac{\epsilon^{3/2}}{576\sqrt{\rho}}\log^{-3}\left(\frac{1}{\omega}\right)+\epsilon\gamma + \frac{L}{2}\gamma^2,
    \end{equation}
which is the desired result.
 \end{proof}

To proceed, we define the ``stuck region'' as follows
\begin{equation}
    \begin{aligned}
        \cX_{\mathrm{stuck}}:=\left\{\x_0\in \bB_{\z}(\gamma):f(\x_T)-f(\z)\geq -F\text{ where }\x_{t+1}=\x_t-\eta\nabla f(\x_t), \forall 0\leq t\leq T-1\right\}.
    \end{aligned}
\end{equation}
Here $F=\frac{\epsilon^{3/2}}{4608\sqrt{\rho}\left(\log^{3}\left(\frac{1}{\gamma}\right)+\log^3\left(\frac{\rho d\Delta_f}{\chi\epsilon}\right)\right)}$ and $T=\frac{2}{\eta\sqrt{\rho\epsilon}}\log(\frac{1}{\omega})$. Our next lemma shows that the probability of $\x_0\sim \mathrm{Unif}(\bB_{\z}(\gamma))$ falling within the stuck region is small.
\begin{sloppypar}
    \begin{lemma}
\label{lem::stuck-prob}
Consider the setting of \Cref{lem::two-escape-one}.
    Suppose that $\omega=\sqrt{\frac{2\pi}{d+1}}\frac{\chi}{N_{\mathrm{perturb}}}\gamma$ and $\gamma=O\left(\min\left\{\sqrt{\frac{\epsilon}{\rho}}\log^{-4}\left(\frac{\rho d\Delta_f}{\chi\epsilon}\right), \frac{\epsilon}{L}\right\}\right)$. Then, for $\x_0\sim \mathrm{Unif}(\bB_{\z}(\gamma))$, we have 
    \begin{equation}
        \bP\left(\x_0 \in \cX_{\mathrm{stuck}}\right)\leq \frac{\chi}{N_{\mathrm{perturb}}}.
    \end{equation}
\end{lemma}
\end{sloppypar}
\begin{proof}
Using \Cref{lem::two-escape-one}, we show that for any two GD sequences with initial points $\x_0, \x_0'\in \bB_{\z}(\gamma)$ satisfying $\x_0-\x_0'=\omega \bv$, at least one of them lies outside the stuck region. We have
\begin{equation}
    \begin{aligned}
        \min\left\{f(\x_T), f(\x_T')\right\}-f(\z)&\leq -\frac{\epsilon^{3/2}}{576\sqrt{\rho}}\log^{-3}\left(\frac{1}{\omega}\right)+\epsilon\gamma + \frac{L}{2}\gamma^2\\
    &\stackrel{(a)}{\leq} -\frac{\epsilon^{3/2}}{2304\sqrt{\rho} \left(\log^{3}\left(\frac{1}{\gamma}\right)+\log^3\left(\frac{\rho d\Delta_f}{\chi\epsilon}\right)\right)}+\epsilon\gamma + \frac{L}{2}\gamma^2\\
        &\stackrel{(b)}{\leq} -\frac{\epsilon^{3/2}}{4608\sqrt{\rho}\left(\log^{3}\left(\frac{1}{\gamma}\right)+\log^3\left(\frac{\rho d\Delta_f}{\chi\epsilon}\right)\right)}=-F.
    \end{aligned}
\end{equation}
\begin{sloppypar}
    \noindent Here in $(a)$, we use the fact that $\omega=\sqrt{\frac{2\pi}{d+1}}\frac{\chi}{N_{\mathrm{perturb}}}\gamma$ and \Cref{eq::N_perturb}.
Moreover, $(b)$ follows from the facts that $L\gamma^2=O(\epsilon\gamma)$ since $\gamma=O\left(\frac{\epsilon}{L}\right)$ and $\epsilon\gamma=O\bigg(\frac{\epsilon^{3/2}}{\sqrt{\rho} \left(\log^{3}\left(\frac{1}{\gamma}\right)+\log^3\left(\frac{\rho d\Delta_f}{\chi\epsilon}\right)\right)}\bigg)$ since $\gamma = O\left(\sqrt{\frac{\epsilon}{\rho}}\log^{-3}\left(\frac{\rho d\Delta_f}{\chi\epsilon}\right)\right)$. This implies that the width of the stuck region in the direction $\bv$ is at most $\omega$. Hence, we can control $\bP\left(\x_0 \in \cX_{\mathrm{stuck}}\right)$ as follows
\end{sloppypar}
    \begin{equation}
        \begin{aligned}
            \bP\left(\x_0 \in \cX_{\mathrm{stuck}}\right)=\frac{\operatorname{vol}\left(\cX_{\text{stuck}}\right)}{\operatorname{vol}\left(\bB^d(\gamma)\right)} \leq \frac{\omega \cdot \operatorname{vol}\left(\bB^{d-1}(\gamma)\right)}{\operatorname{vol}\left(\bB^d(\gamma)\right)}=\frac{\omega \Gamma(d / 2+1)}{ \sqrt{\pi} \gamma\Gamma(d / 2+1 / 2)} &\leq \frac{\omega}{\gamma} \sqrt{\frac{d+1}{2\pi}}\leq \frac{\chi}{N_{\mathrm{perturb}}}.
        \end{aligned}
    \end{equation}
    Here, $\operatorname{vol}(\cX)$ denotes the volume of the set $\cX$. For the $d$-dimensional unit ball, we have $\operatorname{vol}(\bB^d(1))=\frac{\pi^{d/2}}{\Gamma\left(\frac{d}{2}+1\right)}$, where $\Gamma(\cdot)$ is the Gamma function defined as $\Gamma\left(\frac{d}{2}+1\right)=\sqrt{\pi}\left(d-\frac{1}{2}\right)\cdot \left(d-\frac{3}{2}\right)\cdot\ldots\cdot \frac{1}{2}$. In the second inequality, we use $\frac{\Gamma(x+1)}{\Gamma(x+1 / 2)}<\sqrt{x+\frac{1}{2}}$ for all $x\geq 0$.
 \end{proof}

We are now ready to provide the proof of \Cref{prop::escaping-saddle}.
\begin{proof}[Proof of \Cref{prop::escaping-saddle}]
    The proof is immediately implied by \Cref{lem::stuck-prob}, after noting that $T=\frac{2}{\eta\sqrt{\rho\epsilon}}\log(\frac{1}{\omega})\leq \frac{4}{\eta\sqrt{\rho\epsilon}}\left(\log\left(\frac{1}{\gamma}\right)+\log\left(\frac{\rho d\Delta_f}{\chi\epsilon}\right)\right)\leq T_{\mathrm{escape}}$.
 \end{proof}

\section{Establishing Implicit Regularization of IPGD}
\label{subsec::proof-MSOSP}
To establish the implicit regularization of IPGD, we first show how the deviation rate (\Cref{def::deviation_rate}) governs the residual norm dynamic, as outlined in \Cref{prop::residual_norm}. First, we present a basic property of projection operator $\cP^\perp_{\cM}(\cdot)$:
\begin{lemma}\label{lem::project}
    For any $\x,\y\in \mathbb{R}^d$, we have $\norm{\cP^\perp_{\cM}(\x+\y)}\leq \norm{\cP^\perp_{\cM}(\x)}+\norm{\y}$.
\end{lemma}
\begin{proof}
    Defining $\u^\star = \argmin_{\u\in \cM}\norm{\u-(\x+\y)}$ and $\v^\star = \argmin_{\v\in \cM}\norm{\v-\x}$, we have:
    \begin{align}
        \norm{\cP^\perp_{\cM}(\x+\y)} = \norm{\u^\star-(\x+\y)}\leq \norm{\v^\star-(\x+\y)}\leq \norm{\v^\star-\x}+\norm{\y}=\norm{\cP^\perp_{\cM}(\x)}+\norm{\y},
    \end{align}
    which completes the proof.
 \end{proof}

Our next lemma provides a useful control over the dynamic of $\norm{\x_{t+1}^{\perp}}$:

\begin{lemma}[Upper bound on the residual norm]\label{lem::upperbound}
    Suppose that conditions of \Cref{prop::residual_norm} are satisfied. Then, we have:
    \begin{align}
    \norm{\x_{t+1}^{\perp}}\leq \norm{\left(\I-\eta\nabla^2f\big(\xsharp_t\big)\right)\x_t^{\perp}}+\frac{\eta\rho}{2}\norm{\x_t^{\perp}}^2.
    \end{align}
\end{lemma}
\begin{proof}
    We first decompose $\x_{t+1}$ as follows:
    \begin{equation}
        \x_{t+1}=\xsharp_t-\eta\nabla f\big(\xsharp_t\big)+\xperp_t-\eta\left(\nabla f(\x_t)-\nabla f\big(\xsharp_t\big)\right).
    \end{equation}
    Applying the projection $\cP^{\perp}_{\cM}(\cdot)$ to both sides, we obtain
    \begin{equation}
        \begin{aligned}
    \norm{\xperp_{t+1}} &= \norm{\cP^{\perp}_{\cM}(\x_{t+1})}\\
    &= \norm{\cP^{\perp}_{\cM}\left(\xsharp_t-\eta\nabla f\big(\xsharp_t\big)+\xperp_t-\eta\left(\nabla f(\x_t)-\nabla f\big(\xsharp_t\big)\right)\right)}\\
    &\leq \norm{\cP^{\perp}_{\cM}\left(\xsharp_t-\eta\nabla f\big(\xsharp_t\big)\right)}+\norm{\xperp_t-\eta\left(\nabla f(\x_t)-\nabla f\big(\xsharp_t\big)\right)},
\end{aligned}
    \end{equation}
where the inequality follows from \Cref{lem::project}. Since $\xsharp_t\in \cM$ and $\cM$ is assumed to be closed under GD updates, we have $\norm{\cP^{\perp}_{\cM}\left(\xsharp_t-\eta\nabla f\big(\xsharp_t\big)\right)} = 0$, which implies:
    \begin{equation}
        \norm{\x_{t+1}^{\perp}}\leq \norm{\x_t^{\perp}+\eta\left(\nabla f\big(\x_t^{\sharp}\big)-\nabla f(\x_t)\right)}.
        \label{eq::46}
    \end{equation}
    Applying the Fundamental Theorem of Calculus to the above inequality leads to:
    \begin{equation}
    \label{eq::16-main}
        \begin{aligned}
            \norm{\x_{t+1}^{\perp}}&\leq \norm{\left(\I-\eta\int_0^1\nabla^2f\left(\xsharp_t+\tau \x_s^{\perp}\right)\vd \tau\right)\x_t^{\perp}}\\
            &\leq \norm{\left(\I-\eta\nabla^2f\big(\xsharp_t\big)\right)\x_t^{\perp}}+\eta \left(\int_0^1\norm{\nabla^2f\left(\xsharp_t+\tau \x_s^{\perp}\right)-\nabla^2f\big(\xsharp_t\big)}\vd\tau\right) \norm{\x_t^{\perp}}\\
            &\leq  \norm{\left(\I-\eta\nabla^2f\big(\xsharp_t\big)\right)\x_t^{\perp}}+\frac{\eta\rho}{2}\norm{\x_t^{\perp}}^2.
        \end{aligned}
    \end{equation}
    where in the last inequality, we use the local Hessian-Lipschitzness of $f$.
 \end{proof}
Equipped with \Cref{lem::upperbound}, we are now ready to present the proof of \Cref{prop::residual_norm}.

\begin{proof}[Proof of \Cref{prop::residual_norm}]
Due to \Cref{lem::upperbound}, the result trivially holds if $\norm{\xperp_t}=0$. Therefore, we assume $\norm{\xperp_t}>0$. 
    Let $\lambda_1\geq \lambda_2\geq \dots\geq \lambda_d$ denote the eigenvalues of $\nabla^2 f(\xsharp)$ with the corresponding eigenvectors $\v_1,\v_2,\dots, \v_d$. One can write:
    \begin{equation}
        \begin{aligned}
        \norm{\left(\I-\eta\nabla^2f\big(\xsharp_t\big)\right)\x_t^{\perp}}^2 &= \norm{\sum_{i=1}^d (1-\eta\lambda_i)\inner{\v_i}{\xperp_t}\v_i}^2=  \sum_{i=1}^d (1-\eta\lambda_i)^2\inner{\v_i}{\xperp_t}^2\\
        &=  \sum_{i=1}^d \left(1-2\eta\lambda_i+\eta^2\lambda_i^2\right)\inner{\v_i}{\xperp_t}^2.
    \end{aligned}
    \end{equation}
    On the other hand, since $\eta\leq \frac{1}{L}$, we have $\eta^2\lambda_i^2\leq \eta|\lambda_i|$, which implies:
    \begin{align}
        \begin{cases}
            1-2\eta\lambda_i+\eta^2\lambda_i^2\leq 1-\eta\lambda_i, & \text{if $\lambda_i>0$},\\
            1-2\eta\lambda_i+\eta^2\lambda_i^2\leq 1-3\eta\lambda_i, & \text{if $\lambda_i<0$}.
        \end{cases}
    \end{align}
    Therefore, we have
    \begin{equation}
    \begin{aligned}
        &\norm{\left(\I-\eta\nabla^2f\big(\xsharp_t\big)\right)\x_t^{\perp}}^2\leq  \sum_{i: \lambda_i>0} \left(1-\eta\lambda_i\right)\inner{\v_i}{\xperp_t}^2+\sum_{i: \lambda_i<0} \left(1-3\eta\lambda_i\right)\inner{\v_i}{\xperp_t}^2\\
        &= \norm{\xperp_t}^2+\eta\left(-\sum_{i: \lambda_i>0} \lambda_i\inner{\v_i}{\xperp_t}^2-\sum_{i: \lambda_i<0} 3\lambda_i\inner{\v_i}{\xperp_t}^2\right)\\
        &\stackrel{(a)}{=} \norm{\xperp_t}^2+\eta\left(-\inner{\xperp_t}{\nabla^2_+f\big(\xsharp_t\big)\xperp_t}-3\inner{\xperp_t}{\nabla^2_-f\big(\xsharp_t\big)\xperp_t}\right)\\
        &= \norm{\xperp_t}^2+\eta\left(\frac{-1}{\norm{\xperp_t}^2}\inner{\xperp_t}{\nabla^2_+f\big(\xsharp_t\big)\xperp_t}+\frac{-3}{\norm{\xperp_t}^2}\inner{\xperp_t}{\nabla^2_-f\big(\xsharp_t\big)\xperp_t}\right)\norm{\xperp_t}^2\\
        &\stackrel{(b)}{=} \norm{\xperp_t}^2+\eta\left(r_-(\x_t)+3r_+(\x_t)\right)\norm{\xperp_t}^2\stackrel{(c)}{=} (1+\eta r(\x_t))\norm{\xperp_t}^2.
    \end{aligned}
    \end{equation}
    Here, $(a)$ follows from the definitions of the PSD and NSD components of the Hessian. Additionally, $(b)$ and $(c)$ follow from the definition of the deviation rate. This implies that:
\begin{align}
    \norm{\left(\I-\eta\nabla^2f\big(\xsharp_t\big)\right)\x_t^{\perp}}\leq \sqrt{1+\eta r(\x_t)}\norm{\xperp_t}\leq \left(1+\frac{\eta}{2}r(\x_t)\right)\norm{\xperp_t}.
\end{align}
This inequality combined with \Cref{lem::upperbound} completes the proof.
 \end{proof}
We now proceed with the proof of \Cref{thm::IPGD-MSOSP}. To this end, we first revisit the steps of IPGD. At each iteration, IPGD performs one of the following two operations: 1) GD update, or 2) addition of a small perturbation. For the GD update, we can directly apply \Cref{prop::residual_norm} to control the residual norm. For the small perturbation step, we can only invoke a conservative bound $\norm{\x_{t+1}^{\perp}}\leq \norm{\x_t^{\perp}}+\gamma$. The next lemma combines these two cases and establishes a uniform upper bound on the residual norm across all iterations via an auxiliary sequence.

\begin{lemma}\label{lem::a_sequence}
    Suppose that the conditions of \Cref{prop::residual_norm} are satisfied for all $0\leq t\leq T$. Define the sequence:
    \begin{align}
        a_{t+1}=\left(1+\frac{\eta}{2} R(\tau)  + \frac{\eta\rho}{2} a_t\right)a_t, \quad \text{ with }a_0=\norm{\x_0^{\perp}}+N_{\mathrm{perturb}}\cdot\gamma.
    \end{align}
    Then, the residual norm satisfies $\norm{\xperp_{t}}\leq a_t$ for every $t=0,1,\dots, T$. 
\end{lemma}
\begin{proof}
Let $t_1\leq\cdots\leq t_{N_{\mathrm{perturb}}}$ be the iteration indices at which the perturbations are added. Define the series $\{b_t\}_{t=0}^{T}$ with $b_0=\norm{\x_0^{\perp}}$ and
    \begin{equation}
        b_{t+1}=\begin{cases}
\left(1+\frac{\eta}{2} R(\tau)  + \frac{\eta\rho}{2} b_t\right)b_t & t\not\in\{t_1,\dots, t_{N_{\mathrm{perturb}}}\}, \\
\left(1+\frac{\eta}{2} R(\tau)  + \frac{\eta\rho}{2} (b_t+\gamma)\right)(b_t+\gamma) & t\in\{t_1,\dots, t_{N_{\mathrm{perturb}}}\}.
\end{cases}
    \end{equation}
    According to \Cref{prop::residual_norm}, we have $\norm{\xperp_t}\leq b_t$ for every $t$. Therefore, it suffices to show that $b_t\leq a_t$ for every $t$. To this end, we argue that shifting the perturbations to earlier iterations can only increase the elements of the sequence.
    Without loss of generality, assume that $t_{N_{\mathrm{perturb}}} > 0$; otherwise, the claim follows immediately. Now, consider a modified sequence $\{b'_t\}_{t=0}^{T}$, where the index of the last perturbation is shifted one step earlier. Specifically, we define $b'_t=b_t$ for all $t\leq t_{N_{\mathrm{perturb}}}-1$, and apply the last perturbation at iteration $t_{N_{\mathrm{perturb}}}-1$. 
    By doing so, we have 
    \begin{equation}
        \begin{aligned}
            b'_{t_{N_{\mathrm{perturb}}}}&=\left(1+\frac{\eta}{2} R(\tau)  + \frac{\eta\rho}{2} \left(b_{t_{N_{\mathrm{perturb}}}-1}+\gamma\right)\right)\left(b_{t_{N_{\mathrm{perturb}}}-1}+\gamma\right)\\
            &\geq \left(1+\frac{\eta}{2} R(\tau)  + \frac{\eta\rho}{2} b_{t_{N_{\mathrm{perturb}}}-1}\right)b_{t_{N_{\mathrm{perturb}}}-1}+\gamma=b_{t_{N_{\mathrm{perturb}}}}+\gamma.
        \end{aligned}
    \end{equation}
    This implies that
    \begin{equation}
        \begin{aligned}
            b'_{t_{N_{\mathrm{perturb}}}+1}&=\left(1+\frac{\eta}{2} R(\tau)  + \frac{\eta\rho}{2} b'_{t_{N_{\mathrm{perturb}}}}\right)b'_{t_{N_{\mathrm{perturb}}}}\\
            &\geq\left(1+\frac{\eta}{2} R(\tau)  + \frac{\eta\rho}{2} \left(b_{t_{N_{\mathrm{perturb}}}}+\gamma\right)\right)\left(b_{t_{N_{\mathrm{perturb}}}}+\gamma\right)\geq b_{t_{N_{\mathrm{perturb}}}+1}.
        \end{aligned}
    \end{equation}
    Note that both sequences $\{b_t\}_{t=0}^T$ and $\{b'_t\}_{t=0}^T$ follow the GD updates for $t\geq t_{N_{\mathrm{perturb}}}+1$ (i.e., no perturbations are added after iteration $t_{N_{\mathrm{perturb}}}$). Therefore, we must have $b'_t\geq b_t$ for $t\geq t_{N_{\mathrm{perturb}}}+1$. Combined with the fact that $b'_t=b_t$ for all $t\leq t_{N_{\mathrm{perturb}}}-1$, it follows that $b'_t\geq b_t$ for $0\leq t\leq T$. 
    By repeating this procedure and shifting all the perturbations to iteration $t=0$, we establish that $a_t\geq b_t$ for $0\leq t\leq T$, thereby completing the proof.
 \end{proof}
Our next lemma provides an upper bound for the sequence $\{a_{t}\}_{t=0}^T$. 
\begin{lemma}
\label{lem::induction-a-s}
    Suppose that $T = O\left(\frac{\Delta_f}{\eta\epsilon^2}\left(\log^{4}\left(\frac{1}{\gamma}\right)+\log^4\left(\frac{\rho d\Delta_f}{\chi\epsilon}\right)\right)\right)$ and $R(\tau)= O\left(\frac{1}{\eta T}\right)$. Then, for the sequence defined in \Cref{lem::a_sequence}, we have
    \begin{align}
        a_{t}\leq a_0\left(1+\eta R(\tau)+\frac{1}{T}\right)^t, \quad t=0,1,\dots, T.
    \end{align}
\end{lemma}
\begin{proof}[Proof of \Cref{lem::induction-a-s}]
Note that $a_{t+1}= a_0\prod_{s=0}^{t}\left(1+\eta  R(\tau)  + \eta\rho a_s\right)$. Hence, it suffices to prove $a_t\leq \frac{1}{\eta\rho T}$ for all $0\leq t\leq T-1$.
We prove this by induction on $t$. 
    It is easy to verify the base case $t=0$. Now, assume that $a_{s}\leq \frac{1}{\eta\rho T}$ for all $0\leq s\leq t$. 
    Then, we have
\begin{equation}
    \begin{aligned}
        a_{t+1} &= a_0\prod_{s=0}^{t}\left(1+\eta  R(\tau)  + \eta\rho a_s\right)\leq a_0\left(1+\eta  R(\tau)+\frac{1}{T}\right)^t\stackrel{(a)}{\leq} a_0\exp\left\{\eta R(\tau)t+\frac{t}{T}\right\}\stackrel{(b)}{\leq} 10a_0\stackrel{(c)}{\leq} \frac{1}{\eta\rho T},
    \end{aligned}
\end{equation}
where $(a)$ is due to $1+x\leq \exp\{x\}$ for all $x\in \bR$. Moreover, $(b)$ is due to the condition that $R(\tau)=O\left(\frac{1}{\eta T}\right)$. Lastly, $(c)$ follows from
\begin{align}
    a_0=\norm{\x_0^{\perp}}+N_{\mathrm{perturb}}\cdot\gamma\leq \frac{1}{20\eta\rho T}+\frac{C\sqrt{\rho}\Delta_f\left(\log^{3}\left(\frac{1}{\gamma}\right)+\log^3\left(\frac{\rho d\Delta_f}{\chi\epsilon}\right)\right)}{\epsilon^{3/2}}\cdot\gamma\leq \frac{1}{10\eta\rho T}
\end{align}
where we use the facts that $T = O\left(\frac{\Delta_f}{\eta\epsilon^2}\left(\log^{4}\left(\frac{1}{\gamma}\right)+\log^4\left(\frac{\rho d\Delta_f}{\chi\epsilon}\right)\right)\right)$, and $\gamma=O\left(\frac{\epsilon^{7/2}}{\rho^{3/2}\Delta_f^2}\log^{-7}\left(\frac{\rho d\Delta_f}{\chi\epsilon}\right)\right)$. This completes the proof.
 \end{proof}
Now we are ready to provide the proof of \Cref{thm::IPGD-MSOSP}.
\begin{proof}[Proof of \Cref{thm::IPGD-MSOSP}] First, according to \Cref{lem::a_sequence,lem::induction-a-s}, we have that for all $0\leq t\leq T$,
\begin{equation}
    \norm{\x_t^\perp}\leq a_t\leq \left(\norm{\x_0^{\perp}}+N_{\mathrm{perturb}}\cdot\gamma\right)\left(1+\eta R(\tau)+\frac{1}{T}\right)^t\leq \frac{\epsilon}{L},
\end{equation}
where the final inequality uses $\norm{\x_0^{\perp}}\leq \frac{\epsilon}{10L}$, $N_{\mathrm{perturb}}=\frac{C\sqrt{\rho}\Delta_f\left(\log^{3}\left(\frac{1}{\gamma}\right)+\log^3\left(\frac{\rho d\Delta_f}{\chi\epsilon}\right)\right)}{\epsilon^{3/2}}$, $\gamma=O\left(\frac{\epsilon^{5/2}}{\sqrt{\rho}L\Delta_f}\log^{-3}\left(\frac{\rho d\Delta_f}{\chi\epsilon}\right)\right)$, and $R(\tau)=O(\frac{1}{\eta T})$. Given $\norm{\x_t^\perp} \leq \frac{\epsilon}{L}$ for all $0 \leq t \leq T$, the gradient and Hessian of the objective function remain Lipschitz continuous along the entire trajectory $\{\x_t\}_{t=0}^T$. Thus, we can apply \Cref{thm::IPGD}, which guarantees that, with probability at least $1 - \chi$, IPGD outputs a point $\x_T$ that is an $\epsilon$-SOSP after $T = O\left(\frac{\Delta_f}{\eta\epsilon^2}\left(\log^{4}\left(\frac{1}{\gamma}\right)+\log^4\left(\frac{\rho d\Delta_f}{\chi\epsilon}\right)\right)\right)$ iterations. Combining the above two results ensures that $\x_T$ is an $\epsilon$-$\cM$-SOSP.
 \end{proof}

\section{Establishing Improved Convergence of IPGD with Additional  Structures}
\label{subsec::proof-local}
In this section, we present the proofs of the improved convergence results for IPGD stated in \Cref{subsec:IPGD-local}. We start with the proof of \Cref{cor::IPGD-SSP}.
\begin{proof}[Proof of \Cref{cor::IPGD-SSP}]
    The proof of \Cref{thm::IPGD-MSOSP} can be easily modified to show that, with probability at least $1-\chi$, after $T=O\left(\frac{\Delta_f}{\eta\bar\epsilon^2}\left(\log^4(1/\gamma)+\log^4\left(\frac{\rho d\Delta_f}{\chi\bar\epsilon}\right)\right)\right)$ iterations, IPGD returns a solution $\x_T$ that satisfies
    \begin{equation}
        \norm{\x_T^\perp}\leq \frac{\epsilon}{L}, \quad \x_T\text{ is an }\bar\epsilon\text{-$\cM$-SOSP}.
    \end{equation}
    The details of this modification are omitted for brevity. Hence, we have 
    \begin{equation}
        \norm{\nabla f\big(\x_T^\sharp\big)}\leq \norm{\nabla f\left(\x_T\right)} + \norm{\nabla f\left(\x_T\right)-\nabla f\big(\x_T^\sharp\big)}\leq \bar\epsilon + L \norm{\x_T^\perp}\leq \bar\epsilon+\epsilon \leq \bareg,
    \end{equation}
    and 
    \begin{equation}
        \lambda_{\min}\left(\nabla^2f\big(\x_T^\sharp\big)\right)\geq \lambda_{\min}\left(\nabla^2f(\x_T)\right)-\norm{\nabla^2 f\left(\x_T\right)-\nabla^2 f\big(\x_T^\sharp\big)}\geq -\sqrt{\rho\bar\epsilon}-\rho\cdot  \frac{\epsilon}{L}\geq -\bareH.
    \end{equation}
    Here we use the local gradient and Hessian Lipschitzness, which can be used since $\norm{\x_T^\perp}\leq \frac{\epsilon}{L}$, $\bar\epsilon=\frac{1}{4}\min\{\bareg, \bareH^2/\rho\}$, and $\epsilon\leq \min\left\{\frac{3}{4}\bareg, \frac{L}{2\rho}\bareH\right\}$.
    Finally, since the objective function satisfies the $\cM$-strict saddle property, we have 
    \begin{align}
        \dist\left(\xsharp_T, \XM\right)\leq \bareM.
    \end{align}
    This completes the proof.
 \end{proof}

   Next, we turn to prove \Cref{thm::regularity}. To this goal, we first present the following lemma, which shows that the $\cM$-regularity can be utilized to control the deviation of the iterates from the implicit region.
   \begin{lemma}
       \label{lem::strong-local-structure}
       Suppose that the conditions of \Cref{thm::regularity} are satisfied. Consider the GD updates $\x_{t+1}=\x_t-\eta\nabla f(\x_t)$ for $0\leq t\leq T$. If $\x^{\sharp}_t\in \cN_{\XM}(\zeta)$, then, we have 
       \begin{equation}
           \dist\left(\x_{t+1}, \XM\right)\leq \left(1-\frac{\eta\alpha}{4}\right)\dist\left(\x_t, \XM\right)+\left(2+\eta L\right)\norm{\x_t^\perp}.
       \end{equation}
   \end{lemma}
   \begin{proof}[Proof of \Cref{lem::strong-local-structure}]
   We have  
    \begin{equation}
        \dist\left(\x_{t+1}, \XM\right) \leq\underbrace{\dist\left(\x_{t}^\sharp-\eta\nabla f\big(\x_t^\sharp\big), \XM\right)}_{(\rom{1})}+\underbrace{\norm{\x_{t}^\sharp-\eta\nabla f\big(\x_t^\sharp\big)-\x_{t+1}}}_{(\rom{2})}.
        \label{eq::prop-2}
    \end{equation}
To control $(\rom{1})$, we write
\begin{equation}
    \begin{aligned}
        (\rom{1})^2&=\min_{\x^\star\in \XM}\norm{\x_{t}^\sharp-\eta\nabla f\big(\x_t^\sharp\big)-\x^\star}^2\leq \norm{\x_{t}^\sharp-\eta\nabla f\big(\x_t^\sharp\big)-\proj_{\XM}\big(\x_t^\sharp\big)}^2\\
        &=\norm{\x_{t}^\sharp-\proj_{\XM}\big(\x_t^\sharp\big)}^2+\eta\inner{\proj_{\XM}\big(\x_t^\sharp\big)-\x_t^\sharp}{\nabla f\big(\x_t^\sharp\big)}+\eta^2\norm{\nabla f\big(\x_t^\sharp\big)}^2\\
        &\stackrel{(a)}{\leq}\left(1-\frac{\eta\alpha}{2}\right)\norm{\x_{t}^\sharp-\proj_{\XM}\big(\x_t^\sharp\big)}^2+\left(\eta^2-\frac{\eta}{2\beta}\right)\norm{\nabla f\big(\x_t^\sharp\big)}^2\stackrel{(b)}{\leq}\left(1-\frac{\eta\alpha}{2}\right)\norm{\x_{t}^\sharp-\proj_{\XM}\big(\x_t^\sharp\big)}^2,
    \end{aligned}
\end{equation}
where in $(a)$, we use the $\cM$-regularity since $\x^{\sharp}_t\in \cN_{\XM}(\zeta)\cap\cM$, and in $(b)$, we use $\eta \leq \frac{1}{2\beta}$. Therefore, we derive
\begin{equation}
    (\rom{1})\leq \sqrt{1-\frac{\eta\alpha}{2}}\dist\left(\x_t^\sharp, \XM\right)\leq \left(1-\frac{\eta\alpha}{4}\right)\left(\dist\left(\x_t, \XM\right)+\norm{\x_t^\perp}\right).
\end{equation}
In the last inequality, we use the triangle inequality and the basic inequality $\sqrt{1-x}\leq 1-\frac{1}{2}x$ for all $x\leq 1$.
Next, we control $(\rom{2})$ as follows
\begin{equation}
    \begin{aligned}
        (\rom{2})&\leq \norm{\x_t^\perp}+\eta \norm{\nabla f\big(\x_t^\sharp\big)-\nabla f(\x_t)}\leq \left(1+\eta L\right)\norm{\x_t^\perp}.
    \end{aligned}
\end{equation}
Combining these two inequalities, we obtain
\begin{equation}
    \begin{aligned}
        \dist\left(\x_{t+1}, \XM\right)&\leq \left(1-\frac{\eta\alpha}{4}\right)\left(\dist\left(\x_t, \XM\right)+\norm{\x_t^\perp}\right)+\left(1+\eta L\right)\norm{\x_t^\perp}\\
        &\leq \left(1-\frac{\eta\alpha}{4}\right)\dist\left(\x_t, \XM\right)+\left(2+\eta L\right)\norm{\x_t^\perp}.
    \end{aligned}
\end{equation}
This completes the proof.
 \end{proof}

The lemma introduces a useful inequality that will be used in the proof of \Cref{thm::regularity}.
\begin{lemma}[Adapted from Lemma~47 in \cite{ma2024convergence}]
    \label{lem::appendix-series-ineq}
    Suppose that the series $\{a_t\}_{t=0}^{\infty}$ satisfies $a_{t+1}\leq Ax_t+B, \forall t\geq 0$ where $A>0, A\neq 1$ and $a_0+\frac{B}{A-1}\geq 0$. Then, we have 
        \begin{equation}
            a_t\leq A^t\left(a_0+\frac{B}{A-1}\right)-\frac{B}{A-1}.
        \end{equation}
\end{lemma}
We are now ready to present the proof of \Cref{thm::regularity}.
\begin{proof}[Proof of \Cref{thm::regularity}]
First, notice that, for every $0\leq t\leq T$,
\begin{equation}
    \begin{aligned}
        \norm{\x_t^{\perp}}&\stackrel{(a)}{\leq} \left(1+\eta R(\tau)+\frac{1}{T}\right)^T\norm{\x_0^{\perp}} \stackrel{(b)}{=}O\left( \norm{\xperp_0}\right)\stackrel{(c)}{=} O\left(\frac{\eta\alpha\epsilon}{L}\right).
    \end{aligned}
    \label{eq::control-psi}
\end{equation}
Here, $(a)$ follows exactly the same steps as the proof of \Cref{thm::IPGD-MSOSP}; in $(b)$, we use the assumption $R(\tau)=O\left(\frac{1}{\eta T}\right)$ and the inequality $\left(1+x/T\right)^T\leq \exp(x)$; and $(c)$ is due to our assumption $\norm{\xperp_0}=O\left(\frac{\eta\alpha\epsilon}{L}\right)$. 
Now, without loss of generality, assume $\zeta \geq \frac{\epsilon}{2L}$. Otherwise, we immediately obtain $\dist\left(\x_0, \XM\right)\leq \dist\big(\x_0^\sharp, \XM\big)+\norm{\x_0^\perp}\leq \frac{\epsilon}{2L}+\frac{\eta\alpha\epsilon}{L}\leq \frac{\epsilon}{L}$ as desired.
    Next, we prove by induction that $\x^{\sharp}_t \in \cN_{\XM}(\zeta)$ for all $0 \leq t \leq T$. The induction holds at initialization by construction. Suppose that $\x^{\sharp}_t\in \cN_{\XM}(\zeta)$. For iteration $t+1$, according to \Cref{lem::strong-local-structure}, we have 
    \begin{equation}
        \begin{aligned}
            \dist\left(\x_{t+1}, \XM\right)&\leq \left(1-\frac{\eta\alpha}{4}\right)\dist\left(\x_t, \XM\right)+\left(2+\eta L\right)\norm{\x_t^\perp}\\
            &\leq \left(1-\frac{\eta\alpha}{4}\right)\dist\left(\xsharp_t, \XM\right)+\norm{\xperp_t}+\left(2+\eta L\right)\norm{\x_t^\perp}\\
            &\leq \left(1-\frac{\eta\alpha}{4}\right)\zeta+4\max_{0\leq t\leq T}\norm{\x_t^\perp}\stackrel{(a)}{\leq} \left(1-\frac{\eta\alpha}{4}\right)\zeta+O\left(\frac{\eta\alpha}{L}\epsilon\right)\stackrel{(b)}{\leq} \zeta.
        \end{aligned}
    \end{equation}
    In $(a)$, we use \Cref{eq::control-psi}. Step $(b)$ follows from the assumption that $\zeta\geq \frac{\epsilon}{2L}$.
    This completes the inductive step.
    Given that we established $\x^{\sharp}_t \in \cN_{\XM}(\zeta), 0 \leq t \leq T$, we can invoke \Cref{lem::strong-local-structure} for any $0\leq t\leq T-1$, which yields 
\begin{equation}
           \dist\left(\x_{t+1}, \XM\right)\leq \left(1-\frac{\eta\alpha}{4}\right)\dist\left(\x_t, \XM\right)+\left(2+\eta L\right)\norm{\x_t^\perp}.
       \end{equation}
    Upon setting $A=1-\frac{\eta\alpha}{4}$ and $B=(2+\eta L)\max_{0\leq t\leq T}\norm{\x_t^\perp}$ in \Cref{lem::appendix-series-ineq}, we have
\begin{equation}
    \begin{aligned}
        \dist\left(\x_{T}, \XM\right)&\leq \left(1-\frac{\eta\alpha}{4}\right)^{T}\dist\left(\x_{0}, \XM\right)+\frac{8+4\eta L}{\eta\alpha}\max_{0\leq t\leq T}\left\{\norm{\x_t^{\perp}}\right\}\\
        &\stackrel{(a)}{\leq} \left(1-\frac{\eta\alpha}{4}\right)^{T}\cdot 2\zeta+\frac{8+4\eta L}{\eta\alpha}\max_{0\leq t\leq T}\left\{\norm{\x_t^{\perp}}\right\}\stackrel{(b)}{\leq} \left(1-\frac{\eta\alpha}{4}\right)^{T}\cdot 2\zeta+\frac{\epsilon}{2L}.
    \end{aligned}
    \label{eq::control-xt}
\end{equation}
In $(a)$, we use the fact that $\dist\left(\x_{0}, \XM\right)\leq \dist\left(\x_{0}^{\sharp}, \XM\right)+\norm{\x_0^\perp}\leq 2\zeta$; and in $(b)$, we use \Cref{eq::control-psi} and $\eta\leq \frac{1}{L}$.
Thus, within $T= O\left(\frac{1}{\eta\alpha}\log\left(L\zeta/\epsilon\right)\right)$ iterations, we have 
\begin{equation}
    \begin{aligned}
        \dist\left(\x_{T}, \XM\right)&\leq  \frac{\epsilon}{L}.
    \end{aligned}
\end{equation}
This completes the proof. \end{proof}

\begin{proof}[Proof of \Cref{thm::linear}]
    The proof follows by combining \Cref{cor::IPGD-SSP} and \Cref{thm::regularity}.
 \end{proof}

\section{Application: Over-parameterized Matrix Sensing}
\label{sec::application}
In this section, we apply our general framework to the over-parameterized matrix sensing problem. As highlighted in \Cref{subsec:failure}, general-purpose saddle-avoiding methods such as PGD fail to converge to the true solution in this setting. In contrast, we prove the convergence of IPGD for this problem.

Recall that here, the aim is to recover an $n \times n$ and PSD matrix $\mTheta^\star$ of rank at most $r$ from a set of linear measurements $\{(y_i, \A_i)\}_{i=1}^{N}$, where $y_i=\inner{\A_i}{\mTheta^\star}$ for $i = 1,\dots, N$. Therefore, the feasible region can be defined as $\cD=\{\mTheta\in \mathbb{R}^{n\times n}: \mTheta\succeq 0, \rank(\mTheta)\leq r \}$. To this end, a common approach is to employ a reparameterization map $\varphi(\X) = \X\X^\top$, where $\X\in \bR^{n\times r'}$ and $r'\geq r$, and minimize $f=L\circ \varphi$, where $L$ is the mean-squared error:
\begin{equation}
    \label{eq::matrix-sensing}
    \min_{\X\in \bR^{n\times r'}}f(\X)=\frac{1}{4N}\sum_{i=1}^{N}\left(\inner{\A_i}{\X\X^{\top}}-y_i\right)^2. \tag{MS}
\end{equation}
\begin{sloppypar}
\noindent 
	For simplicity, we define the linear operator $\cA:\bR^{n\times n}\to \bR^N$ and its adjoint $\cA^{*}:\bR^N\to \bR^{n\times n}$ as $\cA(\mTheta)=\frac{1}{\sqrt{N}}\left[
		\inner{\A_1}{\mTheta} \ \cdots \ \inner{\A_N}{\mTheta}\right]^{\top}$ and $\cA^{*}(\y)=\frac{1}{\sqrt{N}}\sum_{i=1}^{N}y_i\A_i$, respectively. A typical assumption on $\cA$ is that it satisfies the so-called {\it restricted isometry property} (RIP), defined below.
\end{sloppypar}

\begin{definition}[RIP]
    We say that a linear operator $\cA$ satisfies the $(\delta, r)$-RIP if 
    \begin{align*}
        (1-\delta)\norm{\mTheta}_{\fro}^2 \leq  \norm{\cA(\mTheta)}^2 \leq (1+\delta)\norm{\mTheta}_{\fro}^2,\quad \text{for all $\mTheta \in \bR^{n\times n}$ with $\rank(\mTheta) \leq r$}.
    \end{align*}
\end{definition}
Intuitively, the RIP condition requires the operator $\mathcal{A}$ to be nearly isotropic over all low-rank matrices, a property that holds when $\mathcal{A}$ exhibits a certain type of randomness.
For instance, it is well-known that for a Gaussian operator $\cA$, where all the elements of the matrices $\A_i$ are independently drawn from a standard normal distribution, $(\delta, r)$-RIP holds with
a high probability, provided that the number of measurements satisfies $N=\Omega\left({dr}/{\delta^2}\right)$~\cite{candes2011tight}.

When $r$ is unknown, it is common to over-parameterize the model by choosing a search rank $r'\leq n$ that exceeds the true rank $r$. Indeed, one common choice is $r'=n$ (referred to as \textit{fully over-parameterized} model). Under this assumption, we have $\cM=\varphi^{-1}(\cD) = \{\X\in \mathbb{R}^{n\times r'}: \rank(\X)\leq r\}$. However, verifying the conditions of \Cref{thm::IPGD-MSOSP} or \Cref{thm::linear} using this implicit region is challenging for two key reasons. First, because of the quartic structure of \( f \) and the unboundedness of \( \cM \), neither gradient- nor Hessian-Lipschitz condition holds on \( \cM \). Second, within this implicit region, it is difficult to establish a tight bound on the deviation rate.

We show that both of these challenges can be circumvented with a careful refinement of $\cM$. Let the eigen-decomposition of $\mTheta^\star$ be given by $\mTheta^\star = \V^{\star}\mSigma^{\star}\V^{\star\top}$, where $\mSigma^\star\in \mathbb{R}^{r\times r}$ is a diagonal matrix containing the nonzero eigenvalues of $\mTheta^\star$, and $\V^\star\in \cO_{n\times r}$ is the matrix of corresponding eigenvectors. We rely on the key result that all SOSPs within \( \cM \) coincide with \( \cX^\star = \{\X : \X = \V^{\star} \mSigma^{\star 1/2} \mO^\top,\ \mO \in \cO_{r' \times r}\} \)---that is, \( \XM = \cX^\star \)---provided that the linear operator \( \cA \) satisfies the \((\delta, 5r)\)-RIP with \( \delta \leq \tfrac{1}{10} \) \cite[Theorem~37]{chi2019nonconvex}. Since $\X^{\star}\X^{\star\top}=\M$ if and only if $\X^{\star}\in \cX^\star$~\cite[Lemma 20]{ma2023can}, to approximately recover the true rank-$r$ solution, it suffices to recover an approximate $\cM$-SOSP. Leveraging this characterization of \( \cM \)-SOSPs, we consider a restricted version of \( \cM \) that addresses both of the aforementioned challenges while still containing all the \( \cM \)-SOSPs:
\begin{equation}\label{def: implicit region}
    \begin{aligned}
        \cM_{\text{RIP}} &= \left\{\X\in \bR^{n\times r'}:\rank(\X)\leq r, \norm{\proj_{\V^{\star\perp}}\proj_{\X}}\leq 4\sqrt{r}\kappa\delta, \norm{\X}\leq 2\sqrt{\sigma_1^{\star}}\right\}.
    \end{aligned}
\end{equation}
In the above equation, \( \kappa = \sigma_1^\star / \sigma_r^\star \) represents the condition number of \( \mTheta^\star \), where \( \sigma_1^\star \) and \( \sigma_r^\star \) are the largest and smallest nonzero eigenvalues of \( \mTheta^\star \), respectively. Moreover, \( \V^{\star\perp} \in \cO_{n \times (n - r)} \) denotes the orthogonal complement of \( \V^\star \). It is easy to verify that \( \cX^\star \subset \cM_{\text{RIP}} \), implying that the SOSPs within $\cM$ and $\cM_{\text{RIP}}$ coincide. Intuitively, $\cM_{\text{RIP}}$ corresponds to the set of rank-\( r \) matrices whose column spaces do not deviate significantly from that of the true solution, and whose norms are bounded from above. As will be shown later, the former property enables us to establish sharp bounds on the deviation rate of GD, while the latter ensures that both the gradient- and Hessian-Lipschitz conditions hold on \( \cM_{\text{RIP}} \). Moreover, note that the zero matrix $\zero_{n\times r'}$ belongs to $\cM_{\text{RIP}}$, and hence, can be used as a trivial initial point for the algorithm.

Next, we show that the conditions of \Cref{thm::linear} are guaranteed to hold, leading to a nearly linear convergence of IPGD to an $\cM_{\text{RIP}}$-SOSP. To this goal,
first, we introduce an important implication of the RIP condition, which will be used
extensively throughout this section.

\begin{lemma}[Lemma 2.3. and Lemma C.2. in \cite{li2018algorithmic}]
\label{cor::rip}
    Suppose that $\cA$ satisfies $(\delta, r)$-RIP. Then for any matrix $\vect{\Theta}\in \bR^{n\times n}$, we have $\norm{(\cA^{*}\cA-I)(\vect{\Theta})}\leq \delta \norm{\vect{\Theta}}_{\star}$, where $\norm{\vect{\Theta}}_{\star}$ is the nuclear norm of $\vect{\Theta}$. Moreover, if $\rank(\vect{\Theta})\leq r$, we further have $\norm{(\cA^{*}\cA-I)(\vect{\Theta})}\leq \sqrt{r}\delta \norm{\vect{\Theta}}$.
\end{lemma}
For simplicity, we denote $\mDelta=\mTheta^\star-\X\X^{\top}$.

\paragraph{Local gradient- and Hessian-Lipschitzness}
We first show that $f(\X)$ defined in \eqref{eq::matrix-sensing} satisfies the local gradient- and Hessian-Lipschitzness properties.
\begin{proposition}
    \label{lemma::smoothness-matrix}
    \begin{sloppypar} Suppose that $\cA$ satisfies $(\delta, 2r)$-RIP with $\delta\leq\frac{1}{10\sqrt{r}}$. With the implicit region defined in \eqref{def: implicit region}, the function $f(\X)$ is {${(L, \cM_{\text{RIP}}, \tau)}$-gradient-Lipschitz} and {${(\rho, \cM_{\text{RIP}}, \tau)}$-Hessian-Lipschitz}, with parameters $L = 15\sigma_1^\star$, $\rho = 15\sqrt{\sigma_1^{\star}}$, and $\tau=\frac{1}{10}\sqrt{\frac{\sigma_1^{\star}}{d}}$.
    \end{sloppypar}
\end{proposition}

\begin{proof}
First, we provide the second and third-order derivatives of the objective function $f(\X)$
        \begin{align}
            \nabla^2 f(\X)[\Z, \Z]&=\inner{\cA^{*}\cA\left(\X\Z^{\top}+\Z\X^{\top}\right)\X-\cA^{*}\cA(\mDelta)\Z}{\Z},\label{eq::Hessian}\\
            \nabla^3 f(\X)[\Z, \Z, \Z]&=6\inner{\Z\Z^{\top}}{\cA^{*}\cA\left(\X\Z^{\top}\right)}.\label{eq::3rd_order}
        \end{align}
    The first formula follows from \cite[Lemma~7]{ge2017no}. We omit the derivation of the second formula, as it can be obtained in a similar manner.
    To derive the gradient and Hessian Lipschitz constants, we use the following equivalent definitions for smooth functions \cite[Theorem 9.19]{rudin1964principles}:
    \begin{align}
        L:&=\sup_{\X, \Y}\frac{\norm{\nabla f(\X)-\nabla f(\Y)}_{\fro}}{\norm{\X-\Y}_{\fro}}=\sup_{\X}\norm{\nabla^2 f(\X)}=\sup_{\X, \norm{\Z}_{\fro}\leq 1}\left|\nabla^2 f(\X)[\Z, \Z]\right|,\\
        \rho:&=\sup_{\X, \Y}\frac{\norm{\nabla^2 f(\X)-\nabla^2 f(\Y)}_{\fro}}{\norm{\X-\Y}_{\fro}}=\sup_{\X}\norm{\nabla^3 f(\X)}=\sup_{\X, \norm{\Z}_{\fro}\leq 1}\left|\nabla^3 f(\X)[\Z, \Z, \Z]\right|.
    \end{align}
    We first provide an upper bound for $L$. For any $\X\in \cN_{\cM_{\text{RIP}}}(\tau)$ and $\norm{\Z}_{\fro}\leq 1$, \Cref{eq::Hessian} implies
    \begin{equation}
        \begin{aligned}
            \left|\nabla^2 f(\X)[\Z, \Z]\right|
            &\leq \left|\inner{\left(\X\Z^{\top}+\Z\X^{\top}\right)\X-\mDelta\Z}{\Z}\right|\\
            &\qquad+\norm{(\cA^{*}\cA-I)\left(\X\Z^{\top}+\Z\X^{\top}\right)}\norm{\X}_{\fro}\norm{\Z}_{\fro}+\norm{(\cA^{*}\cA-I)(\mDelta)}\norm{\Z}_{\fro}^2\\
            &\leq  2\norm{\X}^2+\norm{\mDelta}+\norm{(\cA^{*}\cA-I)\left(\X\Z^{\top}+\Z\X^{\top}\right)}\norm{\X}_{\fro}+\norm{(\cA^{*}\cA-I)(\mDelta)},
        \end{aligned}
    \end{equation}
    where in the last inequality, we use triangle inequality and the fact that $\norm{\Z}_{\fro}\leq 1$ repeatedly. Note that $\X\in \cN_{\cM_{\text{RIP}}}(\tau)$ implies that $\rank\big(\X^{\sharp}\big)\leq r$, $\norm{\X^{\sharp}}\leq 2\sqrt{\sigma_1^{\star}}$, and $\norm{\X^{\perp}}_{\fro}\leq \tau$. Consider the following decomposition:
    \begin{equation}
        \mDelta = \mDelta^{\sharp}+\mDelta^{\perp}, \quad \text{where}\quad \mDelta^{\sharp}=\M-\X^{\sharp}\X^{\sharp\top}, \quad \mDelta^{\perp}=-\X^{\perp}\X^{\sharp\top}-\X^{\sharp}\X^{\perp\top}-\X^{\perp}\X^{\perp\top}.
    \end{equation}
    We have $\rank(\mDelta^{\sharp})\leq 2r$, $\norm{\mDelta^{\sharp}}\leq 4\sigma_1^{\star}$, and $\norm{\mDelta^{\perp}}_{\fro}\leq 2\norm{\X^{\sharp}}\norm{\X^{\perp}}_{\fro}+\norm{\X^{\perp}}_{\fro}^2\leq 6\sqrt{\sigma_1^{\star}}\tau$. Consequently, we obtain $\norm{\mDelta}\leq 4\sigma_1^{\star}+6\sqrt{\sigma_1^{\star}}\tau\leq 5\sigma_1^{\star}$ since $\tau\leq \frac{1}{10}\sqrt{\frac{\sigma_1^{\star}}{d}}$. Moreover, one can write
    \begin{equation}
        \begin{aligned}
            \norm{(\cA^{*}\cA-I)(\mDelta)}&\leq \norm{(\cA^{*}\cA-I)(\mDelta^{\sharp})}+ \norm{(\cA^{*}\cA-I)(\mDelta^{\perp})}\\
            &\stackrel{(a)}{\leq} \sqrt{2r}\delta\big\|\mDelta^{\sharp}\big\|+\delta \big\|\mDelta^{\perp}\big\|_{\star}\stackrel{(b)}{\leq} 4\sqrt{2r}\sigma_1^{\star}\delta+\delta\cdot 6\sqrt{d\sigma_1^\star}\tau\stackrel{(c)}{\leq} 10\sqrt{r}\sigma_1^{\star}\delta,
        \end{aligned}
    \end{equation}
    where, in $(a)$, we apply \Cref{cor::rip}. Moreover, in $(b)$, we use the facts that $\norm{\mDelta^{\sharp}}\leq 4\sigma_1^{\star}$ and $\norm{\mDelta^{\perp}}_\star\leq \sqrt{d}\norm{\mDelta^{\perp}}_{\fro}\leq 6\sqrt{d\sigma_1^\star}\tau$. Finally, in $(c)$, we use the fact that $\tau= \frac{1}{10}\sqrt{\frac{\sigma_1^{\star}}{d}}$. Similarly, we obtain $\big\|(\cA^{*}\cA-I)\left(\X\Z^{\top}+\Z\X^{\top}\right)\big\|\leq 5\sqrt{r}\sigma_1^{\star}\delta$. Combining the derived bounds, we obtain 
    \begin{equation}
        \left|\nabla^2 f(\X)[\Z, \Z]\right|\leq 13\sigma_1^{\star}+15\sqrt{r}\sigma_1^{\star}\delta\leq 15 \sigma_1^{\star},
    \end{equation}
    where we use the fact that $\delta\leq \frac{1}{10\sqrt{r}}$. This implies that the objective function is gradient Lipschitz with Lipschitz constant $L=15\sigma_1^{\star}$ within $\cN_{\cM_{\text{RIP}}}(\tau)$.
In a similar fashion, we can show that the objective function is Hessian Lipschitz with Lipschitz constant $\rho=15\sqrt{\sigma_1^\star}$ within $\cN_{\cM_{\text{RIP}}}(\tau)$. We omit the details for brevity.
 
\end{proof}

\paragraph{Closure of $\cM_{\text{RIP}}$ under GD updates}
Next, we prove the closure of $\cM_{\text{RIP}}$ under GD updates. 
\begin{proposition}
    \label{lem::GD-on-the-manifold-matrix-sensing}
    Suppose that $\cA$ satisfies $(\delta, 2r)$-RIP with $\delta\leq\frac{1}{10\sqrt{r}}$. For any $\X \in \cM_{\text{RIP}}$, we have $\X^+:= \X-\eta\nabla f(\X)\in \cM_{\text{RIP}}$, provided that the step-size satisfies $\eta \leq \frac{1}{10\sigma_1^{\star}}$.  
\end{proposition}

Before proving the above proposition, we first present a few helper lemmas.
\begin{lemma}[\protect{Adapted from \cite[Lemma 2.1.2 and 2.1.3]{chen2021spectral}}]\label{lem::helper1}
     For any $\U, \U^\star\in \cO_{n\times r}$, we have
\begin{equation}
    \norm{\proj_{\U^\perp}\proj_{\U^\star}}=\norm{\proj_{\U}-\proj_{\U^\star}}\leq \min _{\R \in \cO_{r \times r}}\norm{\U \R-\U^{\star}} \leq \sqrt{2}\norm{\proj_{\U}-\proj_{\U^\star}}=\sqrt{2}\norm{\proj_{\U^\perp}\proj_{\U^\star}}.
\end{equation}
\end{lemma}
\begin{lemma}[Adapted from \protect{\cite[Theorem 2.4]{chen2016perturbation}}]
    \label{lem::appendix-tu}
    For any $\X, \Y\in \bR^{n_1\times n_2}$, we have
    \begin{equation}
        \norm{\proj_\X - \proj_\Y}\leq \min\left\{\norm{(\X-\Y)\X^{\dagger}}, \norm{(\X-\Y)\Y^{\dagger}}\right\}.
    \end{equation}
\end{lemma}
\begin{lemma}
        \label{lem::matrix-innverse}
        For any two invertible matrices $\A$ and $\B$, we have 
        \begin{equation}
            \norm{\A^{-1}-\B^{-1}}\leq \norm{\A^{-1}}\norm{\B^{-1}}\norm{\A-\B}.
        \end{equation}
    \end{lemma}
    \begin{proof}
        The result follows by observing that $\A^{-1}-\B^{-1}=\A^{-1}(\B-\A)\B^{-1}$ and applying the submultiplicativity of the operator norm $\norm{\A^{-1}-\B^{-1}}\leq \norm{\A^{-1}}\norm{\B^{-1}}\norm{\A-\B}$.
     \end{proof}

\begin{proof}[Proof of \Cref{lem::GD-on-the-manifold-matrix-sensing}]
Note that 
\begin{equation}
    \X^+:= \X-\eta\nabla f(\X)=\Big(\I-\eta \cA^{*}\cA\big(\X\X^{\top}-\mTheta^\star\big)\Big)\X=(\I+\eta \cA^{*}\cA(\mDelta))\X.
\end{equation}
Based on the definition $\cM_{\text{RIP}}$ in \Cref{def: implicit region}, to show that $\X^+=(\I+\eta \cA^{*}\cA(\mDelta))\X\in \cM_{\text{RIP}}$, we need to verify the following three conditions 
{\it \begin{enumerate}
    \item[(A)] $\rank(\X^+)\leq r$;
    \item[(B)] 
    $\norm{\X^+}\leq 2\sqrt{\sigma_1^{\star}}$;
    \item[(C)] $\norm{\proj_{\V^{\star\perp}}\proj_{\X^+}}\leq 4\sqrt{r}\kappa\delta$. 
\end{enumerate}}
\paragraph{\it Verification of (A)} This is already shown in \Cref{example::matrix-opt}. 
\paragraph{\it Verification of (B)} This can be established as follows:
    \begin{equation}
        \begin{aligned}
            \norm{\X^+}&\leq \norm{(\I+\eta \mDelta)\X}+\norm{(\cA^{*}\cA-I)(\mDelta)}\norm{\X}\\
            &\stackrel{(a)}{\leq} \norm{\left(\I-\eta \X\X^{\top}\right)\X} + \eta\sigma_1^{\star}\norm{\X}+4\sqrt{2r}\sigma_1^{\star}\delta\norm{\X}\\
            &\stackrel{(b)}{=}\left(1+\eta\left(\sigma_1^{\star}+4\sqrt{2r}\sigma_1^{\star}\delta-\norm{\X}^2\right)\right)\norm{\X}\\
            &\stackrel{(c)}{\leq} \left(1+\eta\left(\sigma_1^{\star}+4\sqrt{2r}\sigma_1^{\star}\delta-4\sigma_1^{\star}\right)\right)\cdot 2\sqrt{\sigma_1^{\star}}\leq 2\sqrt{\sigma_1^{\star}}.
        \end{aligned}
    \end{equation}
    In $(a)$, we use triangle inequality and \Cref{cor::rip} along with the facts that $\rank\left(\mDelta\right)\leq 2r$ and $\norm{\mDelta}\leq 4\sigma_1^\star$. In $(b)$, we use the fact that the function $f(x)=x-\eta x^3$ is increasing in the interval $[0, \frac{1}{\sqrt{3\eta}}]$, which implies $\norm{\left(\I-\eta \X\X^{\top}\right)\X}=\norm{\X}-\eta \norm{\X}^3$ since $\norm{\X}\leq 2\sqrt{\sigma_1^\star}\leq \frac{1}{\sqrt{3\eta}}$. Lastly, in $(c)$, we use the fact that the function $g(x)=(1+\eta(\sigma_1^{\star}+4\sqrt{2r}\sigma_1^{\star}\delta-x^2))x$ is also increasing in the range $[0, \frac{1}{3\sqrt{\eta}}]$.

    \paragraph{\it Verification of (C)}
We first define $\Tilde{\X}^+=(\I+\eta \mDelta)\X$. Then, we have the following decomposition 
    \begin{equation}
        \begin{aligned}
            \norm{\proj_{\V^{\star\perp}}\proj_{\X^+}}&\leq \underbrace{\norm{\proj_{\V^{\star\perp}}\proj_{\Tilde{\X}^+}}}_{(\rom{1})}+\underbrace{\norm{\proj_{\Tilde{\X}^+}-\proj_{\X^+}}}_{(\rom{2})}.
        \end{aligned}
    \end{equation}
    We control these two terms separately. Note that $\proj_\X=\X\X^\dagger$ where $\X^\dagger$ is the pseudo-inverse of $\X$. Hence, we have
    \begin{equation}
        \begin{aligned}
            \proj_{\V^{\star\perp}}\proj_{\Tilde{\X}^+}&=\proj_{\V^{\star\perp}}\Tilde{\X}^+(\Tilde{\X}^+)^{\dagger}=\proj_{\V^{\star\perp}}\left(\I+\eta\mDelta\right)\X(\Tilde{\X}^+)^{\dagger}\\
            &\stackrel{(a)}{=}\proj_{\V^{\star\perp}}\left(\I-\eta \X\X^{\top}\right)\X(\Tilde{\X}^+)^{\dagger}\stackrel{(b)}{=}\proj_{\V^{\star\perp}}\proj_{\X}\left(\I-\eta \X\X^{\top}\right)(\I+\eta \mDelta)^{-1}\proj_{\Tilde{\X}^+},
        \end{aligned}
    \end{equation}
    where, in $(a)$, we use the fact that $\proj_{\V^{\star\perp}}\mTheta^\star=\bm{0}$. Moreover, in $(b)$, we use the fact that $\X(\Tilde{\X}^+)^{\dagger}=(\I+\eta \mDelta)^{-1}\Tilde{\X}^+(\Tilde{\X}^+)^{\dagger}=(\I+\eta \mDelta)^{-1}\proj_{\Tilde{\X}^+}$.
    Therefore, we obtain 
    \begin{equation}
    \label{eq::54}
        \begin{aligned}
            (\rom{1})&\leq  \norm{\proj_{\V^{\star\perp}}\proj_{\X}}\underbrace{\norm{\proj_{\X}\left(\I-\eta \X\X^{\top}\right)\left(\I+\eta\mDelta\right)^{-1}}}_{(\rom{3})}.
        \end{aligned}
    \end{equation}
    To proceed, we decompose $\mTheta^\star=\V^{\star}\mSigma^\star\V^{\star\top}$ as
    \begin{equation}
        \mTheta^\star=\underbrace{\V_\X\mSigma^\star\V^{\top}_{\X}}_{:=\mTheta^\star_\X}+\underbrace{\left(\V^\star-\V_\X\right)\mSigma^\star\V^{\star\top}+ \V_\X\mSigma^\star\left(\V^\star-\V_\X\right)^{\top}}_{:=\vect{\Delta_M}},
    \end{equation}
    where $\V_\X$ is defined as $\V_\X=\argmin_{\V\in \cO_{n\times r}, \proj_{\X}\V=\V}\norm{\V^{\star}-\V}$. 
    Therefore, we can decompose $(\rom{3})$ as
    \begin{equation}
        \begin{aligned}
            (\rom{3})&\leq \underbrace{\norm{\proj_{\X}\left(\I-\eta \X\X^{\top}\right)\left(\I+\eta \left(\mTheta^\star_\X-\X\X^{\top}\right)\right)^{-1}}}_{(\rom{3}_1)}\\
            &\quad + \underbrace{\norm{\proj_{\X}\left(\I-\eta \X\X^{\top}\right)\left(\left(\I+\eta \left(\mTheta^\star_\X-\X\X^{\top}\right)\right)^{-1}-\left(\I+\eta \left(\mTheta^\star-\X\X^{\top}\right)\right)^{-1}\right)}}_{(\rom{3}_2)}.
        \end{aligned}
    \end{equation}
    To control $(\rom{3}_1)$, we note that all the matrices involved share the same singular space. Therefore, we have 
    \begin{equation}
        \begin{aligned}
            (\rom{3}_1)&=\norm{\left(\I-\eta \mSigma_\X^2\right)\left(\I+\eta\left(\mSigma^\star-\mSigma_\X^2\right)\right)^{-1}}\stackrel{(a)}{\leq} \norm{\left(\I+\eta\mSigma^\star\right)^{-1}}=\frac{1}{1+\eta \sigma_r^\star}\stackrel{(b)}{\leq} 1-\frac{1}{2}\eta \sigma_r^\star.
        \end{aligned}
    \end{equation}
    Here in $(a)$, we use the inequality $\frac{x}{y}\leq \frac{x+a}{y+a}$ for every $x, y, a>0$ and $y\geq x$. In $(b)$, we use the inequality $\frac{1}{1+x}\leq 1-\frac{1}{2}x$ for any $0\leq x\leq 1$. 
    
    To control $(\rom{3}_2)$, we rely on \Cref{lem::matrix-innverse}. 
    Upon setting $\A=\I+\eta \left(\mTheta^\star_\X-\X\X^{\top}\right)$ and $\B=\I+\eta \left(\mTheta^\star-\X\X^{\top}\right)$ in \Cref{lem::matrix-innverse}, we have 
    \begin{equation}
        \begin{aligned}
            (\rom{3}_2)&\leq \norm{\A^{-1}-\B^{-1}}\leq \norm{\A^{-1}}\norm{\B^{-1}}\norm{\A-\B}.
        \end{aligned}
    \end{equation}
    To proceed, notice that 
    \begin{equation}
        \begin{aligned}
            \max\left\{\norm{\A^{-1}}, \norm{\B^{-1}}\right\}&\leq \norm{\left(\I-\eta \left(\X\X^{\top}\right)\right)^{-1}}= \frac{1}{1-\eta \norm{\X}^2}\leq 2.
        \end{aligned}
    \end{equation}
    In the last inequality, we use the fact that $\eta \leq \frac{1}{8\sigma_1^{\star}}$ and $\norm{\X}\leq 2\sqrt{\sigma_1^\star}$. On the other hand, we have
    \begin{equation}
        \begin{aligned}
            \norm{\A-\B}&=\eta \norm{\vect{\Delta_M}}\leq 2\sigma_1^\star\norm{\V^\star-\V_\X}\leq 2\sqrt{2}\sigma_1^\star \norm{\proj_{\V^{\star\perp}}\proj_{\X}}, 
        \end{aligned}
    \end{equation}
    where, in the last inequality, we use \Cref{lem::helper1}.
    Therefore, we obtain $(\rom{3}_2)\leq 8\sqrt{2}\sigma_1^{\star}\norm{\proj_{\V^{\star\perp}}\proj_{\X}}$.
    Combining the derived bounds for $(\rom{3}_1)$ and $(\rom{3}_2)$, we have
    \begin{equation}
        \begin{aligned}
            (\rom{1})&\leq  \norm{\proj_{\V^{\star\perp}}\proj_{\X}}\left((\rom{3}_1)+(\rom{3}_2)\right)\leq \left(1-\frac{1}{2}\eta\sigma_r^\star+8\sqrt{2}\eta\sigma_1^\star\norm{\proj_{\V^{\star\perp}}\proj_{\X}} \right)\norm{\proj_{\V^{\star\perp}}\proj_{\X}}.
        \end{aligned}
    \end{equation}
    Finally, we control $(\rom{2})$. To this goal, we write
    \begin{equation}
        \begin{aligned}
            (\rom{2})&\stackrel{(a)}{\leq} \norm{\big(\Tilde{\X}^+-\X^+\big)\big(\Tilde{\X}^+\big)^{\dagger}}=\eta\norm{(\cA^{*}\cA-I)(\mDelta)\left(\I+\eta\mDelta\right)^{-1}\Tilde{\X}^+\big(\Tilde{\X}^+\big)^{\dagger}}\\
            &\leq \eta\norm{(\cA^{*}\cA-I)\left(\mDelta\right)}\norm{\left(\I+\eta\mDelta\right)^{-1}}\stackrel{(b)}{\leq} \eta\sqrt{2r}\delta \norm{\mDelta} \norm{\left(\I+\eta\mDelta\right)^{-1}}\stackrel{(c)}{\leq} 4\eta\sqrt{r}\sigma_1^{\star}\delta,
        \end{aligned}
    \end{equation}
where in $(a)$, we use \Cref{lem::appendix-tu}. Moreover, in $(b)$, we use \Cref{cor::rip}. Finally, in $(c)$, we use $\norm{\mDelta}\leq 4\sigma_1^{\star}$ and $\eta\leq \frac{1}{4\sqrt{2}\sigma_1^{\star}}$. Combining the derived bounds for $(\rom{1})$ and $(\rom{2})$, we have 
\begin{equation}
    \norm{\proj_{\V^{\star\perp}}\proj_{\X^+}}\leq \left(1-\frac{1}{4}\eta\sigma_r^{\star}+\eta \sigma_1^{\star}\norm{\proj_{\V^{\star\perp}}\proj_{\X}}\right)\norm{\proj_{\V^{\star\perp}}\proj_{\X}}+ 2\eta\sqrt{r}\sigma_1^{\star}\delta \leq 4\sqrt{r}\kappa\delta,
\end{equation}
where the last inequality follows from $\norm{\proj_{\V^{\star\perp}}\proj_{\X}}\leq 4\sqrt{r}\kappa\delta$ by our assumption.
 \end{proof}

\paragraph{$\cM_{\text{RIP}}$-strict saddle property and $\cM_{\text{RIP}}$-regularity property}
Next, we establish the $\cM_{\text{RIP}}$-strict saddle property and the $\cM_{\text{RIP}}$-regularity property, both of which are essential for ensuring the efficient convergence of our algorithm.
\begin{proposition}
    \label{prop::matrix-factorization-strict-saddle-local-linear-convergence}
    Suppose that $\cA$ satisfies $(\delta, 6r)$-RIP with $\delta \leq \frac{1}{10}$. The following statements hold for $f(\X)$ defined in \eqref{eq::matrix-sensing}:

        \begin{itemize}
            \item It satisfies $\left(\frac{1}{160}\sigma_r^{\star 3 / 2}, \frac{1}{5}\sigma_r^{\star}, \frac{1}{8}\sigma_r^{\star 1 / 2}\right)$-$\cM_{\text{RIP}}$-strict saddle property (\Cref{assumption::strict-saddle}).
            \item It satisfies $\left(\frac{1}{2} \sigma_r^{\star}, \frac{425}{64} \sigma_1^{\star}, \frac{1}{4}\sigma_r^{\star 1 / 2}\right)$-$\cM_{\text{RIP}}$-regularity property (\Cref{assumption::regularity}).
        \end{itemize}
\end{proposition}

{To prove the above proposition, we establish a connection between the local geometry of $f(\X)$ within the implicit region $\cM_{\text{RIP}}$ to that of the exactly-parameterized matrix sensing problem:
    \begin{equation}
        \min_{\barX\in \bR^{n\times r}}\bar{f}\left(\barX\right)=\frac{1}{4N}\sum_{i=1}^{N}\left(\inner{\A_i}{\barX\barX^{\top}}-y_i\right)^2. \tag{MS-exact}
        \label{eq::exact-sym}
    \end{equation}
    We note that, unlike the over-parameterized setting in \eqref{eq::matrix-sensing}, where the rank of the matrix variable can be as large as $r' \geq r$, the rank in \eqref{eq::exact-sym} is fixed to the true value $r$.
    For the exactly-parameterized matrix sensing, all SOSPs coincide with the set of true solutions, defined as
\begin{align}\label{eq::barX}
    {\cX}^\star_{\leq r}=\left\{\V^{\star}\mSigma^{\star 1/2}\mO\in \bR^{n\times r}:\mO\in \cO_{r\times r}\right\},
\end{align}
    provided that the linear operator $\cA$ satisfies $(\delta, 6r)$-RIP with $\delta \leq \frac{1}{10}$ \cite[Theorem~8]{ge2017no}.\footnote{Note that the definition of ${\cX}^\star_{\leq r}$ should not be confused with the previously introduced set $\cX^\star = \{\X: \X = \V^{\star}\mSigma^{\star 1/2}\mO^\top, \mO\in \cO_{r'\times r}\}$, which characterizes the true solutions in the over-parameterized regime.}
    Moreover, this objective function satisfies the strict saddle property and exhibits strong local structure globally (i.e., over the exactly-parameterized space $\mathbb{R}^{n \times r}$), as detailed below.

    \begin{lemma}
        \label{lem::strict-saddle-property} Assume the linear operator $\cA$ satisfies $(\delta, 6r)$-RIP with $\delta \leq \frac{1}{10}$. The following statements hold for $\bar{f}\left(\barX\right)$ defined in \eqref{eq::exact-sym}:
        \begin{itemize}
            \item It satisfies $\left(\epsilon, \frac{1}{5}\sigma_r^{\star}, \frac{20\epsilon}{\sigma_r^{\star}}\right)$-strict saddle property, for any $\epsilon>0$ \cite[Theorem~8]{ge2017no}.
            \item It satisfies $\left(\frac{1}{2} \sigma_r^{\star}, \frac{425}{64} \sigma_1^{\star}, \frac{1}{4}\sigma_r^{\star 1 / 2}\right)$-regularity property \cite[Eq.~5.7]{tu2016low}.
        \end{itemize}
    \end{lemma}

    To prove \Cref{prop::matrix-factorization-strict-saddle-local-linear-convergence}, we leverage the above lemma to show that, when restricted to the implicit region $\cM_{\text{RIP}}$ (as defined in \Cref{def: implicit region}), the function $f(\X)$ satisfies the same properties with exactly the same parameters. 

    \begin{proof}[Proof of \Cref{prop::matrix-factorization-strict-saddle-local-linear-convergence}]
    We first establish a one-to-one mapping $h:\cM_{\text{RIP}}\to \bR^{n\times r}$, which maps an arbitrary $\X\in \cM_{\text{RIP}}$ to a point $\bar \X=h(\X)\in \bR^{n\times r}$. Specifically, for an arbitrary $\X \in \cM_{\text{RIP}}$ with the SVD of the form $\X=\mL_\X\mSigma_{\X}\R_{\X}^{\top}$, where $\mL_\X\in \cO_{n\times r}$, $\R_{\X}\in \cO_{r'\times r}$, and $\mSigma_{\X}\in \bR^{r\times r}$ is a diagonal matrix\footnote{If $\rank(\X) < r$, we construct $\mSigma_{\X}$ by padding its diagonal with additional zero entries.}, we define the mapping as $\barX=h(\X)=\mL_{\X}\mSigma_{\X}\in \bR^{n\times r}$. 

    \paragraph{\it Proof of the first statement} Suppose that $\X \in \cM_{\text{RIP}}$ satisfies $\norm{\nabla f(\X)}_{\fro}<\epsilon$ and $\dist(\X, \cX_{\cM_{\text{RIP}}})>\frac{20\epsilon}{\sigma_r^\star}$. To establish the $\cM_{\text{RIP}}$-strict saddle property with the specified parameters, it suffices to show that $\lambda_{\min}\left(\nabla^2f(\X)\right)\leq -\frac{1}{5}\sigma_r^\star$. To this end, we first show that for $\barX=h(\X)$, we also have $\norm{\nabla \bar{f}\left(\barX\right)}_{\fro}<\epsilon$ and $\dist\left(\barX, {\cX}^\star_{\leq r}\right)>\frac{20\epsilon}{\sigma_r^{\star}}$.
One can write
    \begin{equation}\label{eq::equal-grad}
        \begin{aligned}
            \norm{\nabla f(\X)}_{\fro}&=\norm{\cA^{*}\cA\big(\X\X^{\top}-\mTheta^\star\big)\X}_{\fro}=\norm{\cA^{*}\cA\left(\barX\barX^{\top}-\mTheta^\star\right)\barX \R_{\X}^{\top}}_{\fro}\\
            &=\norm{\cA^{*}\cA\left(\barX\barX^{\top}-\mTheta^\star\right)\barX}_{\fro}=\norm{\nabla \bar{f}\left(\barX\right)}_{\fro},
        \end{aligned}
    \end{equation} 
    where, in the second equality, we use the fact that $\X\X^\top=\barX\barX^\top$. This implies that $\norm{\nabla f(\X)}_{\fro}=\norm{\nabla \bar f\left(\barX\right)}_{\fro}\leq \epsilon$. To establish the second inequality, we write
    \begin{equation}\label{eq::equal-dist}
        \begin{aligned}
            \dist(\X, \cX_{\cM_{\text{RIP}}})&=\min_{\mO\in \cO_{r'\times r'}}\norm{\X\mO - \begin{bmatrix}
                \V^{\star}\mSigma^{\star 1/2} & \zero_{n\times (r'-r)}
            \end{bmatrix}}_{\fro}\\
            &\stackrel{(a)}{\leq} \min_{\bar\mO\in \cO_{r\times r}}\norm{\barX\R_{\X}^{\top}\begin{bmatrix}
                \R_{\X}\bar\mO & \R_{\X}^{\perp}
            \end{bmatrix} - \begin{bmatrix}
                \V^{\star}\mSigma^{\star 1/2} & \zero_{n\times (r'-r)}
            \end{bmatrix}}_{\fro}\\
            &=\min_{\bar\mO\in \cO_{r\times r}}\norm{\barX\bar\mO - \V^{\star}\mSigma^{\star 1/2}}_{\fro}=\dist\left(\barX, {\cX}^\star_{\leq r}\right),
        \end{aligned}
    \end{equation}
    where $(a)$ follows after setting $\mO = \begin{bmatrix}
                \R_{\X}\bar\mO & \R_{\X}^{\perp}
            \end{bmatrix}\in \cO_{r'\times r'}$
    and choosing $\R_{\X}^{\perp}\in \cO_{n\times (r'-r)}$ that satisfies $\R_{\X}^{\top}\R_{\X}^{\perp}=\zero$. This implies that $\dist\left(\barX, {\cX}^\star_{\leq r}\right) = \dist(\X, \cX_{\cM_{\text{RIP}}})>\frac{20\epsilon}{\sigma_r^{\star}}$. Therefore, we have established that $\norm{\nabla f(\X)}_{\fro}\leq \epsilon$ and $\dist\left(\barX, {\cX}^\star_{\leq r}\right)>\frac{20\epsilon}{\sigma_r^{\star}}$. Based on these inequalities, the first statement of \Cref{lem::strict-saddle-property} implies that $\lambda_{\min}\left(\nabla^2\bar{f}\left(\barX\right)\right)\leq -\frac{1}{5}\sigma_r^{\star}$. Next, we use this inequality to establish a similar bound for $\lambda_{\min}\left(\nabla^2f(\X)\right)$. This is achieved as follows:

    \begin{equation}
        \begin{aligned}
            \lambda_{\min}\left(\nabla^2f(\X)\right)&= \min_{\norm{\Z}_{\fro}\leq 1}\nabla^2f(\X)[\Z, \Z]\\
            &= \min_{\norm{\Z}_{\fro}\leq 1} \inner{\cA^{*}\cA\left(\X\Z^{\top}+\Z\X^{\top}\right)\X+\cA^{*}\cA\big(\X\X^{\top}-\M\big)\Z}{\Z}\\
            &\stackrel{(a)}{\leq} \min_{\norm{\Bar{\Z}}_{\fro}\leq 1} \inner{\cA^{*}\cA\left(\barX\Bar{\Z}^{\top}+\Bar{\Z}\barX^{\top}\right)\barX\R_{\X}^{\top}+\cA^{*}\cA\left(\barX\barX^{\top}\!-\!\M\right)\Bar{\Z}\R_{\X}^{\top}}{\Bar{\Z}\R_{\X}^{\top}}\\
            &=\lambda_{\min}\left(\nabla^2\bar{f}\left(\barX\right)\right)\leq -\frac{1}{5}\sigma_r^{\star},
        \end{aligned}
    \end{equation}
    where $(a)$ follows after setting $\Z=\Bar{\Z}\R_{\X}^{\top}$. Therefore, $f(\X)$ satisfies $\left(\epsilon, \frac{1}{5}\sigma_r^{\star}, \frac{20\epsilon}{\sigma_r^{\star}}\right)$-$\cM_{\text{RIP}}$-strict saddle property. Setting $\epsilon=\frac{1}{160}\sigma_r^{\star 3/2}$ completes the proof of the first statement of \Cref{prop::matrix-factorization-strict-saddle-local-linear-convergence}.

    \paragraph{\it Proof of the second statement} According to the second statement of \Cref{lem::strict-saddle-property}, for any $\barX \in \mathbb{R}^{n\times r}$ satisfying $\dist\left(\barX, {\cX}^\star_{\leq r}\right)\leq \frac{1}{4}\sigma_r^{\star 1/2}$, we have
    \begin{equation}
        \inner{\nabla \bar{f}\left(\barX\right)}{\barX-\cP_{{\cX}^\star_{\leq r}}\left(\barX\right)} \geq \frac{\alpha}{2}\dist^2\left(\barX, {\cX}^\star_{\leq r}\right)+\frac{1}{2 \beta}\norm{\nabla \bar{f}\left(\barX\right)}_{\fro}^2
    \end{equation}
    with $\alpha=\frac{1}{2} \sigma_r^{\star}$ and $\beta=\frac{425}{64} \sigma_1^{\star}$. To proceed, we first show that $\cP_{\cX_{\cM_{\text{RIP}}}}(\X)\R_{\X}=\cP_{{\cX}^\star_{\leq r}}\left(\barX\right)$. Note that $\cP_{\cX_{\cM_{\text{RIP}}}}(\X)=\V^{\star}\mSigma^{\star 1/2}\mO_\X$ where $\mO_\X$ is given by
    \begin{equation}
        \begin{aligned}
            \mO_\X&=\argmin_{\mO\in \cO_{r\times r'}}\norm{\X - \V^{\star}\mSigma^{\star 1/2}\mO}_{\fro}=\argmin_{\mO\in \cO_{r\times r'}}\norm{\barX\R_{\X}^{\top} - \V^{\star}\mSigma^{\star 1/2}\mO}_{\fro}=\argmin_{\mO\in \cO_{r\times r'}}\norm{\barX - \V^{\star}\mSigma^{\star 1/2}\mO\R_{\X}}_{\fro}.
        \end{aligned}
    \end{equation}
    Therefore, we conclude that $\mO_\X\R_{\X}=\mO_{\barX}$ where  $\mO_{\barX}=\argmin_{\mO\in \cO_{r\times r}}\norm{\barX - \V^{\star}\mSigma^{\star 1/2}\mO}_{\fro}$. This implies that
    $$\cP_{\cX_{\cM_{\text{RIP}}}}(\X)\R_{\X}=\V^{\star}\mSigma^{\star 1/2}\mO_\X\R_{\X}=V^{\star}\mSigma^{\star 1/2}\mO_{\barX} =\cP_{{\cX}^\star_{\leq r}}\left(\barX\right).$$ 
    Based on the above equality, we have
    \begin{equation}
        \begin{aligned}
            \inner{\nabla f(\X)}{\X-\cP_{\cX_{\cM_{\text{RIP}}}}(\X)} &=\inner{\nabla \bar{f}\left(\barX\right)\R_\X^{\top}}{\X-\cP_{\cX_{\cM_{\text{RIP}}}}(\X)}=\inner{\nabla \bar{f}\left(\barX\right)}{\barX-\cP_{{\cX}^\star_{\leq r}}\left(\barX\right)}\\
            &\geq \frac{\alpha}{2}\dist^2\left(\barX, {\cX}^\star_{\leq r}\right)+\frac{1}{2 \beta}\norm{\nabla \bar{f}\left(\barX\right)}_{\fro}^2\\
            &= \frac{\alpha}{2}\dist^2(\X, \cX_{\cM_{\text{RIP}}})+\frac{1}{2 \beta}\norm{\nabla f(\X)}_{\fro}^2,
        \end{aligned}
    \end{equation}
    where, in the last equality, we use $\dist\left(\barX, {\cX}^\star_{\leq r}\right) = \dist(\X, \cX_{\cM_{\text{RIP}}})$ and $\norm{\nabla \bar{f}\left(\barX\right)}_{\fro} = \norm{\nabla f(\X)}_{\fro}$, established in \Cref{eq::equal-grad} and \Cref{eq::equal-dist}, respectively. This completes the proof.
 \end{proof}
}

\paragraph{Bounding the deviation rate}
Lastly, we establish an upper bound on the deviation rate.
\begin{proposition}
    \label{prop::deviation-rate-ms}
    Suppose that $\cA$ satisfies $(\delta, 2r)$-RIP with $\delta\leq\frac{1}{10r}$. Then, for any $\tau\leq\frac{1}{10}\sqrt{\frac{\sigma_1^{\star}}{d}}$, the $\tau$-deviation rate satisfies $R({\tau})\leq 45\sigma_1^{\star}r\kappa\delta$.
\end{proposition}

\begin{proof}
According to the definition of $\tau$-deviation rate (\Cref{def::deviation_rate}), we have
\begin{align}\label{eq::Rtau}
    R(\tau) = \sup_{\X\in\cN_{\cM_{\text{RIP}}}(\tau)} r(\X)\leq \sup_{\X\in\cN_{\cM_{\text{RIP}}}(\tau)} 3r_+(\X) = \sup_{\X\in\cN_{\cM_{\text{RIP}}}(\tau)-\cM_{\text{RIP}}}\left\{-\frac{3\nabla^2_-f\big(\X^{\sharp}\big)\left[\X^{\perp},\X^{\perp}\right]}{\norm{\X^{\perp}}_{\fro}^2}\right\},
\end{align}
where the first inequality is due to the fact that $r_-(\X)\leq 0$ by definition. Therefore, to control $R(\tau)$, it suffices to provide a uniform upper bound on $-\nabla^2_-f\big(\X^{\sharp}\big)\left[\X^{\perp},\X^{\perp}\right]$ for every $\X\in\cN_{\cM_{\text{RIP}}}(\tau)-\cM_{\text{RIP}}$.
    Recall that
    \begin{equation}
        \nabla^2 f(\X)[\Z, \Z]=\inner{\cA^{*}\cA\left(\X\Z^{\top}+\Z\X^{\top}\right)\X-\cA^{*}\cA(\mDelta)\Z}{\Z}.
    \end{equation}
    Upon setting $\X=\X^{\sharp}$ and $\Z=\X^{\perp}$, we obtain
    \begin{equation}
        \nabla^2f\big(\X^{\sharp}\big)\left[\X^{\perp},\X^{\perp}\right]=\inner{-\cA^{*}\cA\big(\mDelta^\sharp\big)\X^{\perp}+\cA^{*}\cA\left(\X^{\sharp}\X^{\perp\top}+\X^{\perp}\X^{\sharp\top}\right)\X^{\sharp}}{\X^{\perp}},
    \end{equation}
    where $\mDelta^{\sharp}=\M-\X^{\sharp}\X^{\sharp\top}$.
    To calculate the NSD component of the Hessian, $\nabla^2_-f\big(\X^{\sharp}\big)\left[\X^{\perp},\X^{\perp}\right]$, we construct two auxiliary ``Hessians'' $\nabla^2 F(\X^\sharp), \nabla^2 H(\X^\sharp)$ defined as
    \begin{equation}
        \begin{aligned}
            \nabla^2F(\X^\sharp)[\Z,\Z]&=\inner{(\X^\sharp\Z^{\top}+\Z\X^{\sharp\top})\X^\sharp-\mDelta^\sharp\Z}{\Z}\\
            &=\frac{1}{2}\norm{\X^\sharp\Z^{\top}+\Z\X^{\sharp\top}}_{\fro}^2 + \inner{(\X^\sharp\X^{\sharp\top}-\M)\Z}{\Z},\\
            \nabla^2H(\X^\sharp)[\Z,\Z]&=\frac{1}{2}\norm{\X^\sharp\Z^{\top}+\Z\X^{\sharp\top}}_{\fro}^2 + \inner{\X^\sharp\X^{\sharp\top}\Z}{\Z}.
        \end{aligned}
    \end{equation}
    Here $\nabla^2 F(\X^\sharp)$ is the expected version of $\nabla^2 f(\X^\sharp)$ (after setting $\delta=0$ in $(\delta, r)$-RIP) and $\nabla^2 H(\X^\sharp)$ is the positive semi-definite part of $\nabla^2 F(\X^\sharp)$.
     Now, we are ready to control $-\nabla^2_-f\big(\X^{\sharp}\big)\left[\X^{\perp},\X^{\perp}\right]$. First, we have
    \begin{equation}\label{eq::NSD}
        \begin{aligned}
            -\nabla^2_-f\big(\X^{\sharp}\big)\left[\X^{\perp},\X^{\perp}\right]&\leq -\nabla^2_-F\big(\X^{\sharp}\big)\left[\X^{\perp},\X^{\perp}\right]+\left|\left(\nabla^2f\big(\X^{\sharp}\big)-\nabla^2F\big(\X^{\sharp}\big)\right)\left[\X^{\perp},\X^{\perp}\right]\right|\\
            &\leq -\nabla^2_-F\big(\X^{\sharp}\big)\left[\X^{\perp},\X^{\perp}\right]+\norm{\nabla^2f\big(\X^{\sharp}\big)-\nabla^2F\big(\X^{\sharp}\big)}\norm{\X^{\perp}}_{\fro}^2.
        \end{aligned}
    \end{equation}
    To proceed, we control $\norm{\nabla^2f\big(\X^{\sharp}\big)-\nabla^2F\big(\X^{\sharp}\big)}$. For any arbitrary $\Z$ satisfying $\norm{\Z}_{\fro}\leq 1$, we have
    \begin{equation}\label{eq::diff}
        \begin{aligned}
            &\left|\left(\nabla^2f\big(\X^{\sharp}\big)-\nabla^2F\big(\X^{\sharp}\big)\right)[\Z, \Z]\right|\\
            &=\left|\inner{\left(I-\cA^{*}\cA\right)\big(\mDelta^\sharp\big)\Z+\left(\cA^{*}\cA-I\right)\left(\X^{\sharp}\Z^{\top}+\Z\X^{\sharp\top}\right)\X^{\sharp}}{\Z}\right|\\
            &\leq \norm{(\cA^{*}\cA-I)\big(\mDelta^\sharp\big)}\norm{\Z}_{\fro}^2+\norm{(\cA^{*}\cA-I)(\X^{\sharp}\Z^{\top}+\Z\X^{\sharp\top})}\norm{\X^{\sharp}}_{\fro}\norm{\Z}_{\fro}\\
            & \stackrel{(a)}{\leq}\sqrt{r}\delta \norm{\mDelta^\sharp}+\sqrt{r}\delta\norm{\X^{\sharp}\Z^{\top}+\Z\X^{\sharp\top}} \cdot\sqrt{r}\norm{\X^\sharp}\\
            &\leq \sqrt{r}\delta \max\left\{\sigma_1^\star, \norm{\X^\sharp}^2\right\}+2r\delta \norm{\X^\sharp}^2\stackrel{(b)}{\leq} 12 r\sigma_1^\star\delta,
        \end{aligned}
    \end{equation}
    where, in $(a)$, we use $(\delta, 2r)$-RIP, and in $(b)$ we use the fact that $\norm{\X^{\sharp}}\leq 2\sqrt{\sigma_1^{\star}}$.
    Next, we control $-\nabla^2_-F\big(\X^{\sharp}\big)\left[\X^{\perp},\X^{\perp}\right]$. Evidently, we have $\nabla^2 H(\X^\sharp)\succeq \zero$. This implies that
    \begin{equation}
        \begin{aligned}
            -\nabla^2_-F\big(\X^{\sharp}\big)\left[\X^{\perp},\X^{\perp}\right]\leq -\nabla^2_-(F-H)\big(\X^{\sharp}\big)\left[\X^{\perp},\X^{\perp}\right]=\inner{\mTheta^\star\X^\perp}{\X^{\perp}},
        \end{aligned}
    \end{equation}
    where the equality holds since $\nabla^2(F-H)\big(\X^{\sharp}\big)\left[\X^{\perp},\X^{\perp}\right]=-\inner{\mTheta^\star\X^\perp}{\X^{\perp}}\leq 0$, which in turn implies that $\nabla^2(F-H)\big(\X^{\sharp}\big)=\nabla^2_-(F-H)\big(\X^{\sharp}\big)$. Furthermore, we have
    \begin{equation}\label{eq::MXperp}
        \begin{aligned}
            \inner{\mTheta^\star\X^\perp}{\X^{\perp}}&\leq \norm{\mTheta^\star\X^\perp}_{\fro}\norm{\X^\perp}_\fro\leq \sigma_1^{\star}\norm{\proj_{\V^{\star}}\proj_{\X^{\perp}}}\norm{\X^\perp}_{\fro}^2\leq 3\sigma_1^{\star}\sqrt{r}\kappa\delta \norm{\X^\perp}_{\fro}^2.
        \end{aligned}
    \end{equation}
    In the last inequality, we use $\norm{\proj_{\V^{\star}}\proj_{\X^{\perp}}}=\norm{\proj_{\V^{\star\perp}}\proj_{\X^\sharp}}\leq 4\sqrt{r}\kappa\delta$ since $\X^\sharp\in \cM_{\text{RIP}}$. Combining the obtained bounds in \Cref{eq::MXperp} and \Cref{eq::diff} with \Cref{eq::NSD}, we obtain
    \begin{equation}
        -\nabla^2_-f\big(\X^{\sharp}\big)\left[\X^{\perp},\X^{\perp}\right]\leq \left(3\sigma_1^{\star}\sqrt{r}\kappa\delta+12\sigma_1^{\star} r\delta \right)\norm{\X^\perp}_{\fro}^2\leq 15\sigma_1^{\star}r\kappa\delta \norm{\X^\perp}_{\fro}^2.
    \end{equation}
    Combining this inequality with \Cref{eq::Rtau} yields $R(\tau)\leq 45\sigma_1^{\star}\kappa r\delta$ with $\tau=\frac{1}{10}\sqrt{\frac{\sigma_1^{\star}}{d}}$.

\end{proof}

\paragraph{Establishing near-linear convergence of IPGD+}
With all the conditions in place, we are ready to prove the convergence of IPGD+ for the over-parameterized matrix sensing.
\begin{sloppypar}
    \begin{theorem}
    \label{thm::global-convergence-matrix-sensing}
    Suppose that $\cA$ satisfies $(\delta, 6r)$-RIP with $\delta = O\left(\frac{1}{r^{2}\kappa^{5}\left(\log^4(1/\gamma)+\log^4(\sigma_1^\star d/\chi)\right)}\right)$. Assume that the initial point is chosen as $\X_0=\zero_{n\times r'}$. Let $\bar\epsilon=O\Big(\frac{\sigma_r^{\star 2}}{\sqrt{\sigma_1^\star}}\Big)$ and $T' = O\left(\kappa \log\left(\frac{\sigma^\star_1}{\epsilon}\right)\right)$. Then, with probability at least $1-\chi$, IPGD+ with any perturbation radius satisfying $\gamma= O\Big(\min\Big\{\frac{\epsilon}{r\kappa^3\sigma_1^\star}, \frac{\sqrt{\sigma_r^\star}}{r^2\kappa^{6.5}}\Big\}\cdot\log^{-7}\left(\frac{\sigma_1^\star d}{\chi\epsilon}\right)\Big)$ and step-size $\eta = O\left(\frac{1}{\sigma_1^{\star}}\right)$ outputs a solution $\X_T$ satisfying
        \begin{equation}
            \dist(\X_T, \cX^\star)\leq \frac{\epsilon}{15\sigma_1^\star},
        \end{equation}
        within $T = O\left(\frac{r\kappa^3}{\eta\sigma_r^{\star}}\left(\log^4\left(\frac{1}{\gamma}\right)+\log^4\left(\frac{\sigma_1^\star d}{\chi}\right)\right)+\kappa \log\left(\frac{\sigma^\star_1}{\epsilon}\right)\right)$ iterations.
\end{theorem}
\end{sloppypar}
As an immediate consequence of the above theorem, for sufficiently small $\epsilon$, upon setting the perturbation radius $\gamma=\tilde O\left(\frac{\epsilon}{r\kappa^{3}\sigma_1^{\star}}\right)$ and stepsize $\eta=\Theta\left(\frac{1}{\sigma_1^{\star}}\right)$, with probability at least $1-\chi$, we have $\dist(\X_T, \cX^\star)=O\left(\epsilon/\sigma_1^\star\right)$ after $T=\tilde O(r\kappa^4\log^4(1/\epsilon))$ iterations.

\begin{proof}[Proof of \Cref{thm::global-convergence-matrix-sensing}]
    The proof is an application of \Cref{thm::linear}. Indeed, \Cref{lemma::smoothness-matrix,lem::GD-on-the-manifold-matrix-sensing,prop::matrix-factorization-strict-saddle-local-linear-convergence,prop::deviation-rate-ms} hold under the conditions of \Cref{thm::global-convergence-matrix-sensing}. 
   Specifically, we have
    \begin{itemize}
        \item $f(\X)$ in \eqref{eq::matrix-sensing} is ${(L, \cM_{\text{RIP}}, \tau)}$-gradient-Lipschitz and {${(\rho, \cM_{\text{RIP}}, \tau)}$-Hessian-Lipschitz} with parameters $L = 15\sigma_1^\star$, $\rho = 15\sqrt{\sigma_1^{\star}}$, and $\tau\leq \frac{1}{10}\sqrt{\frac{\sigma_1^{\star}}{d}}$.
        \item $f(\X)$ satisfies the gradient closure property (\Cref{assumption::GD-on-the-manifold}).
        \item $\cM_{\text{RIP}}$ satisfies $\left(\bareg, \bareH,\bareM\right)$-$\cM_{\text{RIP}}$-strict saddle property with parameters $\bareg=\frac{1}{160}\sigma_r^{\star 3 / 2}, \bareH=\frac{1}{5}\sigma_r^{\star},\bareM=\frac{1}{8}\sigma_r^{\star 1 / 2}$. Moreover, it satisfies $\left(\alpha,\beta,\zeta\right)$-$\cM_{\text{RIP}}$-regularity property with parameters $\alpha=\frac{1}{2} \sigma_r^{\star}, \beta=\frac{425}{64} \sigma_1^{\star}, \zeta=\frac{1}{4}\sigma_r^{\star 1 / 2}$.
        \item The deviation rate satisfies $R(\tau)\leq 45\sigma_1^\star r\kappa\delta$ for any $\tau\leq \frac{1}{10}\sqrt{\frac{\sigma_1^\star}{d}}$.
    \end{itemize}
    We now specify the input parameters required for the IPGD+ algorithm.
    \begin{itemize}
        \item Due to our choice of the initial point and the definition of the implicit region in \Cref{def: implicit region}, we have $\norm{\X_0^\perp}=0$. Moreover, due to RIP, we have $\Delta_f= f(\X_0)-f^\star=\frac{1}{4}\norm{\cA(\mTheta^\star)}^2\leq \frac{1+\delta}{4}\norm{\mTheta^\star}_{\fro}^2\leq \frac{r}{2}\sigma_1^{\star 2}$.
        \item The stepsize satisfies $\eta \leq\frac{1}{10\max\{\alpha, \beta, L\}}=\frac{1}{150\sigma_1^\star}$.
        \item the gradient threshold satisfies $\bar\epsilon= O\left(\min\left\{\bareg, \frac{\bareH^2}{\rho}\right\}\right)=O\Big(\frac{\sigma_r^{\star 2}}{\sqrt{\sigma_1^\star}}\Big)$.
        \item The perturbation radius $\gamma$ satisfies
        \begin{equation}
            \begin{aligned}
                \gamma &=O\bigg(\min\bigg\{\frac{\bar\epsilon^{3/2}\cdot\epsilon}{\sqrt{\rho}L\Delta_f}\cdot \log^{-7}\left(\frac{\rho d\Delta_f}{\chi\epsilon}\right), \frac{\bar\epsilon^{7/2}}{\rho^{3/2}\Delta_f^2}\cdot \log^{-7}\left(\frac{\rho d\Delta_f}{\chi\epsilon}\right) \bigg\}\bigg)\\
                &=O\left(\min\left\{\frac{\epsilon}{r\kappa^3\sigma_1^\star}\cdot\log^{-7}\left(\frac{\sigma_1^\star d}{\chi\epsilon}\right), \frac{\sqrt{\sigma_r^\star}}{r^2\kappa^{6.5}}\cdot\log^{-7}\left(\frac{\sigma_1^\star d}{\chi\epsilon}\right)\right\}\right).
            \end{aligned}
        \end{equation}
        \item The iteration count $T'$ satisfies 
        \begin{equation}
            T'=O\left(\frac{1}{\eta\alpha}\log\left(\frac{L\zeta}{\epsilon}\right)\right) = O\left(\kappa\log\left(\frac{\sigma^\star_1\sqrt{\sigma^\star_r}}{\epsilon}\right)\right)=O\left(\kappa\log\left(\frac{\sigma^\star_1}{\epsilon}\right)\right)
        \end{equation}
    \end{itemize}
    Under these settings, \Cref{thm::linear} guarantees that after 
    \begin{equation}
        \begin{aligned}
            T&= O\left(\frac{\Delta_f}{\eta\bar\epsilon^2}\left(\log^4\left(\frac{1}{\gamma}\right)+\log^4\left(\frac{\rho d\Delta_f}{\chi\bar\epsilon}\right)\right)+\frac{1}{\eta\alpha}\log\left(\frac{L\zeta}{\epsilon}\right)\right)\\
            &=O\left(\frac{r\kappa^3}{\eta\sigma_r^\star}\left(\log^4\left(\frac{1}{\gamma}\right)+\log^4\left(\frac{\sigma_1^\star d}{\chi}\right)\right)+\kappa\log\left(\frac{\sigma^\star_1}{\epsilon}\right)\right)\nonumber
        \end{aligned}
    \end{equation}
        \noindent iterations, IPGD+ outputs a solution $\X_T$ satisfying $\dist(\X_T, \cX_{\cM_{\text{RIP}}})\leq \frac{\epsilon}{15\sigma_1^\star}$. Here, we set $\delta = O\Big(\frac{1}{r^{2}\kappa^{5}\left(\log^4(1/\gamma)+\log^4(\sigma_1^\star d/\chi)\right)}\Big)$ to ensure that $R(\tau)=O(1/(\eta T))$, as required by \Cref{thm::linear}. Moreover, since $\cX_{\cM_{\text{RIP}}}=\cX^\star$ when $\delta\leq \frac{1}{10}$, it follows that $\dist(\X_T, \cX^\star)\leq \frac{\epsilon}{15\sigma_1^\star}$.
    
\end{proof}

\paragraph{Initialization} To apply IPGD+ to the over-parameterized matrix sensing problem, we initialize the algorithm at the origin. It is straightforward to verify that the origin belongs to $\cM_{\text{RIP}}$ and is a strict saddle point of \eqref{eq::matrix-sensing}. As a result, IPGD+ automatically applies a small random perturbation to the initial point, allowing the algorithm to escape this saddle point. This strategy is akin to the small random initialization used for GD in \cite{stoger2021small}.
\begin{sloppypar}
    \paragraph{Iteration complexity} The convergence rate of IPGD+ toward the true rank-$r$ solution scales as $O(\log^4(1/\epsilon))$.
While this convergence rate is significantly better than that of general saddle---escaping methods such as PGD—which fail to converge to the true solution in the over-parameterized setting where $r'>r$---it remains worse than the linear convergence rate of $O(\log(1/\epsilon))$ achieved by GD with small random initialization \cite{stoger2021small}. This suboptimality primarily arises because our algorithm is designed to handle the worst-case landscape, where it may encounter up to $O(\log^3(1/\epsilon))$ saddle points. In contrast, \cite{jin2023understanding} shows that for over-parameterized matrix sensing, GD encounters at most $r$ saddle points before reaching the true rank-$r$ solution. We believe this insight can also be incorporated into the analysis of IPGD+ to further improve its convergence rate.
\end{sloppypar}
\paragraph{Sample complexity}  \Cref{thm::global-convergence-matrix-sensing} requires the linear operator $\cA$ to satisfy the $(\delta, 6r)$-RIP with $\delta = \Tilde{O}\left(\frac{1}{r^2\kappa^5}\right)$. When the entries of the measurement matrices are drawn independently from a standard Normal distribution, the sample complexity needed to ensure $\delta = \Tilde{O}\left(\frac{1}{r^2\kappa^5}\right)$ is $\Tilde{\Omega}(\kappa^{10} d r^5)$. This bound achieves optimal dependence on the ambient dimension $d$ and, notably, is independent of the over-parameterized rank $r'$, which can be as large as $d$.

\section{Numerical Verification Beyond Over-parameterized Matrix Sensing}
\label{sec::simulation}
In this section, we empirically examine the performance of our algorithm in over-parameterized models beyond the symmetric matrix sensing setting. Specifically, we focus on two problem classes: \textit{low-rank matrix recovery} (in both symmetric and asymmetric forms) and \textit{sparse recovery}. Throughout all experiments, we employ IPGD+, which consistently demonstrates a linear convergence rate in practice.

\paragraph{Low-rank matrix recovery} 
Recall that the goal in low-rank matrix recovery is to recover a ground-truth rank-$r$ matrix $\mTheta^\star$ under various loss functions. We consider both \textit{symmetric} and \textit{asymmetric} variants. In the symmetric low-rank matrix recovery, the ground truth matrix $\mTheta^\star\in \bR^{n\times n}$ is assumed to be PSD. We consider the reparameterization map $\mTheta=\X\X^\top$ with $\X\in \bR^{n\times r'}$ and $r'\geq r$, and optimize the loss function $f(\X)=L(\X\X^\top)$. For this family of problems, we take the implicit region as $\cM = \{ \X \in \mathbb{R}^{n \times r'}: \mathrm{rank}(\X) \le r \}$. We emphasize that this definition of $\cM$ is similar to \Cref{def: implicit region}, with the key difference that the constraints $\norm{\proj_{\V^{\star\perp}}\proj_{\X}}\leq 4\sqrt{r}\kappa\delta$ and $\norm{\X}\leq 2\sqrt{\sigma_1^{\star}}$ are omitted. This simpler formulation of the implicit region is more broadly applicable across various low-rank matrix recovery problems and, notably, does not depend on the specific observation model used (such as the RIP). Even within the aforementioned matrix sensing setting, we have observed in practice that the omitted constraints do not become active. In other words, $\X^\sharp_t = \cP_{\cM}(\X_t)$ automatically satisfies $\big\|\proj_{\V^{\star\perp}}\proj_{\X^\sharp_t}\big\|\leq 4\sqrt{r}\kappa\delta$ and $\big\|\X^\sharp_t\big\|\leq 2\sqrt{\sigma_1^{\star}}$ for almost all iterations $t$.

\begin{sloppypar}
   In the asymmetric case, the target rank-$r$ matrix $\mTheta^\star \in \mathbb{R}^{n_1 \times n_2}$ is assumed to be potentially indefinite and rectangular. We consider the reparameterization map $\mTheta = \X \Y^\top$, where $\X \in \mathbb{R}^{n_1 \times r'}$ and $\Y \in \mathbb{R}^{n_2 \times r'}$ with $r' \geq r$, and optimize the loss function $f(\X, \Y) = L(\X \Y^\top)$. We take the rank-balanced implicit region as $\cM = \{ (\X, \Y) \in \mathbb{R}^{n_1 \times r'} \times \mathbb{R}^{n_2 \times r'}: \mathrm{rank}(\X) \leq r,\ \mathrm{rank}(\Y) \leq r \}$.
\end{sloppypar}

Beyond matrix sensing, we also consider the following two settings:

\begin{itemize}
    \item \textit{Matrix completion \cite{johnson1990matrix, candes2012exact}}: Our goal is to recover a rank-$r$ matrix $\mTheta^\star$ from a subset of its observed entries $\Omega \subseteq [n] \times [n]$. The loss function is defined as
\begin{equation}
    L(\mTheta) = \sum_{(i, j) \in \Omega} \left(\mTheta_{i, j} - \mTheta^\star_{i, j} \right)^2.
\end{equation}
We generate $\Omega$ uniformly at random with $|\Omega| = np$, where the observation probability is fixed at $p = 0.8$ in all experiments.

    \item \textit{1-bit matrix completion \cite{davenport20141}}: This problem extends matrix completion by observing only binary (1-bit) measurements of the underlying matrix $\mTheta^\star$. Each entry $(i, j)$ is observed as $1$ with probability $\sigma(\mTheta^\star_{i,j})$ and $0$ otherwise, Here, $\sigma(x) = (1 + e^{-x})^{-1}$ denotes the sigmoid function. The corresponding expected loss function is 
\begin{equation}
    L(\mTheta) = \sum_{i, j} \left( \log\left(1 + \exp\left(\mTheta_{i, j}\right)\right) - \sigma(\mTheta^\star_{i, j}) \mTheta_{i, j} \right).
\end{equation}   

\end{itemize}

In all of our experiments, we set $n_1 = n_2 = n = 20$, $r'=20$, and $r = 3$. For the symmetric variants, the ground-truth matrix is constructed as $\mTheta^\star = \mathbf{U}^\star \mathbf{\Sigma}^\star \mathbf{U}^{\star\top}$, where $\mathbf{U}^\star \in \mathbb{R}^{n \times r}$ is a randomly generated orthogonal matrix, and $\mathbf{\Sigma}^\star \in \mathbb{R}^{r \times r}$ is a diagonal matrix with entries $10$, $5$, and $1$. Similarly, for the asymmetric variants, the ground-truth matrix is constructed as $\mTheta^\star = \mathbf{U}^\star \mathbf{\Sigma}^\star \mathbf{V}^{\star\top}$, where $\mathbf{U}^\star \in \mathbb{R}^{n_1 \times r}$ and $\mathbf{V}^\star \in \mathbb{R}^{n_2 \times r}$ are randomly generated orthogonal matrices, and $\mathbf{\Sigma}^\star$ is the same diagonal matrix as in the symmetric case. In all experiments, the initial points are initialized as all-zero matrices. The perturbation radius is set to $\gamma = 1 \times 10^{-15}$, and the gradient norm threshold is set to $G = 1 \times 10^{-7}$. We use $F = 1 \times 10^{-10}$ and $T_{\mathrm{escape}} = 50$ in all settings. The step size is set to $\eta = 0.1$ for matrix sensing and 1-bit matrix completion, and to $\eta = 0.02$ for matrix completion.

\Cref{fig:simu} presents the behavior of IPGD+ (\Cref{alg::tsp-gd-local-refinement}) applied to the symmetric and asymmetric variants of matrix sensing, matrix completion, and 1-bit matrix completion.
Our first key observation is that IPGD+ not only stays close to the defined implicit region $\cM$---as evidenced by the consistently small residual norm---but also efficiently converges to a ground-truth solution in $\XM$. Notably, the final distance to $\XM$ precisely matches the residual norm across all instances considered. This phenomenon indicates that our signal-residual decomposition provides an accurate characterization of the performance of IPGD+.  
Moreover, a closer examination of \Cref{fig:simu} reveals that the cumulative effect of the deviation rate $r(\cdot)$ (\Cref{def::deviation_rate}) ultimately governs the final residual norm. Specifically, it suggests that a larger cumulative deviation rate, quantified as $\bar r_T = \sum_{t=0}^{T-1} r(\x_t)$, leads to a faster growth in the residual norm. This observation aligns with the bound established in \Cref{eq: x perp bound}. Later in this section, we provide a more detailed analysis of the relationship between the cumulative deviation rate and the final residual norm.

\begin{figure}[t]
     \centering
     \subfigure[Sym.~Matrix~sensing]{
         \centering
         \includegraphics[width=0.26\textwidth]{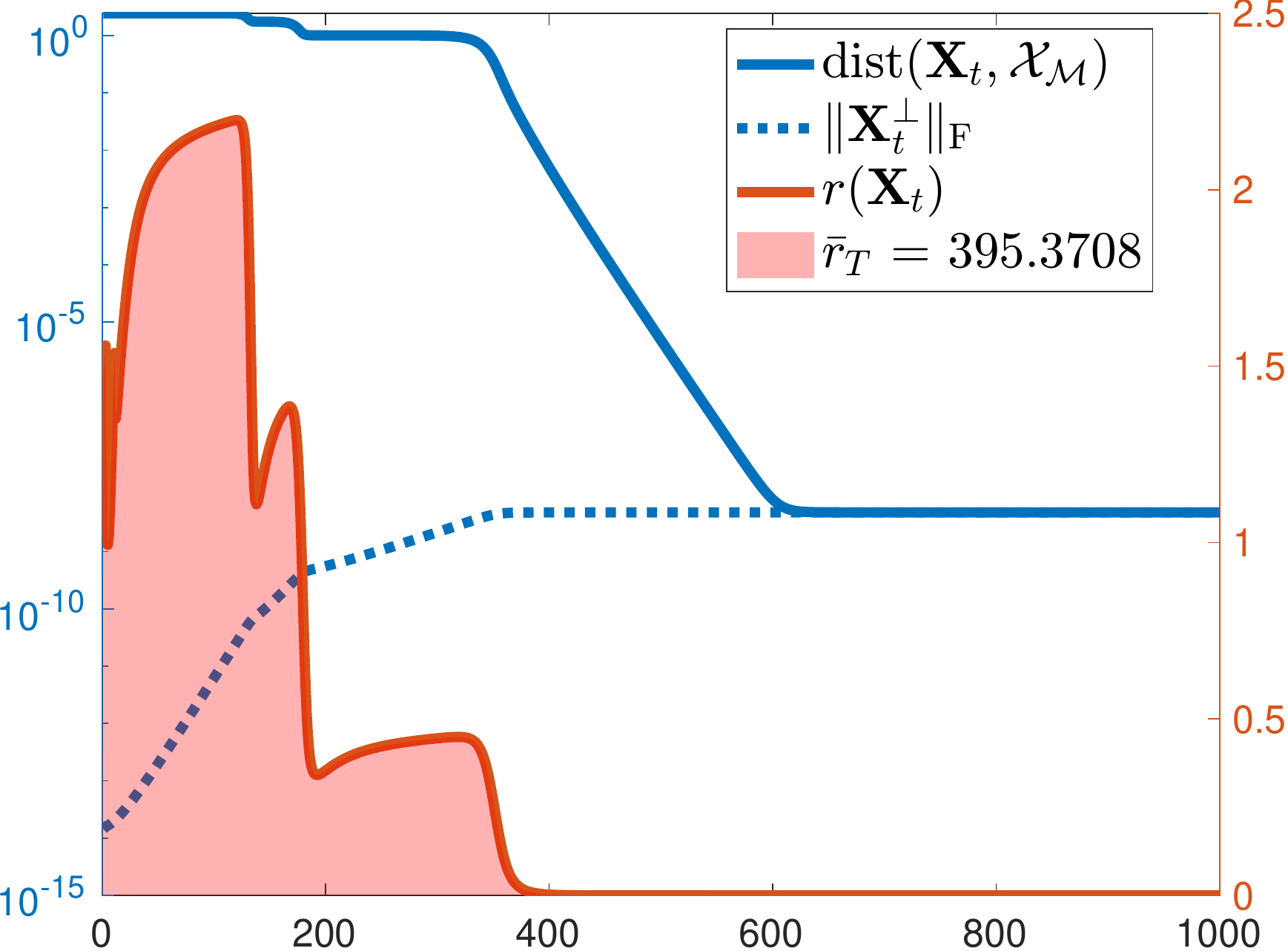}
         \label{fig:ms sym}
     }
     \subfigure[Sym. matrix completion]{
         \centering
         \includegraphics[width=0.26\textwidth]{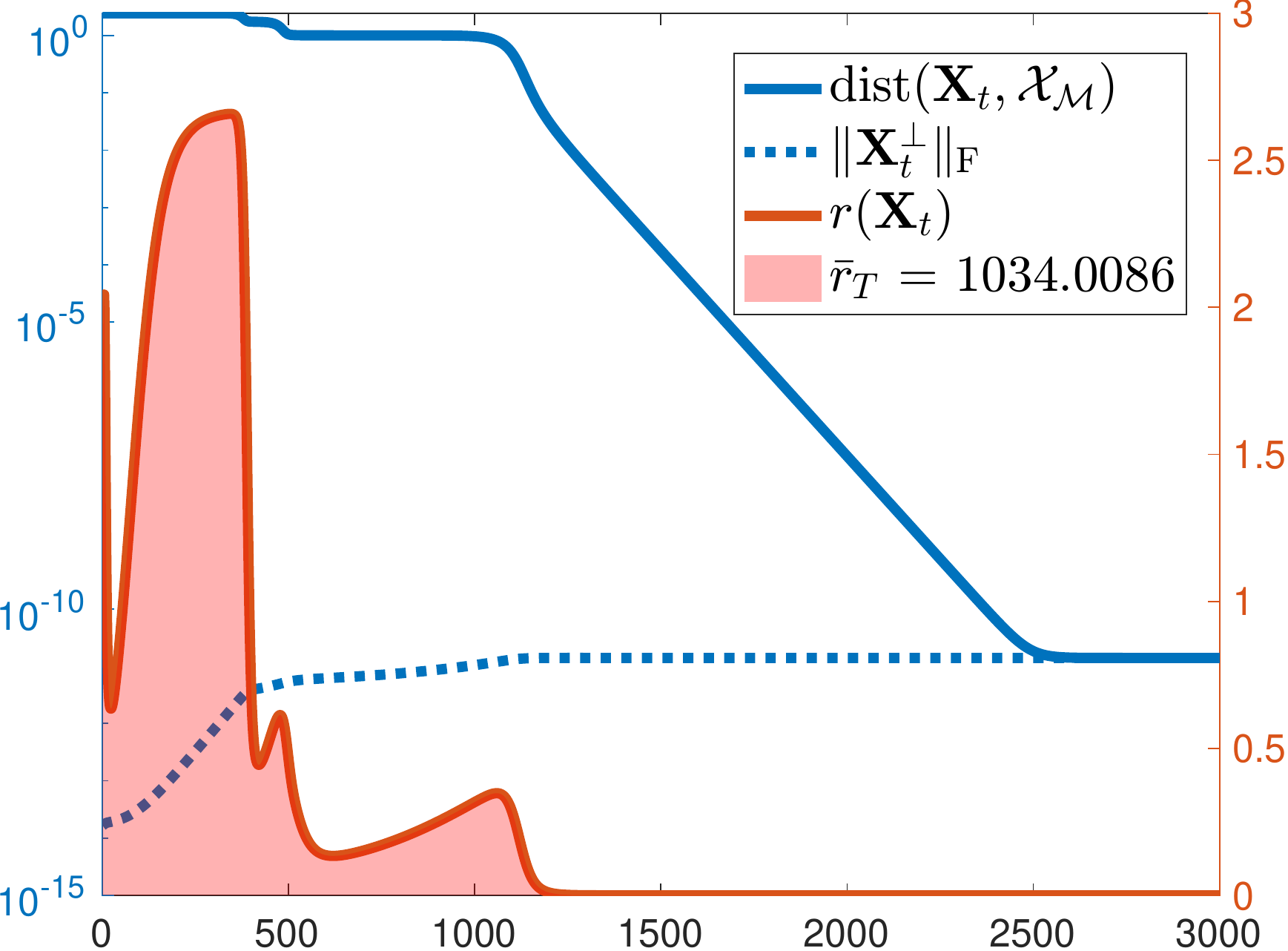}
         \label{fig:mc sym}
     }
     \subfigure[Sym.~1-bit~completion]{
         \centering
         \includegraphics[width=0.26\textwidth]{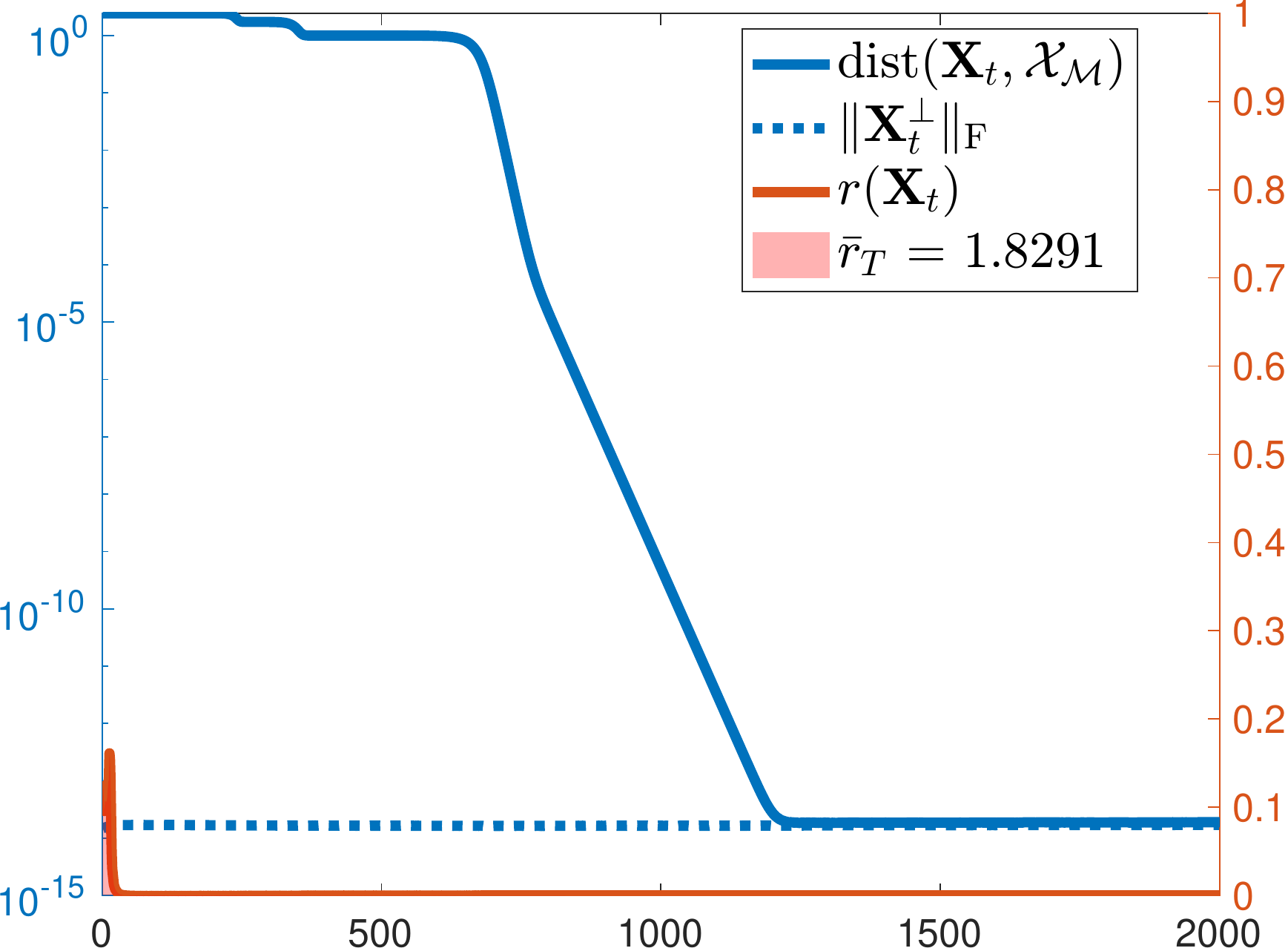}
         \label{fig:1bit sym}
     }
     \\
     \subfigure[Asym. matrix sensing]{
         \centering
         \includegraphics[width=0.26\textwidth]{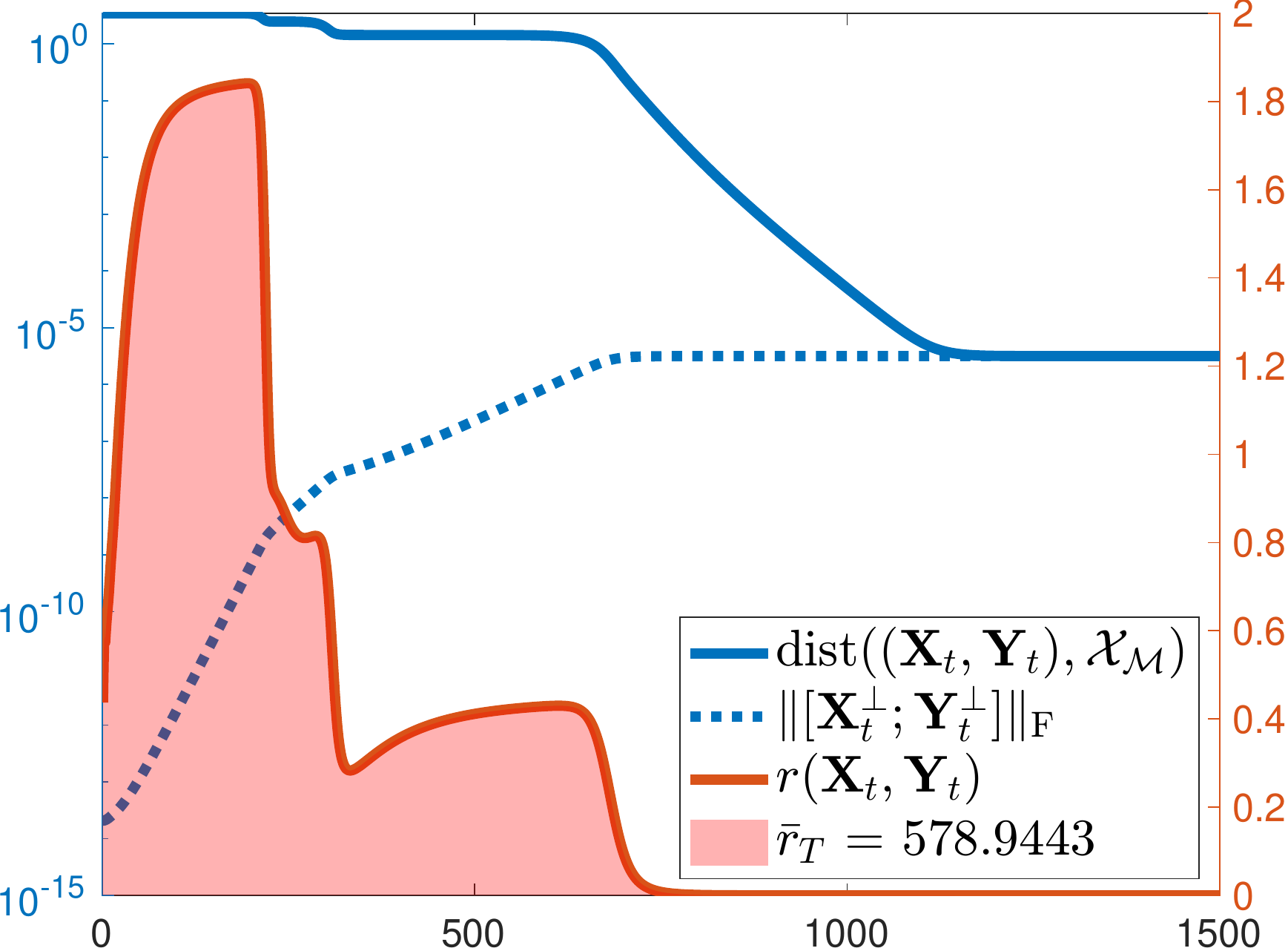}
         \label{fig:ms asym}
     }
     \subfigure[Asym. matrix completion]{
         \centering
         \includegraphics[width=0.26\textwidth]{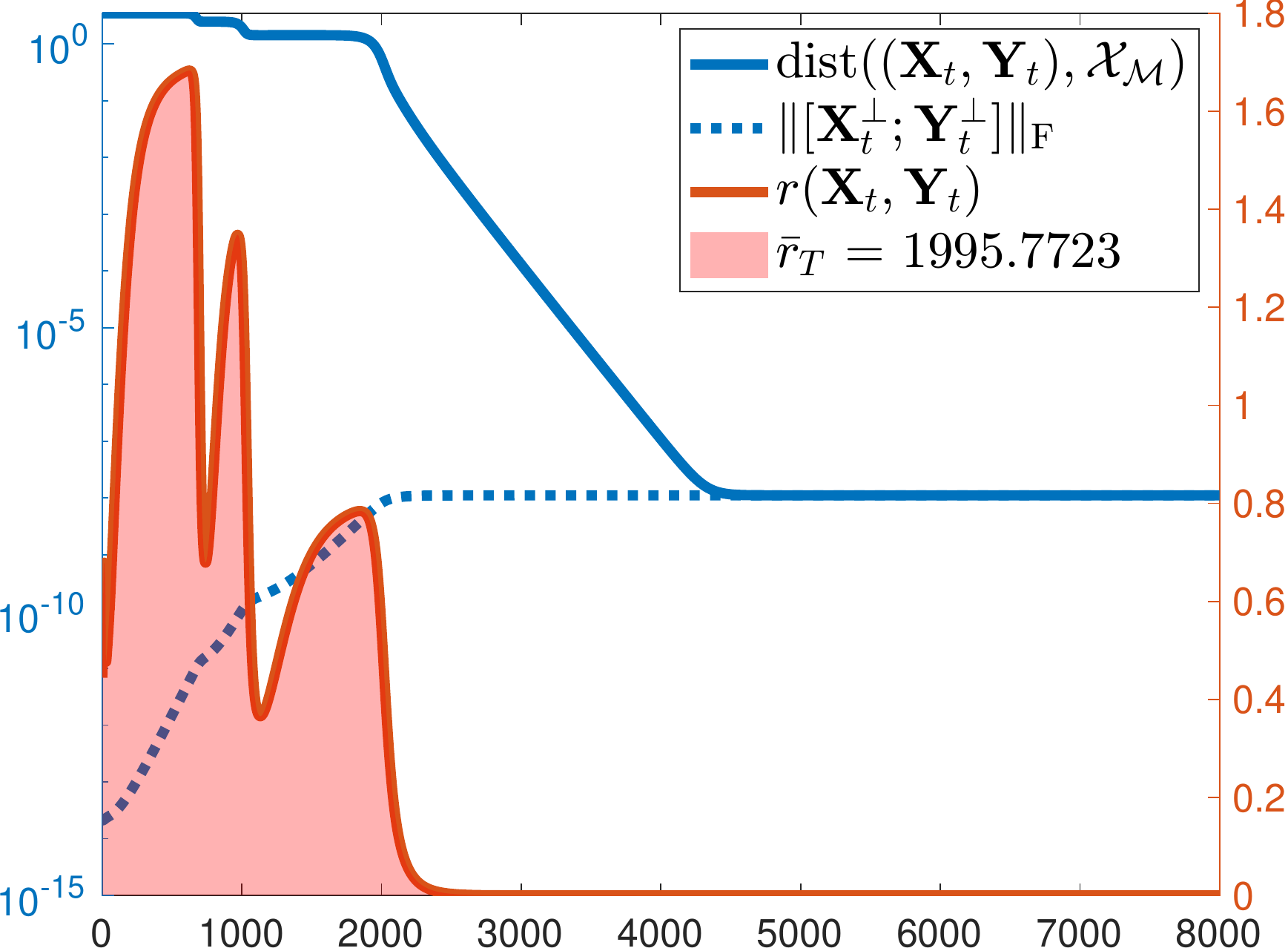}
         \label{fig:mc asym}
     }
     \subfigure[Asym. 1-bit completion]{
         \centering
         \includegraphics[width=0.26\textwidth]{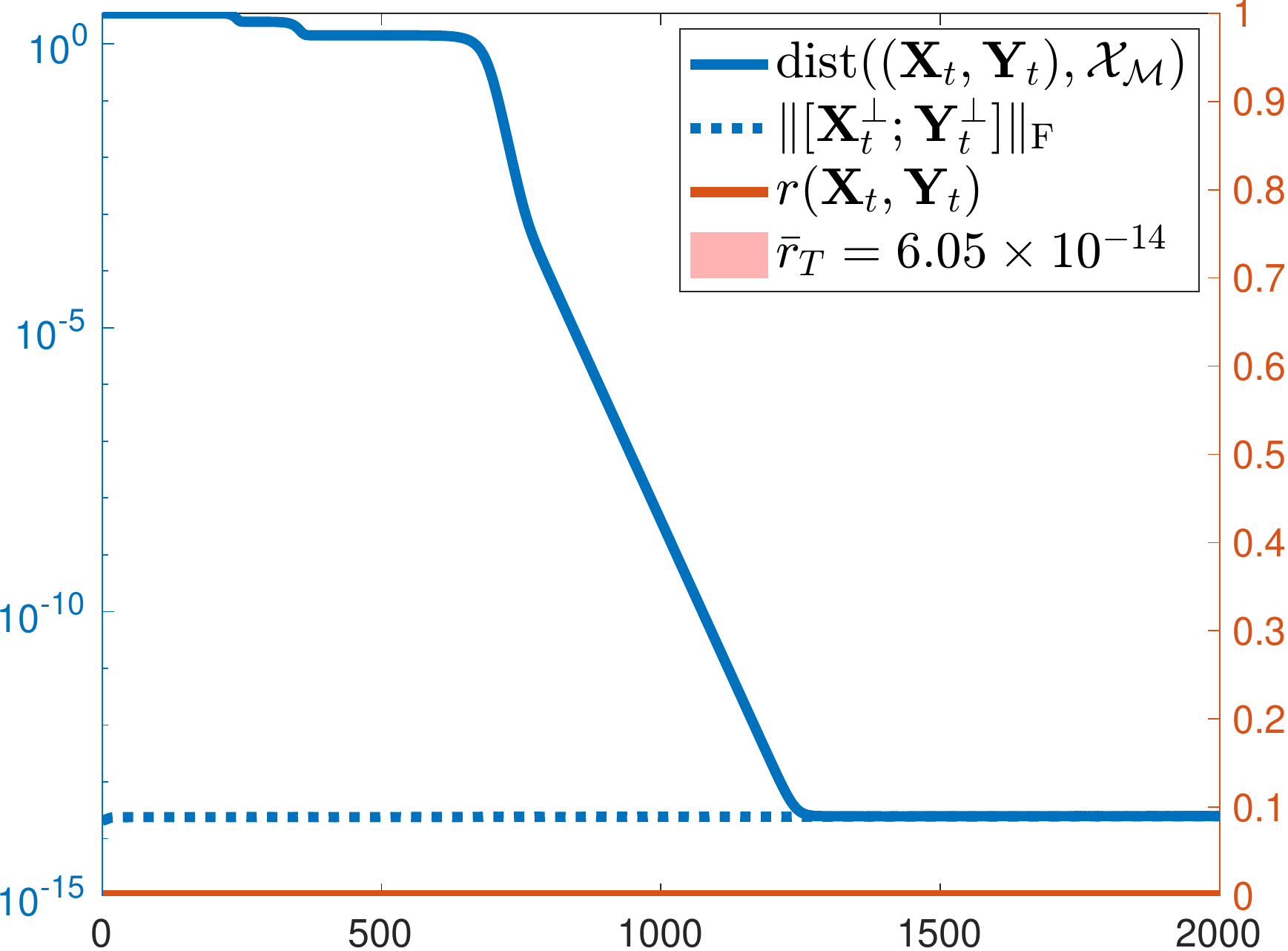}
         \label{fig:1bit asym}
     }
        \caption{{\bf Behavior of IPGD+ on low-rank matrix recovery.} The deviation rate (denoted as $r(\cdot)$) and its cumulative counterpart (shown as $\bar r_T$) serve as strong indicators of whether IPGD+ converges to the ground truth. A large value of $\bar r_T$ suggests that the residual norm grows significantly, indicating that IPGD+ is directing the iterations away from the implicit region $\cM$. Conversely, a small value of $\bar r_T$ implies that the residual norm remains small, thereby keeping the iterations close to the implicit region $\cM$.}
        \label{fig:simu}
\end{figure}

\paragraph{Sparse recovery}\label{subsec::sparse-recovery} In this problem, the goal is to recover a $k$-sparse vector $\mtheta^{\star} \in \bR^d$ from a dataset $\{(\x_i, y_i)\}_{i=1}^N$, where $y_i = \inner{\mtheta^{\star}}{\x_i}$. We use the over-parameterized model $\mtheta^\star = \bu \odot \bv$, and minimize the objective function 
\begin{equation}
    f(\bu, \bv) = \sum_{i=1}^N \left( y_i - \inner{\bu \odot \bv}{\x_i} \right)^2.
    \label{eq::sparse-recovery}
\end{equation}
We use the implicit region defined in \Cref{example::sparse-recovery}: $\cM_{\text{balanced}} = \{(\u,\v) : \u_i = \v_i = 0 \ \text{if} \ \mtheta^\star_i = 0\}$.

\begin{figure}
     \begin{center}
\includegraphics[width=0.4\textwidth]{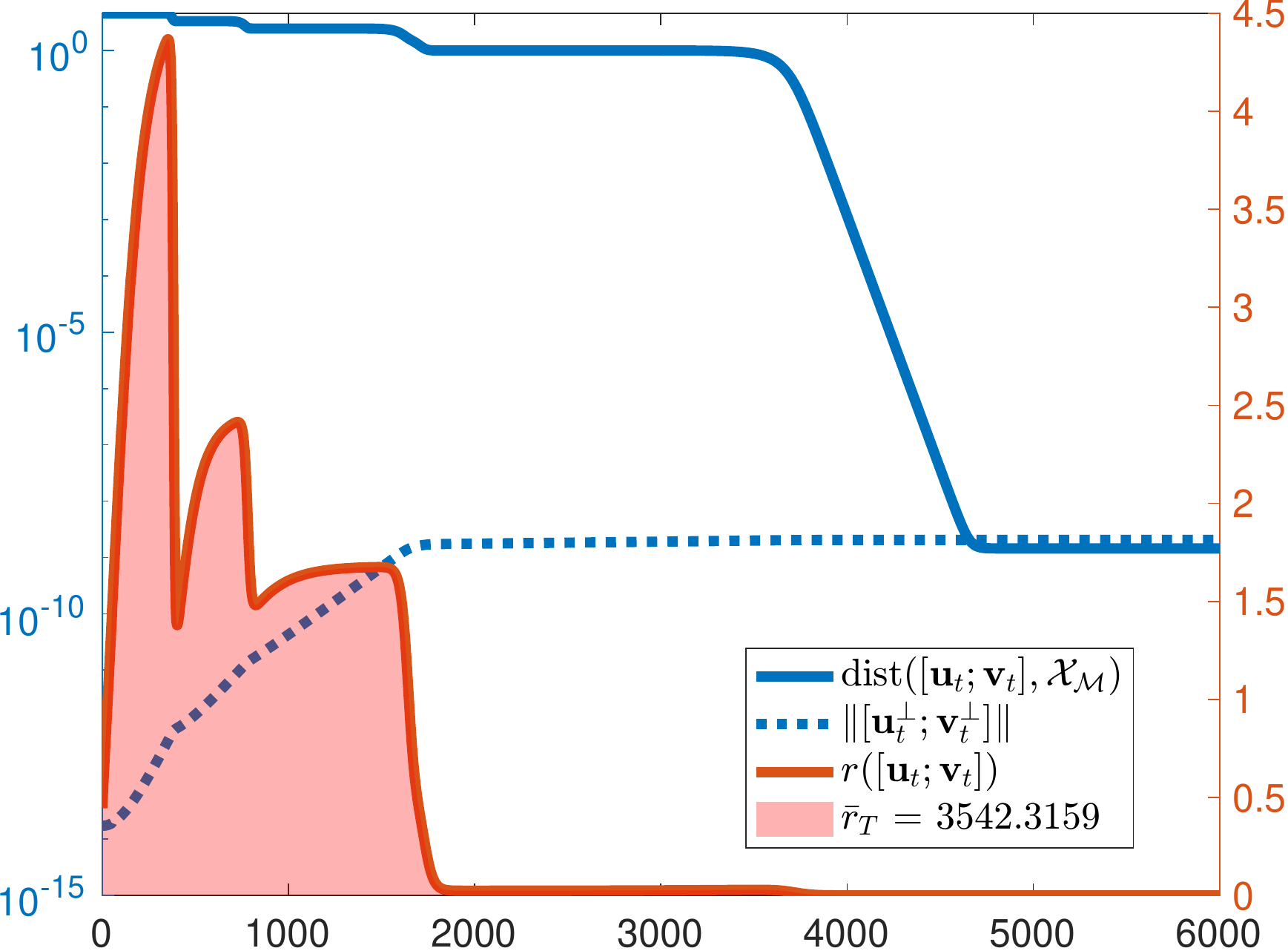}
\caption{{\bf Behavior of IPGD+ on sparse recovery.} Despite the rapid growth of the residual norm during the early iterations of IPGD+, driven by the large cumulative deviation rate, the residual norm eventually stabilizes at approximately $10^{-9}$, resulting in a final solution error of the same magnitude.}\label{fig::sparse-recovery}
  \end{center}
\end{figure}

In our experiments, we set $d = 150$, $N = 300$, and $k = 5$. The ground-truth vector is set to an all-zero vector, except for its first five entries, which are set to \(10\), \(-5\), \(3\), \(-2\), and \(1\), respectively. The entries of $\{\x_i\}_{i=1}^N$ are independently drawn from a standard normal distribution. As in previous case studies, the initialization is set to an all-zero vector, with a perturbation radius of $\gamma = 1 \times 10^{-15}$. We set the step size to $\eta = 1 \times 10^{-2}$, the gradient norm threshold to $G = 1 \times 10^{-7}$, the escape time to $T_{\mathrm{escape}} = 50$ and the termination threshold to $F = 1 \times 10^{-10}$.

\Cref{fig::sparse-recovery} shows the performance of IPGD+ on a sparse recovery instance. Similar to the low-rank matrix recovery, the final solution produced by IPGD+ lies in close proximity to the ground truth, with an error matching the residual norm. Notably, the observed cumulative deviation rate is larger compared to the instances of low-rank matrix recovery considered, leading to a rapid \(10^5\)-fold increase in the residual norm during the early iterations of the algorithm. Nonetheless, despite this steep rise, the final residual norm remains below $10^{-8}$, coinciding with the final error of the recovered solution.

\paragraph{Effect of the deviation rate on the empirical growth of the residual norm} In our experiments, we observed that a faster growth rate of the residual norm is strongly correlated with a larger cumulative deviation rate. This observation is also supported by our analysis, and in particular by \Cref{eq: x perp bound}, which yields
\begin{equation}\label{eq::cdr_growth}
    \log\left(\frac{\norm{\mathbf{x}^\perp_T}}{\norm{\mathbf{x}^\perp_0}}\right) \le \frac{\eta}{2} \left(\sum_{t=0}^{T-1} r(\x_t)\right) + \frac{\eta \rho}{2} \left(\sum_{t=0}^{T-1} \norm{\xperp_t}\right) = \frac{\eta}{2} \cdot \bar{r}_T + \frac{\eta \rho}{2} \left(\sum_{t=0}^{T-1} \norm{\xperp_t}\right),
\end{equation}
where $\bar{r}_T$ denotes the cumulative deviation rate. We note that the above inequality provides only an upper bound on the growth rate of the residual norm. This naturally raises the follow-up question: how tight is this bound in practice? \Cref{fig::r plot} plots the values of $\log\left(\frac{\norm{\mathbf{x}^\perp_T}}{\norm{\mathbf{x}^\perp_0}}\right)$ against $\eta \bar{r}_T$ for different instances of low-rank matrix recovery and sparse recovery. Empirically, we observe that the cumulative deviation rate and the empirical growth rate of the residual norm follow the relationship
\begin{align}
    \log\left(\frac{\norm{\mathbf{x}^\perp_T}}{\norm{\mathbf{x}^\perp_0}}\right) \approx \frac{\eta}{3} \bar{r}_T,
\end{align}
indicating that \Cref{eq::cdr_growth} is tight up to a factor of $1.5$, modulo the additional term $\frac{\eta \rho}{2} \left(\sum_{t=0}^{T-1} \norm{\xperp_t}\right)$, which remains negligible as $\norm{\xperp_t}$ stays small throughout the iterations of IPGD+. This observation confirms that the deviation rate, identified in our theory as the primary driver of the residual norm growth, also manifests in practice, underscoring the sharpness of our theoretical analysis.

\begin{figure}[t]
    \centering
    \includegraphics[width=0.5\textwidth]{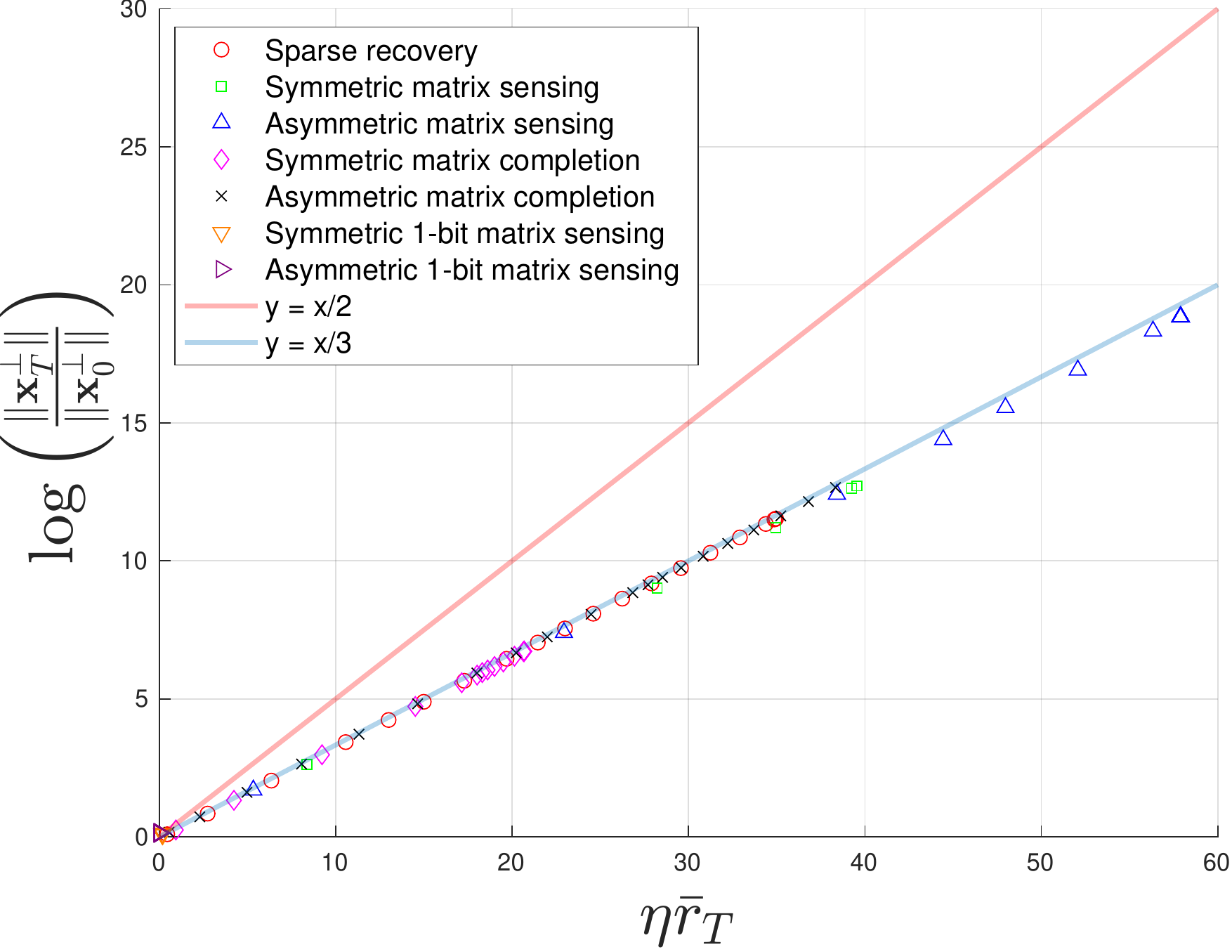}
    \caption{{\bf Cumulative deviation rate vs. empirical growth of the residual norm.}  We report the values of $\log\left(\frac{\|\mathbf{x}^\perp_T\|}{\|\mathbf{x}^\perp_0\|}\right)$ and $\eta \bar r_T$ across different values of $T$ and for various objective functions. Specifically, for each experiment discussed above, we select 20 values of $T$ uniformly spaced over the trajectory. The red line labeled $y=x/2$ represents the bound we derive from \Cref{eq::cdr_growth}. In practice, the values concentrate around the blue line representing $y=x/3$.}
    \label{fig::r plot}
\end{figure}

\section{Conclusion}
\label{sec::conclusion}

In this paper, we study why gradient-based methods in over-parameterized models often enjoy implicit regularization towards low-dimensional solutions, even in the absence of explicit constraints or regularization. To address this question, we investigate the conditions that enable implicit regularization by examining when gradient descent converges to a second-order stationary point (SOSP) within an implicit low-dimensional region of a smooth, nonconvex function. We identify and formalize the key conditions under which such implicit regularization emerges: namely, the combined ability to (i) initialize within a neighborhood of the implicit region, (ii) escape strict saddle points through infinitesimal perturbations, and (iii) remain close to the implicit region by controlling the deviation rate. Building on these insights, we introduce the \textit{infinitesimally perturbed gradient descent} (IPGD) algorithm and show that, with proper initialization and reparameterization, it can provably satisfy all of these conditions. We further analyze the application of IPGD in the context of over-parameterized matrix sensing. Our theoretical results are supported by extensive empirical evidence, suggesting that these principles extend broadly across diverse over-parameterized learning settings.

We now discuss two future directions that our proposed framework could potentially address.

\paragraph*{Beyond gradient descent with infinitesimal perturbations} While this work focuses on GD with infinitesimal perturbations, a natural follow-up question is whether our framework extends to other widely used variants of GD, such as preconditioned~\cite{zhang2021preconditioned, zhang2023preconditioned, tong2021accelerating, xu2023power}, accelerated~\cite{carmon2018accelerated, agarwal2017finding}, or stochastic gradient descent (SGD)~\cite{fang2019sharp}. Take SGD, for instance: unlike IPGD, the perturbations in SGD need not be small; however, they are also not isotropic, meaning they affect the signal and residual norms unevenly. In some cases, such anisotropic perturbations may be advantageous; for example, when they impact the residual norm much less than the signal norm.  An intriguing open question is whether the proposed framework, based on the signal-residual decomposition and the deviation rate from the implicit region, can be generalized to encompass such settings. Developing a broader notion of deviation rate could enable direct, principled comparisons between algorithms. 

\paragraph*{Extension to nonsmooth optimization} One notable limitation of our current framework is its reliance on the gradient and Hessian Lipschitz properties of the objective function within a neighborhood of the implicit region, restricting its applicability to smooth problems. An enticing direction for future work is to explore whether our framework can be extended to nonsmooth settings by leveraging tools from variational analysis. Such an extension would significantly broaden the scope of our framework, encompassing important applications such as robust principal component analysis, robust matrix recovery, and neural networks with nonsmooth activation functions. We emphasize, however, that such an extension is far from trivial, as the challenge of escaping strict saddle points is substantially more delicate in nonsmooth settings. Notably, recent work has shown that even in exactly-parameterized problems, first-order methods such as the stochastic sub-gradient method \cite{davis2025active} and proximal methods \cite{davis2019active} can only escape \textit{active} strict saddle points. This distinction introduces an additional layer of complexity for developing a theoretical understanding of implicit regularization in nonsmooth regimes.

\section*{Acknowledgements}
The authors would like to thank Richard Y. Zhang for his thoughtful comments on the initial draft of this paper.
This research is supported, in part, by NSF CAREER Award CCF-2337776, NSF Award DMS-2152776, and ONR Award N00014-22-1-2127.
The authors have no conflicts of interest or financial ties to disclose.
	
	\bibliographystyle{alpha}
	\bibliography{ref.bib}

	\appendix

\end{document}